\documentclass{article}

\usepackage{natbib}
\usepackage[margin=1in]{geometry}
\usepackage[utf8]{inputenc} 
\usepackage[T1]{fontenc}    
\usepackage{lmodern}
\usepackage[colorlinks, linkcolor=blue, anchorcolor=blue, citecolor=blue,hypertexnames=false]{hyperref}      
\usepackage{url}            
\usepackage{booktabs}       
\usepackage{amsfonts}       
\usepackage{nicefrac}       
\usepackage{microtype}      
\usepackage{xcolor}         
\usepackage{latexnotation,enumitem}
\usepackage{amsmath,mathrsfs}
\usepackage{amssymb}
\usepackage{mathtools}
\usepackage{amsthm}
\usepackage{lipsum}
\usepackage[capitalize,noabbrev]{cleveref}
\usepackage{microtype}
\usepackage{graphicx}
\usepackage{subfigure}
\usepackage{booktabs} 
\usepackage{latexnotation,enumitem,multicol}
\usepackage[textsize=tiny]{todonotes}

\newcommand{\commentout}[1]{}

\newcommand{\card}{\mathrm{Card}}

\newcommand{\ReLU}{\mathrm{ReLU}}
\newcommand{\CL}{\mathrm{CL}}
\newcommand{\Var}{\mathrm{Var}}

\theoremstyle{plain}
\newtheorem{theorem}{Theorem}[]

\newtheorem{lemma}{Lemma}
\newtheorem{setting}{Setting}
\newtheorem{corollary}{Corollary}
\newtheorem{example}{Example}

\theoremstyle{definition}
\newtheorem{definition}{Definition}
\newtheorem{assumption}{Assumption}
\theoremstyle{remark}
\newtheorem{remark}{Remark}

\newcommand{\vquarter}{\vspace{-.04in}}
\newcommand{\vhalf}{\vspace{-.09in}}

\allowdisplaybreaks

\title{ Generalization Error Guaranteed  Auto-Encoder-Based Nonlinear Model Reduction for Operator Learning}

\author{
Hao Liu, Biraj Dahal, Rongjie Lai, Wenjing Liao
\thanks{Hao Liu is affiliated with the Math department of Hong Kong Baptist University; Biraj Dahal and Wenjing Liao are affiliated with the School of Math at Georgia Tech; Rongjie Lai is affiliated with the Department of Mathematics at Purdue University. Email: \text{haoliu@hkbu.edu.hk, lairj@purdue.edu, $\{$bdahal, wliao60$\}$@gatech.edu}. This research is partially supported by National Natural Science Foundation of China  12201530, HKRGC ECS 22302123, HKBU 179356,  NSF DMS-2401297, NSF DMS--2012652,  NSF DMS-2145167 and DOE SC0024348.}}

\date{}
\begin{document}

\setlength{\abovedisplayskip}{3pt}
\setlength{\belowdisplayskip}{3pt}

\maketitle

\begin{abstract}

Many physical processes in science and engineering are naturally represented by operators between infinite-dimensional function spaces. The problem of operator learning, in this context, seeks to extract these physical processes from empirical data, which is challenging due to the infinite or high dimensionality of data. An integral component in addressing this challenge is model reduction, which reduces both the data dimensionality and problem size. In this paper, we utilize low-dimensional nonlinear structures in model reduction by investigating Auto-Encoder-based Neural Network (AENet). AENet first learns the latent variables of the input data and then learns the transformation from these latent variables to corresponding output data. Our numerical experiments validate the ability of AENet to accurately learn the solution operator of nonlinear partial differential equations. Furthermore, we establish a mathematical and statistical estimation theory that analyzes the generalization error of AENet. Our theoretical framework shows that the sample complexity of training AENet is intricately tied to the intrinsic dimension of the modeled process, while also demonstrating the remarkable resilience of AENet to noise.
\end{abstract}

\section{Introduction}

In the last two decades, deep learning has made remarkable successes in various fields such as computer vision \citep{krizhevsky2012imagenet, goodfellow2014generative}, natural language processing \citep{graves2013speech}, healthcare \citep{miotto2017deep}, and robotics \citep{gu2017deep}, among others. More recently, deep neural networks have been extended to a wide range of applications in scientific computing. This expansion includes numerical partial differential equations (PDEs) \citep{sirignano2018dgm,khoo2021solving,han2018solving,raissi2019physics,zang2020weak}, computational inverse problems \citep{ongie2020deep,khoo2019switchnet}, dynamics prediction \citep{ling2016reynolds},  and model reduction \citep{otto2019linearly,wang2018model,lee2020model,fresca2021comprehensive,gonzalez2018deep,fresca2021comprehensive}, to name a few. These developments demonstrate the versatility and potential of deep learning in scientific machine learning.

In a wide array of scientific and engineering applications, numerous objects of interest are represented as functions or vector fields. For instance, many physical processes are modeled by operators that act between these function spaces. Differential equations are typical tools used to model such physical processes. With the advancement of machine learning, there has been a surge in data-driven approaches to understand physical processes. These methods enable the characterization and simulation of physical processes based on training data.
In recent years, significant advances have been made in the realm of operator learning, which focuses on learning unknown operators within functional spaces. 
Representative works include DeepONets \citep{lu2021learning} based on the universal approximation theory in \cite{chen1995universal}, Neural Operators \citep{kovachki2021neural,li2020neural,li2020fourier}, BelNet \citep{zhang2023belnet}, etc. Mathematical theories on the approximation and generalization errors of operator learning can be found in \cite{lanthaler2022error} for DeepONets, in \cite{kovachki2021universal} for FNO, and in \cite{bhattacharya2020model,liu2022deep} for operator learning with dimension reduction techniques.

One of the major challenges of operator learning arises from the infinite or high dimensionality of the problem/data.
Recent approximation theories on neural networks for operator learning \cite{lanthaler2023curse} demonstrate that,  operator learning methods, including DeepONet \citep{lu2021learning}, NOMAD \citep{NEURIPS2022_24f49b2a} and Fourier Neural Operator (FNO) \citep{li2020fourier}, suffer from the curse of dimensionality without additional assumptions on low-dimensional structures. 
Specifically, \cite{lanthaler2023curse} proves that, there exists a $r$-times Fr\'echet differentiable functional, such that in order to achieve an $\epsilon$ approximation error for this functional, the network size of DeepONet, NOMAD (NOnlinear MAnifold Decoder) and FNO  is lower bounded by ${\rm exp}(C\epsilon^{-1/(cr)})$ where $C,c$ are constants depending on the problem. Such results demonstrate that, a huge network is needed to universally approximate $r$-times Fr\'echet differentiable functionals on an infinite dimensional space. It is impossible to reduce the network size unless additional structures about the operator (or input and output) are exploited.

Fortunately, the vast majority of real-world problems exhibit low-dimensional structures. For example, functions generated from translations or rotations only depend on few parameters \citep{tenenbaum2000global,roweis2000nonlinear,coifman2005geometric}; Whale vocal signals  can be parameterized by polynomial phase coefficients \citep{xian2016intrinsic}; In molecular dynamics, the dynamical evolution is often governed by a small number of slow modes \citep{ferguson2010systematic,ferguson2011nonlinear}. Thus it is natual to consider model reduction to reduce the data dimension and the problem size. 

In the literature, linear reduction methods have shown considerable success when applied to models existing within low-dimensional linear spaces. Examples include the reduced-basis technique \cite{prud2002reliable,rozza2008reduced}, proper orthogonal decomposition \cite{carlberg2011low,holmes2012turbulence}, Galerkin projection \cite{holmes2012turbulence}, and many others. A survey about model reduction can be found in \cite{benner2015survey,benner2017model}. More recently,
linear model reduction methods have been combined with deep learning in various ways.   \cite{hesthaven2018non,wang2019non} consider very low-dimensional inputs and employ Principal Component Analysis (PCA) \citep{hotelling1933analysis} for the output space. \cite{bhattacharya2020model} use PCA for both the input and output spaces. The active subspace method is used in \cite{o2022derivative} for dimension reduction. In these works, dimension reduction is achieved by existing linear model reduction methods, and operator learning is carried on the latent variables by a neural network. 
Theoretically, the network approximation error and stochastic error of PCA are analyzed in \cite{bhattacharya2020model}. A generalization error analysis on operator learning with linear dimension reduction techniques is given in \cite{liu2022deep}. This paper shows that fixed linear encoders given by Fourier basis or Legendre polynomials give rise to a slow rate of convergence of the generalization error as $n$ increases, and data-driven PCA encoders are suitable for input and output functions concentrated near low-dimensional linear subspaces.

However, many physical processes in practical applications are inherently nonlinear, such as fluid motion, nonlinear optical processes, and shallow water wave propagation. As a result, the functions of interest frequently reside on low-dimensional manifolds rather than within low-dimensional subspaces.
Addressing these nonlinear structures is vital in model reduction. Recent studies have shown that deep neural networks are capable of representing a broad spectrum of nonlinear functions~\citep{yarotsky2017error,lu2021deep,suzuki2018adaptivity} and adapting to the low-dimensional structures of data \citep{chen2019efficient,chen2022nonparametric,hao2021icml,liu2023deep,nakada2020adaptive}. Auto-Encoders, in particular, have gained widespread use in identifying low-dimensional latent variables within data~\citep{kramer1991nonlinear,kingma2019introduction}.  Approximation and Statistical guarantees of Auto-Encoders for data near a low-dimensional manifold are established in~\citet{schonsheck2019chart,tang2021empirical,liu2023deep}.

In literature,  Auto-Encoder-based neural networks have been proposed for model reduction in various ways \citep{otto2019linearly,wang2018model,lee2020model,fresca2021comprehensive,gonzalez2018deep,fresca2021comprehensive,franco2023deep,NEURIPS2022_24f49b2a,kontolati2023learning,kim2020efficient}.  \cite{NEURIPS2022_24f49b2a} assumes that the output functions in operator learning are concentrated near a low-dimensional manifold, and proposes NOnlinear MAnifold Decoder (NOMAD) for the solution submanifold. Numerical experiments in \cite{NEURIPS2022_24f49b2a} demonstrate that nonlinear decoders significantly outperform linear decoders, when the output functions are indeed on a low-dimensional manifold.
In \cite{kontolati2023learning}, Auto-Encoders are used to extract the latent features for the inputs and outputs respectively, and DeepONet is applied on latent features for operator learning. Numerical experiments in \cite{kontolati2023learning} demonstrate improved predictive
accuracy when DeepONet is applied on latent features.

Despite the experimental success witnessed in Auto-Encoder-based neural networks for nonlinear model reduction, there is currently no established mathematical and statistical theory that can justify the heightened accuracy and reduced sample complexity achieved by these networks. Our paper aims to investigate this line of research through a comprehensive generalization error analysis of Auto-Encoder-based Neural Networks (AENet) within the context of nonlinear model reduction in operator learning. We present theoretical analysis that demonstrates the sample complexity of AENet depends on the intrinsic dimension of the model, rather than the dimension of its ambient space. This analysis provides a theoretical foundation for understanding how Auto-Encoder-based neural networks effectively exploit low-dimensional nonlinear structures in the realm of operator learning, offering a novel perspective on this subject.

\subsection*{Summary of our main results}

This paper explores the use of Auto-Encoder-based neural network (AENet) for operator learning in function
spaces, leveraging the Auto-Encoder-based nonlinear model reduction technique. 
Our goal is to achieve numerical success of nonlinear model reduction in comparison with linear model reduction methods. More importantly, as a novel part of this paper, we will establish a generalization error analysis in this context.

We explore AENet to handle the operator learning problems when the inputs are concentrated on a low-dimensional nonlinear manifold. Our algorithm has two stages. The first stage is to build an Auto-Encoder to learn the latent variable for the input. The second stage is to learn a transformation from the input latent variable to the output. The architecture of AENet is shown in Figure \ref{figaenet}(a). Furthermore, we provide a framework to analyze the generalization error and sample complexity of AENet.

Let $\cX$ and $\cY$ be two sets of functions in two Hilbert spaces  and $\Psi:\cX\rightarrow \cY$ be an unknown Lipschitz operator. Consider i.i.d. samples $\{u_i\}_{i=1}^{2n} \subset\cX $ and the noisy outputs 
$$\whv_i = v_i +\epsilon_i, \ \text{with} \ v=\Psi(u),$$ where the i.i.d. noise $\{\epsilon_i\}_{i=1}^{2n}$ is independent of the $u_i$'s. The functions $u\in \cX,v\in \cY$ are  discretized as $ \cSX(u)\in \RR^{D_1},\cSY(v)\in \RR^{D_2}$, where $\cSX$ and $\cSY$ are discretization operators for functions  in $\cX$ and $\cY$, respectively. Given the discretized data $\{\cSX(u_i),\cSY(\whv_i)\}_{i=1}^{2n}$, we aim to learn the operator $\Psi$. 

When the input $u\in \cX$ is concentrated on a low-dimensional nonlinear set parameterized by $d$ latent variables, we study AENet which learn the input latent variable and the operator in two stages.

{\bf Stage I:} We use $\{\cSX(u_i)\}_{i=1}^{n}$ to train an Auto-Encoder $(\EXN,\DXN)$ with $$\EXN:\RR^{D_1}\rightarrow\RR^{d} \ \text{ and }\ \DXN:\RR^{d}\rightarrow\RR^{D_1}$$ for the input. This Auto-Encoder gives rise to the input latent variable $\EXN\circ\cSX(u) \in \RR^d$.

{\bf Stage II:} We use $\{\cSX(u_i),\cSY(\whv_i)\}_{i=n+1}^{2n}$ to learn a  transformation from the input latent variable to the output:
\begin{equation}
		 \Gamma_{\rm NN}^n\in \argmin_{\Gamma'_{\rm NN}\in \cF_{\rm NN}^{\Gamma}} \frac{1}{n}\sum_{i=n+1}^{2n}\| \Gamma'_{\rm NN}\circ E_{\cX}^n \circ \cSX(u_i) -\cSY(\whv_i)\|_{\cSY}^2,
  \label{eq:gamma:loss}
	\end{equation}
 where $\|\cdot\|_{\cSY}$ is the discretized counterpart of the function norm in $\cY$.

Combining Stage I and Stage II gives rise to the operator estimate in the discretized space
$$\Phi^n_{\rm NN} =\Gamma_{\rm NN}^n \circ\EXN: \ \RR^{D_1} \rightarrow \RR^{D_2},$$
which transforms the discretized input function $\cSX(u)$  to the discretized output function $\cSY(v)$. 

 Numerical experiments are provided in Section \ref{secnum} to learn the solutions of nonlinear PDEs from various initial conditions. We consider the transport equation for transportation models, the Burgers' equation with viscosity $10^{-3}$ in fluid mechanics and the  Korteweg–De Vries (KdV) equation modeling waves on shallow water surfaces. Our experiments demonstrate that AENet significantly outperforms linear model reduction methods \citep{bhattacharya2020model}. AENet is effective in handling nonlinear structures in the input, and are robust to noise. 

 This paper  provides a solid mathematical and statistical estimation theory on the generalization error of AENet.
Our Theorem \ref{thm:generalization} shows that, the squared generalization error decays exponentially as the sample size $n$ increases, and the rate of decay depends on the intrinsic dimension $d$. Specifically, Theorem \ref{thm:generalization} proves the following upper bound on the squared generalization error:
\begin{equation}\EE_{\rm Data}\EE_{u} \|\Phi^n_{\rm NN}\circ \cSX(u) -\cSY\circ\Psi(u)\|^2_{\cSY}
\leq C(1+\sigma^2) n^{-\frac{1}{2+d}}\log^3 n, 
\label{eq:generalizationerrorintro}
\end{equation}
where $C$ is a constant depending on the model parameters, and $\sigma^2$ represents the variance of noise.
The contributions of  Theorem \ref{thm:generalization} are summarized below:

{\bf Leverage the dependence on intrinsic parameters:} This theory justifies the benefits of model reduction by AENet. The rate of convergence for the generalization error depends on the intrinsic dimension $d$, even though the unknown operator is between two infinite-dimensional function spaces.
To our best knowledge, this is the first statistical estimation theory on the generalization error of nonlinear model reduction by deep neural networks.

{\bf Robustness to noise:} The constant $C(1+\sigma^2)$ in \eqref{eq:generalizationerrorintro} is proportional to the variance of noise. Our result demonstrates that AENet is robust to noise. Moreover, AENet has a denoising effect as the sample size increases, since squared generalization error decreases to  $0$ as $n$ increases to $\infty$.

\textbf{Dependence on the interpolation error:}
In some applications, test functions are discretized on a different grid as training functions. We can interpolate the test function on the training grid, and evaluate the output. In Remark  \ref{remark.gene.interp}, we show that, in this case the squared generalization error has an additional term about the interpolation error.

\subsection*{Organization} This paper is organized as follows: We provide  preliminary definitions and discuss function discretization in Section \ref{sec:prelim}. We then  introduce the operator learning problem, explore nonlinear models and describe our AENet architecture in Section \ref{secaenet}.  Our main results, including the approximation theory and generalization error guarantees of AENet,  are presented in Section \ref{secerror}. Numerical experiments are detailed in Section \ref{secnum} and the proof of our main results is given in Section \ref{sec:proof}. Proofs of lemmas are postponed to Appendix \ref{applemma}.  Finally, we conclude our paper with discussions in Section \ref{secdiscussion}.

\section{Preliminaries and  discretization of functions }
\label{sec:prelim}

In this section, we delineate the notations utilized throughout this paper. Additionally, we define key concepts such as Lipschitz operators, the Minkowski dimension and ReLU networks. Furthermore, we provide details about the function spaces of interest and the discretization operators employed in discretizing functions.

\subsection{Notation}

We use bold letters to denote vectors, and capital letters to denote matrices.
For any vector $\xb \in \RR^{d}$, we denote its Euclidean norm by $\|\xb\|_2 =\sqrt{\sum_i |x_i|^2}$, its $\ell^\infty$ norm by $\|\xb\|_\infty =\sup_i |x_i|$, and its $\ell^1$ norm by $\|\xb\|_1 =\sum_i |x_i|$. We use $\|\cdot\|_0$ to denote the number of nonzero elements of its argument. We use $B_2^d(\xb,\delta)$ to represent the open Euclidean ball in $\RR^d$ centered at $\xb$ with radius $\delta$. Similarly,  $B_\infty^{d}(\xb,\delta)$ denotes the $L^\infty$ ball in $\RR^d$ centered at $\xb$ with radius $\delta$. We use $\#E$ to denote the cardinality of the set $E$ and $|E|$ to denote the volume of $E$. For a function $u:\Omega\rightarrow \RR$, its $L^p$ norm is $\|u\|_{L^p(\Omega)}:=\left(\int_{\Omega}|u(\xb)|^p d\xb \right)^{1/p}$ and its $L^\infty$ norm is $\|u\|_{L^{\infty}(\Omega)}:=\sup_{\xb\in\Omega}|u(\xb)|$. For a vector-valued function $\fb$ defined on $\Omega$, we denote $\|\fb\|_{\infty,\infty}=\sup_{\xb\in \Omega}\|\fb(\xb)\|_{\infty}$. Throughout the paper, we use letters with a tilde to denote their discretized counterpart,  letters with a subscript NN to denote networks, letters with a superscript $n$ to denote empirical estimations.

\subsection{Preliminaries}

\begin{definition}[Lipschitz operators]
\label{deflip}
An operator $\Theta:\cA\rightarrow \cB$ is Lipschitz if 
\[L_{\Theta} :=\sup_{u_1\neq u_2\in\cA} \frac{\|\Theta(u_1)-\Theta(u_2)\|_{\cB}}{\|u_1-u_2\|_{\cA}} <\infty,\]
where $L_{\Theta} $ is called the Lipschitz constant of $\Theta$. 
\end{definition}

\begin{definition}[Minkowski dimension]
Let $E\subset \RR^D$. For any $\varepsilon>0$, $\cN(\varepsilon,E,\|\cdot\|_\infty)$ denotes the fewest number of $\varepsilon$-balls that cover $E$ in terms of $\|\cdot\|_\infty$.
The (upper) Minkowski dimension of $E$ is defined as

\[d_{M}E:=\limsup_{\varepsilon\rightarrow 0^{+}} \frac{\log \cN(\varepsilon,E,\|\cdot\|_\infty)}{
\log(1/\varepsilon)}.\]

\end{definition}

The Minkowski dimension is also called the box-counting dimension. It describes how the box covering number  $\cN(\varepsilon,E,\|\cdot\|_\infty)$ scales with respect to the box side length $\varepsilon$. If $\cN(\varepsilon,E,\|\cdot\|_\infty)\approx C \varepsilon^{-d}$, then $d_{M}E=d$.
\\

\noindent 
{\bf Deep neural networks:}
We study the ReLU activated feedforward neural network (FNN):
\begin{equation}
	f_{\rm NN}(\xb)=  W_L\cdot\ReLU\big( W_{L-1}\cdots \ReLU(W_1\xb+\bbb_1) + \cdots +\bbb_{L-1}\big)+\bbb_L,	
 \label{eq.FNN.f}
\end{equation}
where $W_l$'s are weight matrices, $\bbb_l$'s are biases and $\ReLU(a)=\max\{a,0\}$ denotes the rectified linear unit (ReLU). We consider the following class of FNNs
\begin{align}
	&\cF_{\rm NN}(d_1,d_2,L,p,K,\kappa,M)=
 \Big\{f_{\rm NN}:\RR^{d_1}\rightarrow \RR^{d_2}| f_{\rm NN}(\xb) \mbox{ is in the form of (\ref{eq.FNN.f}) with } L \mbox{ layers,} \label{eq.FNN}\\
 &\quad \mbox{width bounded by } p,  \|f_{\rm NN}\|_{L^\infty}\leq M, \ \|W_l\|_{\infty,\infty}\leq \kappa, \|\bbb_l\|_{\infty}\leq \kappa,\  \textstyle \sum_{l=1}^L \|W_l\|_0+\|\bbb_l\|_0\leq K   \Big\},
	\nonumber
\end{align}
where
$\|W\|_{\infty,\infty}=\max_{i,j} |W_{i,j}|,\ \|\bbb\|_{\infty}=\max_i |b_i|
$ for any matrix $W$ and vector $\bbb$. 

\subsection{Function spaces and discretization}

We consider compact domains $\Omega_{\cX} \subset \RR^{d_{\Omega_{\cX}}}$ and $\Omega_{\cY} \subset \RR^{d_{\Omega_{\cY}}}$, and Hilbert spaces $L^2(\Omega_{\cX})$ and $L^2(\Omega_{\cY})$. 
The space 
$L^2(\Omega_{\cX}):=\{u:\Omega_{\cX} \rightarrow \RR:  \int_{\Omega_{\cX}}  |u(\xb)|^2d\xb < \infty\}
$
 is equipped with the inner product $
\textstyle \langle u_1,u_2\rangle:= \int_{\Omega_{\cX}}  {u_1(\xb)}u_2(\xb)d\xb,\ \forall  u_1,u_2 \in L^2(\Omega_{\cX}).
$
The norms of $L^2(\Omega_{\cX})$ and $L^2(\Omega_{\cY})$ are denoted by $\|\cdot\|_{\cX} = \|\cdot\|_{L^2(\OcX)}$ and $\|\cdot\|_{\cY} = \|\cdot\|_{L^2(\OcY)}$, respectively.
 \commentout{
\begin{equation}
\textstyle \langle u,v\rangle:= \int_{\Omega}  {u(\xb)}v(\xb)d\xb,\ \forall  u,v \in L^2(\Omega).
\label{eql2innerproduct}
\end{equation}}
Let $\cX\subset L^2(\Omega_{\cX})$ and $\cY \subset L^2(\Omega_{\cY})$. 
This paper considers differentiable input and output functions:
\begin{align}
&\cX \subset C^1(\Omega_{\cX}):=\{u\in L^2(\Omega_{\cX}):\ \nabla u \text{ is continuous}  \},
\label{eqcXdiff}
\\
& \cY \subset C^1(\Omega_{\cY}):=\{v\in L^2(\Omega_{\cY}): \ \nabla v \text{ is continuous} \},
\label{eqcYdiff}
\end{align}
and
\begin{equation}
{\sup_{u\in\cX}\sup_{\xb \in \OcX} \|\nabla u(\xb)\|_1} <\infty, \quad {\sup_{v\in \cY}\sup_{\yb \in \OcY} \|\nabla v(\yb)\|_1} <\infty.
\label{eqdiffbound}
\end{equation}

In applications, functions need to be  discretized. 
Let $\{\xb_i\}_{i=1}^{D_1} \subset \Omega_{\cX}$ and $\{\yb_j\}_{j=1}^{D_2} \subset\Omega_{\cY}$ be the  discretization grid on the $\Omega_{\cX}$ and $\Omega_{\cY}$ domain respectively. The discretization operator on $\cX$ and $\cY$ are
\begin{align*}
& \cS_{\cX} : \cX \rightarrow \RR^{D_1}, \ \text{s.t.} \ \cS_{\cX}(u) = \{u(\xb_i)\}_{i=1}^{D_1},
\\
& \cS_{\cY} : \cY \rightarrow \RR^{D_2}, \ \text{s.t.} \ \cS_{\cY}(v) = \{v(\yb_j)\}_{j=1}^{D_2}.
\end{align*}
This discretization operator $\cSX$ gives rise to an inner product and the induced norm on $\RR^{D_1}$ such that
\begin{equation}
   \langle\cSX(u_1),\cSX(u_2)\rangle_{\cSX} = \sum_{i=1}^{D_1}\omega_i u_1(\xb_i)u_2(\xb_i),
    \label{eqdiscreteinnerproduct}
\end{equation}
where $\{\omega_i:\ \omega_i>0\}_{i=1}^D
$ is given by a proper quadrature rule for the integral $\langle u_1,u_2\rangle$. Popular quadrature rules in numerical analysis include  the midpoint, trapezoidal, Simpson's rules, etc \citep{atkinson1991introduction}. The basic properties of the norms $\|\cdot \|_{\cSX}$ and $\|\cdot \|_{\cSY}$ are given in Appendix \ref{sec:empirical:property}.

For regular function sets $\cX$ and $\cY$ of practical interests, the convergence of Riemann integrals yields 
$\|\cSX(u)\|_{\cSX}^2  \approx \|u\|^2_{L^2(\Omega_{\cX})}$ for any  $u\in \cX$ and $\|\cSY(v)\|_{\cSY}^2  \approx \|v\|^2_{L^2(\Omega_{\cY})}$ for any  $v\in \cY$,
when the discretization grid is sufficiently fine. This motivates us to assume the following property:
\begin{assumption}
\label{assumption:samplinglip}
Suppose the function spaces $\cX$ and $\cY$ are sufficiently regular such that: there exist discretization operators $\cSX$ and $\cSY$ satisfying the  property:
\begin{align}
0.5\|u\|_{\cX}\le \|\cSX(u)\|_{\cSX} \le 2\|u\|_{\cX}, \quad 
0.5\|v\|_{\cY}\le \|\cSY(v)\|_{\cSY} \le 2\|v\|_{\cY}
\label{eqsamplinglip}
\end{align}
for all  functions $u\in \cX$ and $v\in \cY$.
\end{assumption}

Assumption \ref{assumption:samplinglip} is a weak assumption which holds for large classes of regular functions as long as the discretization grid is sufficiently fine. For simplicity, we consider $\Omega_{\cX} = [0,1]^{d_{\Omega_{\cX}}}$ and $\cX \subset C^1(\OcX) $. Suppose the grid points $\{\xb_i\}_{i=1}^{D_1}$ are on a uniform grid of  $[0,1]^{d_{\Omega_{\cX}}}$ with spacing $\Delta x$ and the quadrature rule in \eqref{eqdiscreteinnerproduct} is given by the Newton-Cotes formula where the integrand is approximated by splines. Piecewise constant, linear, quadratic approximations of the integrand gives rise to the Midpoint, Trapezoid and Simpson rules, respectively. Taking the Midpoint rule as an example, we can express
\begin{equation*}
    \|\cSX(u)\|_{\cSX} = \|I_{\rm const}\circ \cSX(u)\|_{L^2(\OcX)}, \text{ where } I_{\rm const}: \RR^{D_1}\rightarrow  L^2(\OcX)
\end{equation*}
is the piecewise constant interpolation operator, and $I_{\rm const}\circ \cSX(u)$ is the piecewise constant approximation of $u$.
As a result,
\begin{align*}
  \|u\|_{L^2(\OcX)} - \|I_{\rm const}\circ \cSX(u)-u\|_{L^2(\OcX)} \le  \|\cSX(u)\|_{\cSX}  \le \|u\|_{L^2(\OcX)} + \|I_{\rm const}\circ \cSX(u)-u\|_{L^2(\OcX)}.
\end{align*}
Assumption \ref{assumption:samplinglip} holds as long as $\|I_{\rm const}\circ \cSX(u)-u\|_{L^2(\OcX)} \le \frac 1 2 \|u\|_{L^2(\OcX)}$ uniformly for all functions $u \in \cX$. By Calculus, piecewise constant approximation of $u$ at a uniform grid with spacing $\Delta x$ gives rise to the error
\begin{equation*}
\|I_{\rm const}\circ \cSX(u)-u\|_{L^\infty(\OcX)} \le \Delta x \sup_{\xb \in \OcX} \|\nabla u(\xb)\|_1
\end{equation*}
where $\nabla u$ denotes the gradient of $u$, and $\|\nabla u(\xb)\|_1$ is the $\ell^1$ norm of the gradient vector $\nabla u(\xb)$.

If all functions in $\cX$ satisfy mild conditions such that
\begin{equation}
  \delta :=  \inf_{u \in \cX} \left( \frac{\frac 1 2 \|u\|_{L^2(\OcX)}}{\sup_{\xb \in \OcX} \|\nabla u(\xb)\|_1} \right) >0,
  \label{eqsamplingcondition}
\end{equation}
then the discretization operator $\cSX$ satisfies Assumption \ref{assumption:samplinglip} for all the function $u\in \cX$ as long as $\Delta x \le \delta$. Roughly speaking, the condition in \eqref{eqsamplingcondition} excludes  functions whose function norm is  too small, or whose derivative is too large. From the viewpoint of Fourier analysis, the condition in \eqref{eqsamplingcondition}  excludes infinitely oscillatory functions.

\begin{example}
\label{examplesampling}
Let $ A\ge a >0$ and
$$\cX = \left\{ \sum_{k=-N}^N a_k e^{2\pi i k x}: \ a\le |a_k| \le A\right\} \subset L^2([0,1]).$$
When the uniform sampling grid is sufficiently fine that 
\begin{equation}
\Delta x \le \frac{a\sqrt{2N+1}}{2\pi A N(N+1)},
\label{eqsamplingexample}
\end{equation}
then the discretization operator $\cSX$ satisfies Assumption \ref{assumption:samplinglip}.
\end{example}

Example \ref{examplesampling} is proved in Appendix \ref{applemma:examplesampling}.
In Example \ref{examplesampling}, the function set $\cX$ includes Fourier series up to frequency $N$. Assumption \ref{assumption:samplinglip} holds as long as the grid spacing is sufficiently small to resolve the resolution up to frequency $N$, as shown in \eqref{eqsamplingexample}. The larger $N$ is, the more oscillatory the functions in $\cX$ are, and therefore, $\Delta x$ needs to be smaller.

\section{Nonlinear model reduction by AENet}
\label{secaenet}
In this section, we will present the problem setup and the AENet architecture.

\subsection{Problem formulation}
\label{secproblem}

In this paper, we represent the unknown physical process by an operator $\Psi:\cX\rightarrow\cY$, where $\cX$ and $\cY$ are subsets of two separable Hilbert spaces $L^2(\Omega_{\cX})$ and $L^2(\Omega_{\cY})$ respectively. Our goal is to learn the operator $\Psi$ from the given  samples:  $\{u_i,\whv_i\}_{i=1}^{2n}$, where $u_i$ is an input of $\Psi$ and $\whv_i$ is the noisy output. In practice, the functions $u,v$ are discretized in the given data sets $\{\cSX(u_i),\cSY(\whv_i)\}_{i=1}^{2n}$.

\begin{setting}\label{setting1}
Let $\Omega_{\cX} \subset \RR^{d_{\Omega_{\cX}}}$ and $\Omega_{\cY} \subset \RR^{d_{\Omega_{\cY}}}$ be compact domains, and  $\cX \subset C^1(\Omega_{\cX})\subset L^2(\Omega_{\cX})$ and $\cY \subset C^1(\Omega_{\cY}) \subset L^2(\Omega_{\cY})$ such that \eqref{eqdiffbound} holds. Suppose the function sets $\cX$ and $\cY$ and the discretization operators $\cSX$ and $\cSY$ satisfy Assumption \ref{assumption:samplinglip}.   The unknown operator $\Psi: \cX \rightarrow \cY$ is  Lipschitz  with Lipschitz constant $L_{\Psi} >0$, and $\gamma$ is a probability measure on $\cX$.
 Suppose $\{u_i\}_{i=1}^{2n}$ are i.i.d. samples from $\gamma$ and the $\whv_i$'s are generated according to model:
	\begin{equation}
	v_i = \Psi(u_i) \text{ and }\   {\whv}_i = v_i +{\epsilon_i},
	\label{eq.v}
\end{equation}
where the ${\epsilon}_i$'s are i.i.d. samples from a probability measure on $\cY$, independently of the $u_i$'s. 
 The given data are 
 \begin{equation}
\cJ = \{\cSX(u_i),\cSY(\whv_i)\}_{i=1}^{2n},
\label{eqtrainingsample}
\end{equation}
where $\cSY(\whv_i)=\cSY(v_i)+\cSY(\epsilon_i)$.
\end{setting}

For simplicity, we denote the discretized functions as $\wtu = \cSX(u)$ and $\wtv = \cSY(v)$ for the rest of the paper.

\subsection{Low-dimensional nonlinear models}

Even though $L^2(\OcX)$ and $L^2(\OcY)$ are infinite-dimensional function spaces, the functions of practical interests often exhibit low-dimensional structures. The simplest low-dimensional model is the linear subspace model. However, a large amount of functions in real-world applications exhibit nonlinear structures. For example, functions generated from translations or rotations have a nonlinear dependence on few parameters  \citep{tenenbaum2000global,roweis2000nonlinear,coifman2005geometric}, which motivates us to consider functions with a low-dimensional nonlinear parameterization.
\begin{assumption}
\label{assumptiong:globalparametrization}
In Setting \ref{setting1}, the probability measure $\gamma$ is supported on a low-dimensional set $\cM\subset \cX$  such that: There exist invertible Lipschitz maps 
$$\fb:\cM\rightarrow [-1,1]^d,\ \text{and } \ \gb:[-1,1]^d\rightarrow \cM$$ such that 		$
			u=\gb\circ \fb(u)
		$
		for any $u\in\cM$.
 The Lipchitz constants $\fb$ and $\gb$ are  $L_{\fb}>0$ and $L_{\gb}>0$ respectively, such that
 \begin{align*}
		\|\fb(u_1)-\fb(u_2)\|_2 &\le L_{\fb} \|u_1-u_2\|_{\cX}, \quad 
		\|\gb(\zb_1)-\gb(\zb_2)\|_{\cX} \le L_{\gb} \|\zb_1-\zb_2\|_2
		\end{align*}
		for any $u_1,u_2 \in \cM$, and $\zb_1,\zb_2\in[-1,1]^d$. 
		Additionally, there exists $c_1>0$ such that
	\begin{equation}	
  \cM \subset \{u:\ \|u\|_{L^\infty(\Omega_{\cX})} \le c_1\|u\|_{\cX}\}, \label{assumptiong:globalparametrizationeq1}
		\end{equation}
  and $R_{\cY}:=\sup_{u\in\cM}\{\|\Psi(u)\|_{L^\infty(\Omega_{\cY})}\}<\infty.$
  \vquarter
\end{assumption}

Assumption \ref{assumptiong:globalparametrization} says that, even though the input $u$ is in the infinite-dimensional space, it can be parameterized by a $d$-dimensional latent variable. The intrinsic dimension of the inputs is $d$. Assumption \ref{assumptiong:globalparametrization} includes linear and nonlinear models since $\fb$ and $\gb$ can be linear and nonlinear maps. The condition in \eqref{assumptiong:globalparametrizationeq1} is a mild assumption excluding the case that the large values of $u$ concentrate at a set with a small Lebesgue measure. Assumption \ref{assumptiong:globalparametrization} implies that  
$
    R_{\cX}:=\sup_{u\in\cM} \|u\|_{L^\infty(\Omega_{\cX})} <\infty.
$
Assumption \ref{assumptiong:globalparametrization} also implies a low-dimensional parameterization of $\cSX(u)$.
We denote the range of  $\cM$ under $\cSX$ by
$$\wtcM:= \cSX(\cM)=\{\cSX(u):\ u \in \cM\} \subset \RR^{D_1}.$$

\begin{lemma}
\label{lemma:f1g1lip}
In Setting \ref{setting1} and under Assumptions \ref{assumption:samplinglip} and \ref{assumptiong:globalparametrization},
every point in $\wtcM$ exhibits the low-dimensional parameterization: 
$
\widetilde{\fb}: \ \wtcM \subset \RR^{D_1} \rightarrow [-1,1]^d,$
 such that $\widetilde{\fb} ( \cSX(u)):= \fb(u)$ and $
  \widetilde{\gb}: \  [-1,1]^d\rightarrow \wtcM \subset \RR^{D_1} ,$ such that $  \widetilde{\gb}(\theta):=\cSX\circ \gb(\theta)
$
which guarantees
\begin{equation}
\label{lemma:f1g1lipeq1}
\widetilde{\gb} \circ \widetilde{\fb} (\cSX(u)) = \cSX(u), \ \forall u \in \cM.
\end{equation}
$\widetilde{\fb}$ and $\widetilde{\gb}$ are invertible and Lipschitz with Lipchitz constants bounded by $2L_{\fb}$ and $2L_{\gb}$ respectively. 
  \vquarter
\end{lemma}
Lemma \ref{lemma:f1g1lip} is proved in Appendix \ref{applemma:f1g1lipproof}. The low-dimensional parameterizations in Assumption \ref{assumptiong:globalparametrization} and Lemma \ref{lemma:f1g1lip} motivate us to perform nonlinear dimension reductions of $\wtcM$, shown in Figure \ref{figaenet}(b).
Another advantage is that $\wtcM$ is a $d$-dimensional manifold in $\RR^{D_1}$.
The following lemma shows that $\wtcM$ has  Minkowski dimension no more than $ d$ (proof in Appendix \ref{applemma:tildemmindimproof}). 
\begin{lemma}
\label{lemma:tildemmindim}
In Setting \ref{setting1} and under Assumptions \ref{assumption:samplinglip} and \ref{assumptiong:globalparametrization}, the Minkowski dimension of $\wtcM$ is no more than $d$:
$$d_M \wtcM \le d.$$
\vquarter
\end{lemma}

\begin{figure*}[t!]
\vspace{-0.4cm}
\centering
\subfigure[AENet architecture]{
\includegraphics[height = 2.8cm, width=6cm]{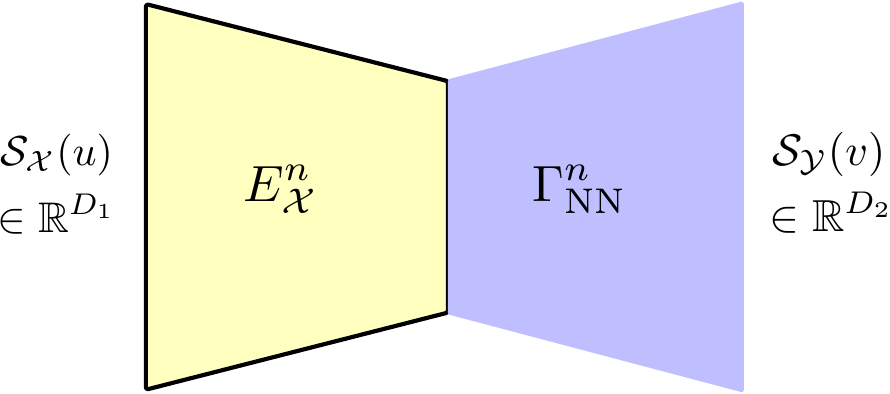}}
\hspace{0.2cm}
\subfigure[A transformation flow chart]{
\includegraphics[trim={0 0 0 0.8cm},clip,scale=0.38]{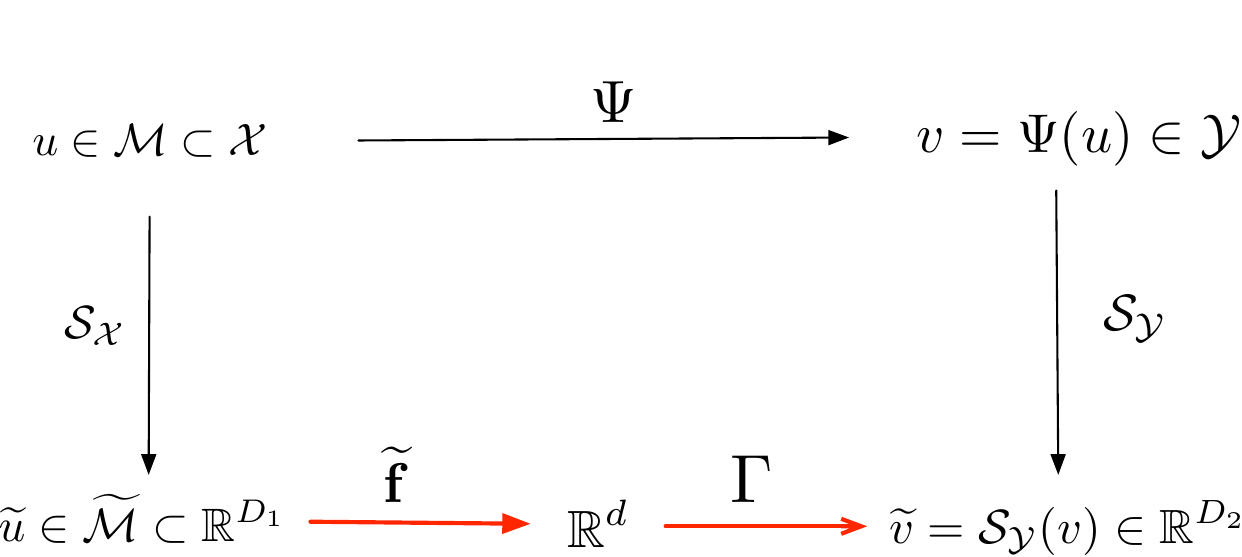}
}
\vspace*{-0.2cm}
\caption{An illustration of the AENet architecture and the transformation flow chart. The oracle transformation $\Phi$ has a dimension reduction component $\widetilde\fb$ and a forward transformation component $\Gamma$. These two components are marked in red. }
\vspace*{-0.5cm}
\label{figaenet}
\end{figure*}

\subsection{AENet}

Since functions are always discretized in numerical simulations, in order to learn the operator $\Psi$, it is sufficient to learn the transformation $\Phi$  on the discretized objects between $\wtu :=\cSX(u)\in \RR^{D_1}$ and $\wtv:=\cSY(v)\in \RR^{D_2}$. Specifically, 
$\Phi: \RR^{D_1} \rightarrow \RR^{D_2}$ is the oracle transformation  such that $ \ \Phi\circ\cSX(u) = \cSY \circ \Psi(u), \ \forall u\in \cX.$ In some applications, the training and test data are sampled on different grids. We will discuss interpolation and the discretization-invariant evaluation for the test data at the end of Section \ref{secerror}.

In this paper, we study AENet which learns the input latent variable by an Auto-Encoder, and then learns the transformation from this input latent variable to the output. The architecture of AENet is demonstrated in Figure \ref{figaenet} (a).
AENet aims to approximate the oracle transformation 
$$\Phi = \cSY\circ\Psi\circ\gb \circ\widetilde{\fb}: \wtcM \subset \RR^{D_1} \rightarrow \RR^{D_2}$$
 such that
 $$
\Phi(\cSX(u)) = \cSY(v).
$$
The oracle transformation $\Phi$ has two components shown in Figure \ref{figaenet}(b):
\begin{itemize} 
\item \textit{A dimension reduction component:} $\widetilde{\fb}: \RR^{D_1} \rightarrow \RR^d;$
\item \textit{A forward transformation component:} 
\begin{equation}
\Gamma:=\cSY\circ\Psi\circ\gb: \RR^{d} \rightarrow \RR^{D_2}.
\label{eqGamma}
\end{equation}
\end{itemize}

We propose to learn AENet  in two stages. Given the training data $\cJ=\{\wtu_i,\cSY(\whv_i)\}_{i=1}^{2n}$ in \eqref{eqtrainingsample} with $\wtu_i = \cSX(u_i)$,
 we split the data into two subsets $\cJ_1=\{\wtu_i,\cSY(\whv_i)\}_{i=1}^{n}$ and $\cJ_2=\{\wtu_i,\cSY(\whv_i)\}_{i=n+1}^{2n}$ (data can be  split unevenly as well), where $\cJ_1$ is used to build the Auto-Encoder for the input space and $\cJ_2$ is used to learn the transformation $\Gamma$ from the input latent variable to the output.

{\bf Stage I:} Based on the inputs $\{\widetilde{u}_i\}_{i=1}^n$  in $\cJ_1$, we learn an Auto-Encoder for the input space. 
The encoder $E_{\cX}^n: \wtcM\rightarrow \RR^{d}$ and the corresponding decoder $D_{\cX}^n:\RR^{d}\rightarrow \RR^{D_1}$ are given by the minimizers of the empirical mean squared loss
\begin{equation}
	 (D_{\cX}^n,E_{\cX}^n)=
 \argmin_{{\DX\in \cF_{\rm NN}^{\DX}, \EX\in \cF_{\rm NN}^{\EX}}} \frac{1}{n}\sum_{i=1}^n \|\wtu_i-\DX\circ \EX( \wtu_i)\|_{\cSX}^2,
	\label{eq:autoencoder:loss}
\end{equation}
with proper network architectures of $\cF_{\rm NN}^{\EX}$ and $\cF_{\rm NN}^{\DX}$.
The Auto-Encoder in Stage I yields the input latent variable
$ \EXN(u)\in \RR^{d}.$
Stage I of AENet is represented by the yellow part of Figure \ref{figaenet} (a).

{\bf Stage II:} We next learn a transformation $\Gamma_{\rm NN}$ between the input latent variable $\EXN(\wtu)$ and the output $\cSY(v)$ using the data in $\cJ_2$ by solving the  optimization problem in \eqref{eq:gamma:loss} to obtain $\Gamma_{\rm NN}^n$.
Stage II of AENet is indicated by the blue part of Figure \ref{figaenet} (a).

Finally, the unknown transformation $\Phi$ is estimated by 
$$\Phi^n_{\rm NN} =\Gamma_{\rm NN}^n \circ\EXN.$$
The performance of AENet can be measured by the squared generalization error 
$$\EE_{\cJ}\EE_{u\sim \gamma} \|\Phi^n_{\rm NN}\circ \cSX(u) -\cSY\circ\Psi(u)\|^2_{\cSY},
$$
where $\EE_{u\sim \gamma}$ is the expectation over the test sample $u\sim \gamma$, and $\EE_{\cJ}$ is the expectation over the joint distribution of training samples. 

\section{Main results on approximation and generalization errors}
\label{secerror}
 
We will state our main theoretical results on approximation and generalization errors in this section and defer detailed proof in Section~\ref{sec:proof}.  Our results show that AENet can efficiently learn the input latent variables and the operator with the sample complexity depending on the intrinsic dimension of the model.

\subsection{Approximation theory of AENet}

Our first result is on the approximation theory of AENet. We show that, the oracle transformation $\Phi$, its reduction component $\widetilde{\fb}$ and transformation component $\Gamma$ can be well approximated by  neural networks with proper  architectures.

\begin{theorem} \label{thm:approximation}
 In Setting \ref{setting1}, suppose  Assumptions  \ref{assumption:samplinglip} and \ref{assumptiong:globalparametrization} hold.
\begin{itemize}[noitemsep,topsep=0pt,leftmargin=*] 
    \item[(i)]  For any $\varepsilon_1\in (0,1/4),\varepsilon_2\in(0,1)$, chose the network architectures $ 
    \cF_{\rm NN}^{E_{\cX}}=\cF_{\rm NN}(D_1,d,L_{1},p_{1},K_{1},\kappa_{1},M_{1})$ and $\cF_{\rm NN}^{D_{\cX}}=\cF_{\rm NN}(d,D_1,L_{2},p_{2},K_{2},\kappa_{2},M_{2})$ with parameters
    \begin{align*}
    & L_{1}=O(\log \varepsilon_1^{-1}),\   p_{1}=O(\varepsilon_1^{-d}),\ 
    K_{1}=O(\varepsilon_1^{-d}),\   \kappa_{1}=O(\varepsilon_1^{-7}),  \ M_{1}=1, \\
        & L_{2}=O(\log \varepsilon_2^{-1}),\   p_{2}=O(\varepsilon_2^{-d}), \  K_{2}=O(\varepsilon_2^{-d}\log \varepsilon_2^{-1}),\  \kappa_{2}= O(\varepsilon_2^{-1}),\ M_{2}=R_{\cX}.
    \end{align*}
    
    There exists $\widetilde{\fb}_{\rm NN}\in \cF_{\rm NN}^{E_{\cX}}$ and $ \widetilde{\gb}_{\rm NN}\in \cF_{\rm NN}^{D_{\cX}}$ satisfying 
    \begin{align*}
        &\textstyle \sup_{\widetilde{u}\in \widetilde{\cM}}\|\widetilde{\fb}_{\rm NN}(\widetilde{u})- \widetilde{\fb}(\widetilde{u})\|_{\infty} \leq \varepsilon_1,
        \\
        & \textstyle \sup_{\zb\in[-1,1]^d}\|\widetilde{\gb}_{\rm NN}(\zb)- \widetilde{\gb}(\zb)\|_{\infty} \leq \varepsilon_2,\\ 
      & \textstyle
        \sup_{\widetilde{u}\in \widetilde{\cM}}\|\widetilde{\gb}_{\rm NN}\circ\widetilde{\fb}_{\rm NN}(\widetilde{u})- \widetilde{\gb}\circ\widetilde{\fb}(\widetilde{u})\|_{\infty} \leq \varepsilon_2+ 2\sqrt{d}L_{\gb}\varepsilon_1.
    \end{align*} 
    The constant hidden in $O(\cdot)$ depends on $d,L_{\fb},L_{\gb},R_{\cX}$ 
    and is polynomial in $D_1$.

    \item[(ii)] For any $\varepsilon_3\in(0,1)$, there exists a network $\Gamma_{\rm NN} \in\cF_{\rm NN}(d,D_2,L_{3},p_{3},K_{3},\kappa_{3},M_{3})$ with
    \begin{align*}
        &L_{3}=O(\log \varepsilon_3^{-1}), \ p_{3}=O(\varepsilon_3^{-d}), \ K_{3}=O(\varepsilon_3^{-d}\log \varepsilon_3^{-1}), \ \kappa_{3}= O(\varepsilon_3^{-1}),\  M_{3}=R_{\cY}
    \end{align*}
    satisfying 
    $$
      \textstyle  \sup_{\zb \in [-1,1]^d}\|\Gamma_{\rm NN}(\zb)- \Gamma(\zb)\|_{\infty}\leq \varepsilon_3.
    $$
    The constant hidden in $O(\cdot)$ depends on $d,L_{\Psi},L_{\gb}$ and is linear in $D_2$.
      \vquarter
\end{itemize}

\end{theorem}

Theorem \ref{thm:approximation} provides a construction of three neural networks to approximate $\widetilde{\fb},\widetilde{\gb},\Gamma$ with accuracy $\varepsilon_1,\varepsilon_2,\varepsilon_3$ respectively.
A proper choice of $\varepsilon_1$ and $\varepsilon_3$ yields the following approximation result on the oracle transformation $\Phi$:
\begin{corollary}\label{coro:approximation}
Under the assumptions in Theorem \ref{thm:approximation}, for any $\varepsilon\in(0,1)$, set $\varepsilon_1=\frac{\varepsilon}{4\sqrt{d}L_{\Psi}L_{\gb}}, \varepsilon_3=\frac{\varepsilon}{2}$ in Theorem \ref{thm:approximation} and denote the network $\Phi_{\rm NN}= \Gamma_{\rm NN}\circ \widetilde{\fb}_{\rm NN}.$ Then we have
$$
 \textstyle  \sup_{\wtu \in \wtcM} \|\Phi_{\rm NN}(\widetilde{u})-\Phi(\widetilde{u})\|_{\infty}\leq \varepsilon.
$$
  \vquarter
\end{corollary}

 Theorem \ref{thm:approximation} and Corollary \ref{coro:approximation} are proved in Section \ref{proof:thm:approximation} and  \ref{proof:coro:approximation} respectively. 
Corollary \ref{coro:approximation} provides an explicit construction of neural networks to approximate the oracle transformation $\Phi$. It demonstrates that AENet with properly chosen parameters has the representation power for the oracle transformation with an arbitrary accuracy $\varepsilon$.

\begin{remark}
    \label{remark.approx.interp}
In Corollary \ref{coro:approximation}, the input is a  discretization by $\cS_{\cX}$, and $\Phi_{\rm NN}$ is constructed for inputs discretized by $\cS_{\cX}$.  By interpolating functions on $\Omega_{\cX}$, we can apply this network to functions discretized on a different grid. 
Suppose a new input $u$ is discretized by another discretization operator $\cSX'(u)\in \RR^{D'_1}$, where the  grid points of $\cSX$ and $\cSX'$ are different. Then $\cSX'(u)$ cannot be directly passed to $\Phi_{\rm NN}$. Let $P_{{\rm intp},\cX}$ be an interpolation operator from the grid to the whole domain $\Omega_{\cX}$. For the input $\cSX'(u)$, we can interpolate it by $P_{{\rm intp},\cX}$ and then discretize it by $\cS_{\cX}$ to obtain $\cSX(P_{{\rm intp},\cX}(\cSX'(u))) \in \RR^{D_1}$. Based on this setting and under the same condition of Corollary \ref{coro:approximation}, we have (see details in Appendix \ref{sec:approx.interpolation.x})
\begin{align}
    &\sup_{u \in \cM} \|\Phi_{\rm NN}\circ \cSX(P_{{\rm intp},\cX}(\cSX'(u)))-\Phi(\cSX(u))\|_{\infty}\leq \varepsilon+ \sup_{u \in \cM} \|\Phi_{\rm NN}\circ \cSX(P_{{\rm intp},\cX}(\cSX'(u)))-\Phi_{\rm NN}\circ \cSX(u)\|_{\infty},
    \label{eq.approx.inter.x}
\end{align}
where the first term captures the network approximation error, the second term arises from the interpolation error on $\Omega_{\cX}$.
\end{remark}

\subsection{Generalization error of AENet}

Our second main result is on the generalization error of AENet with i.i.d. sub-Gaussian noise. 
\begin{assumption}
\label{assumption:noise}
Let $\cSY$ be the discretization operator in $\cY$ under Setting \ref{setting1}. 
The noise distribution $\mu$ in Setting \ref{setting1} satisfies: For any sample $\epsilon \sim \mu$ (a random function defined on $\Omega_{\cY}$), $\{(\cSY(\epsilon))_k\}_{k=1}^{D_2}$ with $(\cSY(\epsilon))_k=\epsilon(\yb_k)$ are i.i.d. sub-Gaussian with $\EE \left[(\cSY(\epsilon))_{k}\right]=0$ and variance proxy $\sigma^2$ for $k=1,\ldots,D_2$. 

\end{assumption}

\begin{theorem} \label{thm:generalization}
Consider Setting \ref{setting1} and suppose Assumptions \ref{assumption:samplinglip}, \ref{assumptiong:globalparametrization} and \ref{assumption:noise} hold. 
In Stage I, set the network architectures $\cF_{\rm NN}^{E_{\cX}}=\cF_{\rm NN}(D_1,d,L_1,p_1,K_1,\kappa_1,M_1)$ and $\cF_{\rm NN}^{D_{\cX}}=\cF_{\rm NN}(d,D_1,L_2,p_2,K_2,\kappa_2,M_2)$ with parameters
\begin{align}
    &L_1=O(\log n), \ p_1=O(n^{\frac{d}{2+d}}),\  K_1=O(n^{\frac{d}{2+d}}), \ \kappa_1=O(n^{\frac{7}{2+d}}),\  M_1=1,
    \label{eq:gene:para:1}
    \\
    &L_2=O(\log n), \ p_2=O(n^{\frac{d}{2+d}}),\  K_2=O(n^{\frac{d}{2+d}} \log n), \ \kappa_2=O(n^{\frac{1}{2+d}}),\  M_2=R_{\cX}.
    \label{eq:gene:para:2}
\end{align}
In Stage II, set the network architecture $\cF_{\rm NN}^{\Gamma}=\cF_{\rm NN}(d,D_2,L_3,p_3,K_3,\kappa_3,M_3)$
with
\begin{align}
    &L_3=O(\log n), \ p_3=O(n^{\frac{d}{2(2+d)}}),\  K_3=O(n^{\frac{d}{2(2+d)}} \log n), \ \kappa_3=O(n^{\frac{7}{2(2+d)}}),\  M_3=R_{\cY}.
    \label{eq:gene:para:3}
\end{align}
Let $E_{\cX}^n,D_{\cX}^n$ be the empirical minimizer of (\ref{eq:autoencoder:loss}) in Stage I, and $\Gamma_{\rm NN}^n$ be the empirical minimizer of (\ref{eq:gamma:loss}) in Stage II. For $n>4^{2+d}$, the squared generalization error of $\Phi^n_{\rm NN}=\Gamma_{\rm NN}^n\circ E_{\cX}^n$ satisfies
\begin{align}
    &\EE_{\cJ}\EE_{u\sim \gamma} \|\Phi^n_{\rm NN}\circ \cSX(u) -\cSY\circ\Psi(u)\|^2_{\cSY}
    \leq  C(1+\sigma^2) n^{-\frac{1}{2+d}}\log^3 n,
    \label{eq:generalizationerr}
\end{align}
for some $C>0$. The constant hidden in $O(\cdot)$ and $C$ depend on $d,L_{\fb},\ L_{\gb},\ L_{\Psi},\ R_{\cX},\ R_{\cY},\ |\Omega_{\cX}|$, $|\Omega_{\cY}|$ and is polynomial in $D_1$ and is linear in $D_2$.
\end{theorem}
Theorem \ref{thm:generalization} is proved in Section \ref{proof:generalization}. Its contributions  are summarized in the introduction.

\begin{remark}
\label{remark.gene.interp}
In Theorem \ref{thm:generalization},  $\Phi_{\rm NN}^n$ is trained on the inputs discretized by $\cS_{\cX}$. 
    The result of Theorem \ref{thm:generalization} can be applied to a new input discretized on a different grid, as considered in Remark \ref{remark.approx.interp}.
    Under the condition of Theorem \ref{thm:generalization}, we have (see details in Appendix \ref{sec:gene.interpolation.x})
    \begin{align}
    &\EE_{\cJ}\EE_{u\sim \gamma} \|\Phi^n_{\rm NN}\circ \cSX(P_{{\rm intp},\cX}(\cSX'(u))) -\cSY\circ\Psi(u)\|^2_{\cSY} \nonumber\\
   \leq & C(1+\sigma^2) n^{-\frac{1}{2+d}}\log^3 n +   \EE_{\cJ}\EE_{u\sim \gamma} \|\Phi^n_{\rm NN}\circ \cSX(P_{{\rm intp},\cX}(\cSX'(u))) - \Phi^n_{\rm NN}\circ \cSX(u) \|^2_{\cSY},
   \label{eq.gene.inter.x}
\end{align}
where the first term captures the network estimation error and the second term arises from the interpolation error on $\Omega_{\cX}$.
\end{remark}

\section{Numerical experiments}
\label{secnum}

We next present several numerical experiments to demonstrate the efficacy of AENet. We consider the solution operators of the linear transport equation, the nonlinear viscous Burgers' equation, and the Korteweg-de Vries (KdV) equation. 

For all examples, we use a 512 dimensional equally spaced grid for the spatial domain. The networks are all fully connected feed-forward ReLU networks as in our theory. We trained every neural network for 500 epochs with the Adam optimization algorithm using the MSE loss, a learning rate of $10^{-3}$, and a batch size of 64. For training of the neural networks, all data (input and outputs function values) were scaled down to fit into the range $[-1, 1]$. 

All training was done with 2000 training samples and 500 test samples, except for the training involved in Figure \ref{fig:transport:generr}, Figure \ref{fig:burgers:generr}, and Figure \ref{fig:kdvexperiment:generr} where the training sample varies.

For all examples, the input data has a low dimensional nonlinear structure. Indeed, the input data matrix (by stacking the initial conditions) for all examples we consider has slowly decaying singular values (see Figure \ref{fig:transportsvd}, \ref{fig:burgerssvd} and \ref{fig:kdvsvd}), which indicates the shortcomings of using a linear encoder and the necessity of using a nonlinear encoder as in AENet. Additional plots showing the nonlinearity of the data can be found in Figure \ref{fig:transportcomponents}, \ref{fig:burgerscomponents}, and \ref{fig:kdvcomponents}. 

We compare AENet with two methods involving dimension reduction and neural networks. PCANet refers to the method in \cite{bhattacharya2020model}, which consists of a PCA encoder for the input, a PCA decoder for the output, and a neural network in between. We also consider  DeepONet \citep{lu2021learning} implemented with the DeepXDE package \citep{lu2021deepxde}, a popular method for operator learning that also involves a dimension reduction component (i.e. the branch net). In DeepONet, we take  the output dimension of the branch/trunk net as the reduced dimension.

For all examples, we implement the Auto-Encoder for AENet with layer widths 500, 500, 500, 500, $d_{ae}$, 500, 500, 500, and we implement the operator neural network for AENet and PCANet with layer widths 500, 500, 500. We use 40 dimensional PCA for the output in PCANet. We use a simple unstacked DeepONet.

\subsection{Transport equation}
\label{sec:transport}

We consider the linear transport equation given by
\begin{align}
\label{eq:transport}
\begin{aligned}
u_t = -u_x, \quad &x \in [0, 1], t \in [0, 0.3]
\end{aligned}
\end{align}
with zero Dirichlet boundary condition and initial condition $u(x, 0) = g(x), x \in [0, 1]$. 
We seek to approximate the operator $\Psi$ that takes $g(x) = u(x, 0)$ as input and outputs $u(x, 0.3)$ from the solution of \eqref{eq:transport} at $t=0.3$. We consider the weak version of this PDE, allowing us to consider an $g$ that is not differentiable everywhere. Note that the analytic solution to this equation is $u(x, t) = g(x - t)$.

Let $\sigma(x) = \max(x, 0)$, and fix $\epsilon = 0.05$. For any $\alpha$ and $t$, define the ``hat'' function 
\[H_{\alpha, t} (x) = \frac{2\alpha}{\epsilon}\left(\sigma(x - t) - 2\sigma\left(x - t - \frac{\epsilon}{2}\right) + \sigma(x - t - \epsilon)\right).\]
Let $a \in [1, 4]$ and $h \in [0, 1]$. We define the ``two-hat'' function 
\begin{equation}
g_{a, h}(x) = H_{a, 0.1}(x) + H_{2.5, 0.2 + 0.1h}(x).
\label{eq:ftransport}
\end{equation}
Our sampling measure $\gamma_{transport}$ is defined on $\cM := \{ g_{a, h} : a \in [1, 4], h \in [0, 1] \}$ by sampling $a, h$ uniformly and then constructing $g_{a, h}$. 

Figure \ref{fig:transportsvd} plots the singular values in descending order of the input data sampled from $\gamma_{transport}$. The slow decay of the singular values indicates that a non-linear encoder would be a better choice than a linear encoder for this problem. Figure \ref{fig:transportcomponents} further shows the non-linearity of the data, when we project the data to the top $2$ principal components. Figure \ref{fig:transportcomponents3} and \ref{fig:transportpca6} further show the  projections of this data set to the 1st-6th principal components. This data set is nonlinearly parametrized by 2 intrinsic parameters, but the top 2 principal components are not sufficient to represent the data, as shown in Figure \ref{fig:transportcomponents}. Figure \ref{fig:transportcomponents3} shows that the top 3 linear principal components yield a better representation of the data, since the coloring by $a$ and $h$ is well recognized.

\begin{figure}[h]
\centering
\subfigure[Singular values]{
    \includegraphics[height=3.5cm]{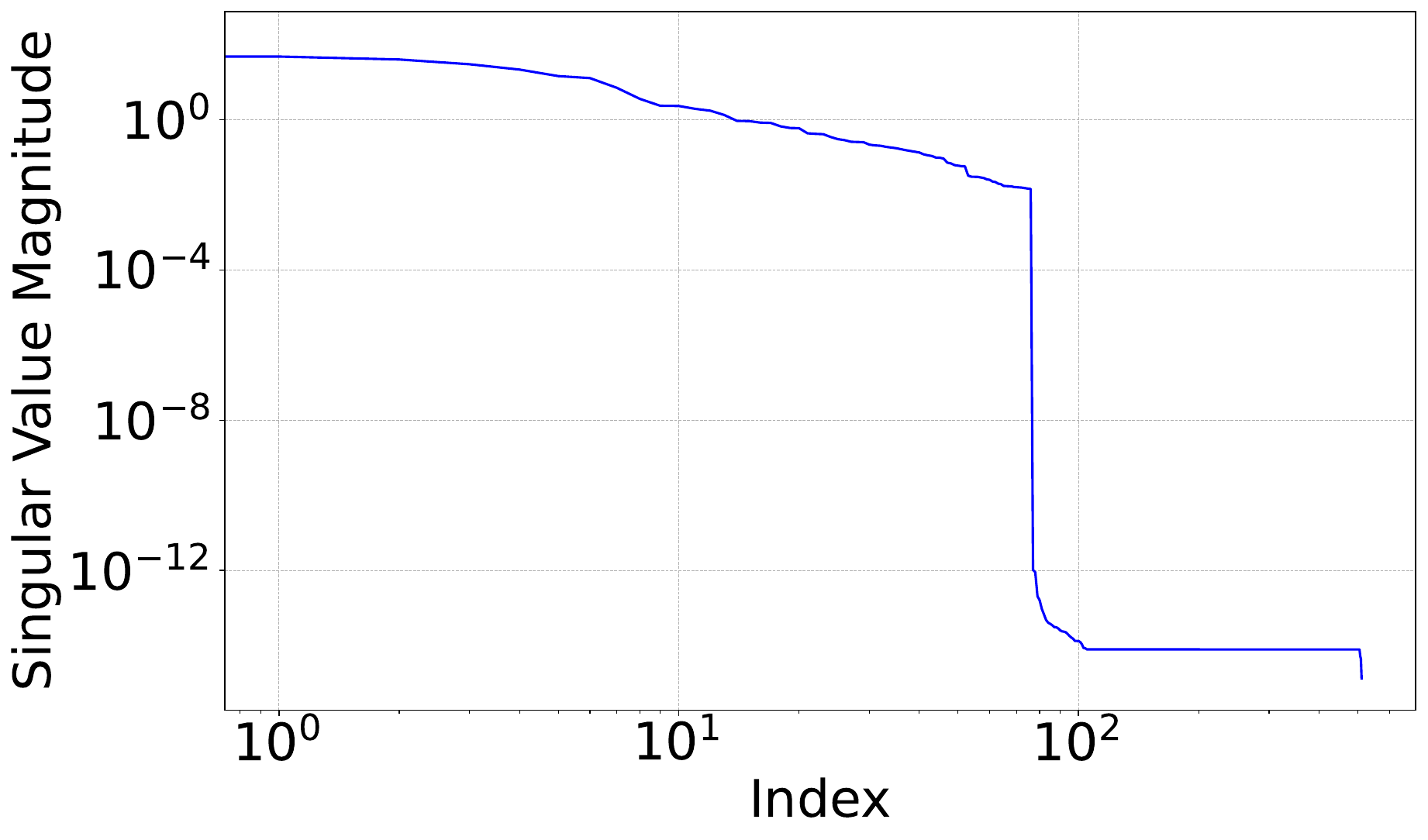}
    \label{fig:transportsvd}
}
\subfigure[Projection to top 2 principal components]{
    \includegraphics[height=4cm]{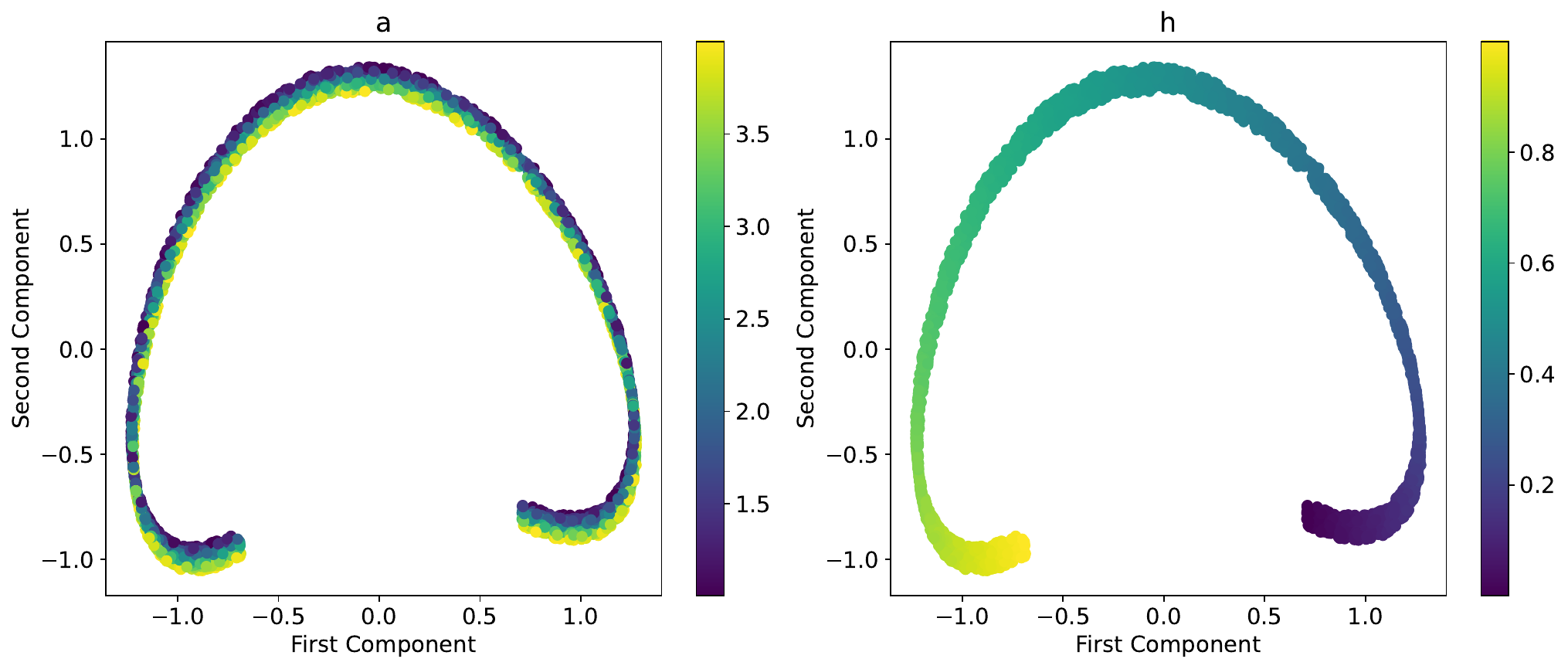}
    \label{fig:transportcomponents}
}
\subfigure[Projection to top 3 principal components]{
    \includegraphics[height=3.5cm]{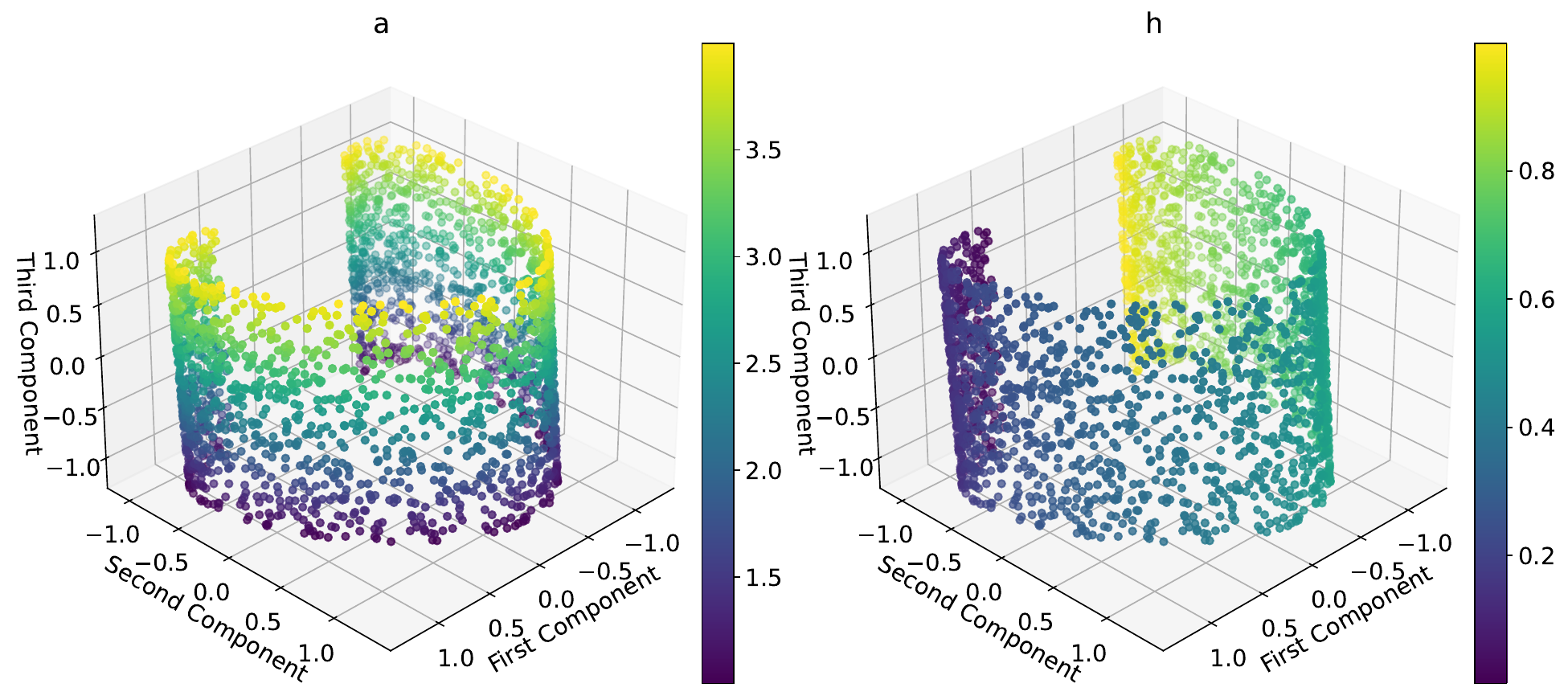}
    \label{fig:transportcomponents3}
}
\subfigure[Projection to 4th-6th principal components]{
    \includegraphics[height=3.5cm]{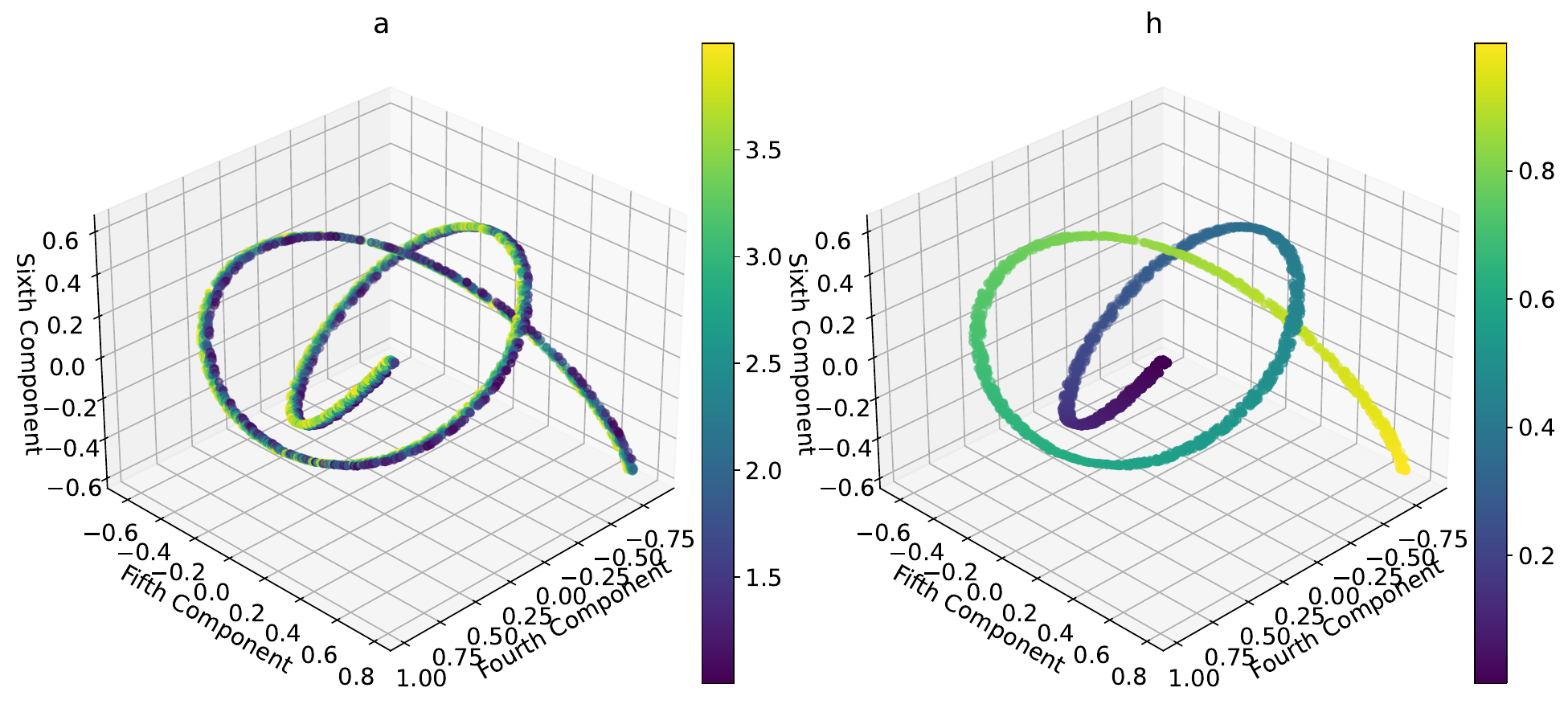}
    \label{fig:transportpca6}
}
\caption{Nonlinearity of the initial conditions for the transport equation. (a) shows the singular values of the data matrix. (b) shows the projection of data to the top 2 principal components and (c) shows the projection to top 3 principal components. (d) shows the projection to the 4th-6th principal components. In (b), (c) and (d), the projections are colored according to the $a$ parameter in the left subplots and according to the $h$ parameter in the right subplots.}
\label{fig:transportnonlinear}
\end{figure}

We then use the nonlinear Auto-Encoder for a nonlinear dimension reduction of the data. Figure \ref{fig:transport:project} shows the latent features given by the Auto-Encoder with reduced dimension $2$.  The intrinsic parameters $a$ and $h$ are well represented in the latent space.

\begin{figure}
    \centering
    \includegraphics[height=4.5cm]{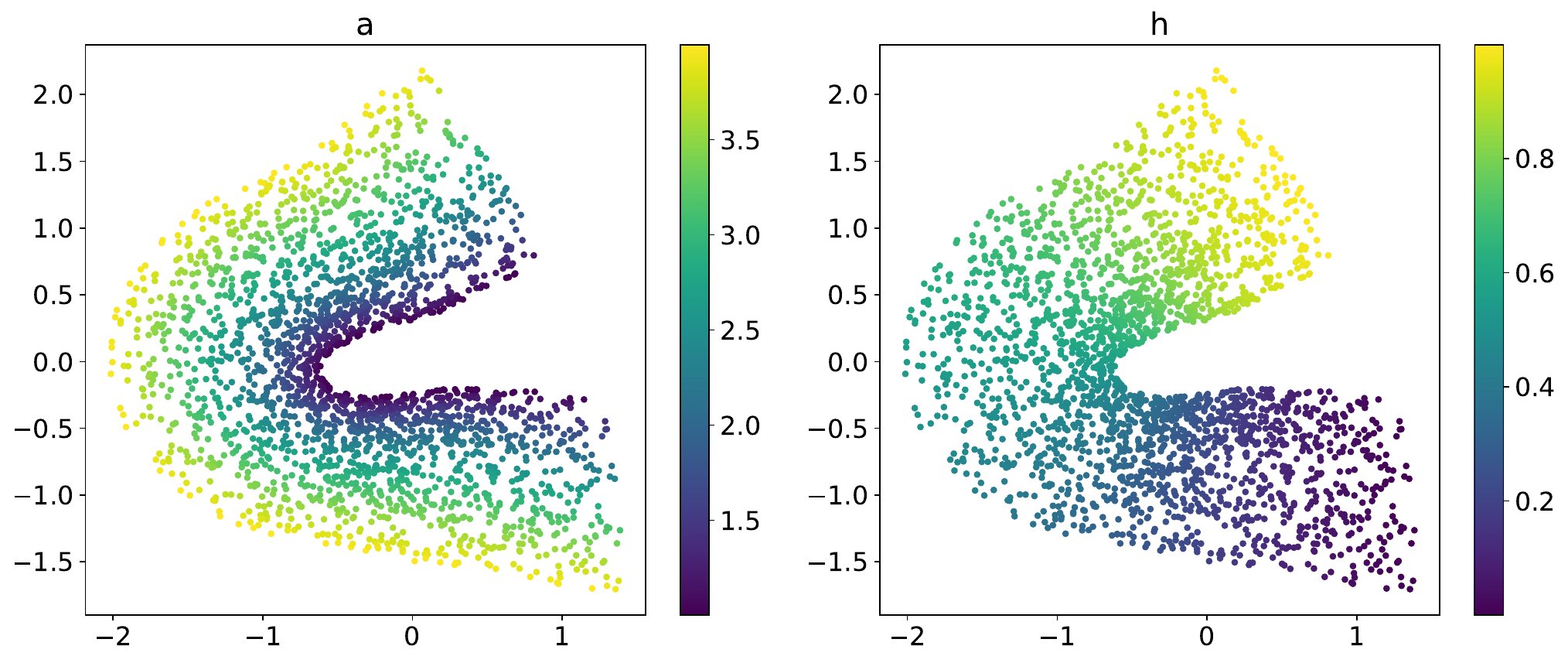}
    \caption{Latent features of the initial conditions $g_{\alpha,h}$ (Transport) in \eqref{eq:ftransport} given by the Auto-Encoder. The left plot is colored according to $a$ and the right plot is colored according to $h$.
    }
    \label{fig:transport:project}
\end{figure}

\begin{figure}[h]
\centering
\subfigure[One draw from $\gamma_{transport}$]{
    \includegraphics[width=5cm]{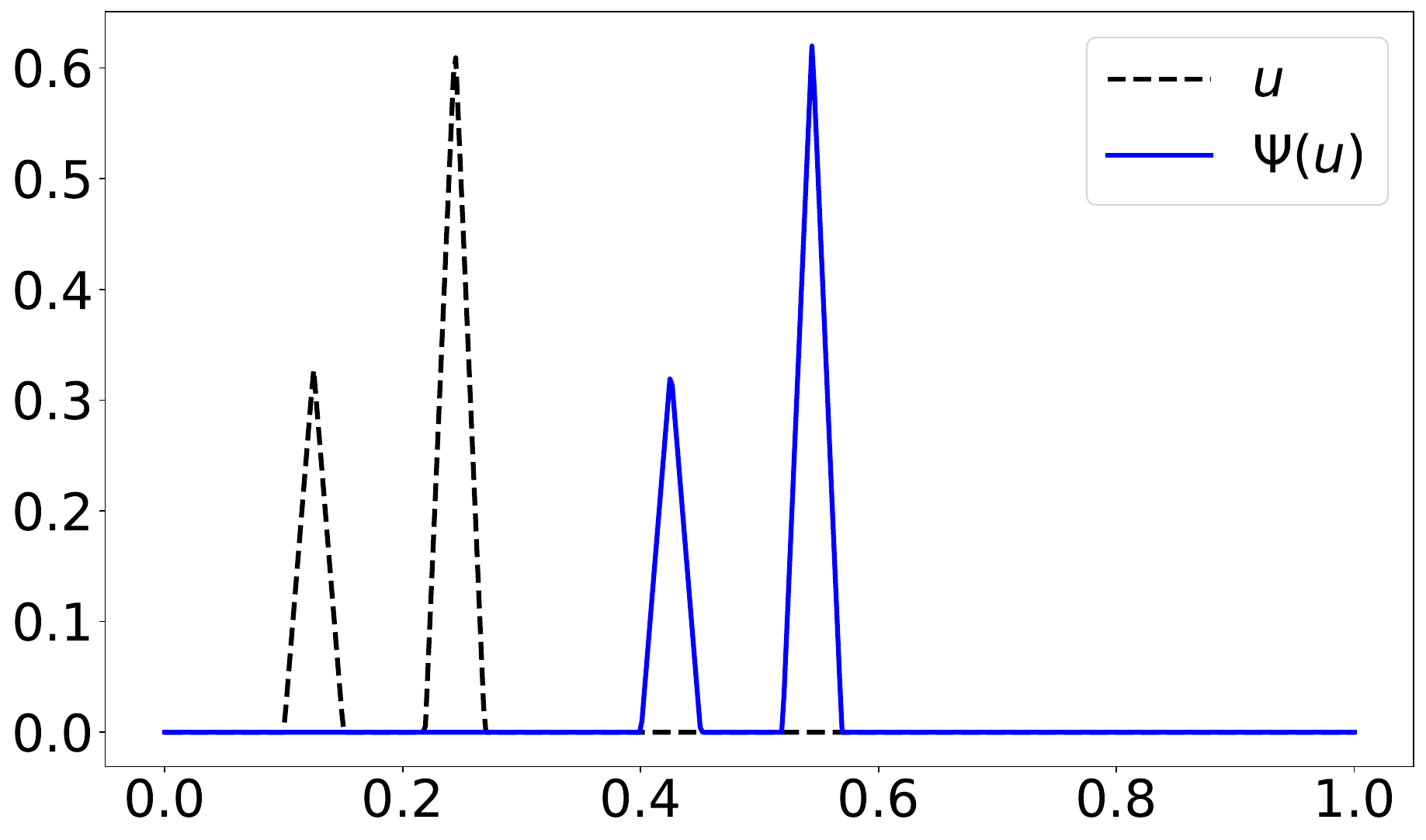}
    \label{fig:transport:samples}
}
\subfigure[Relative projection error]{
    \includegraphics[width=5cm]{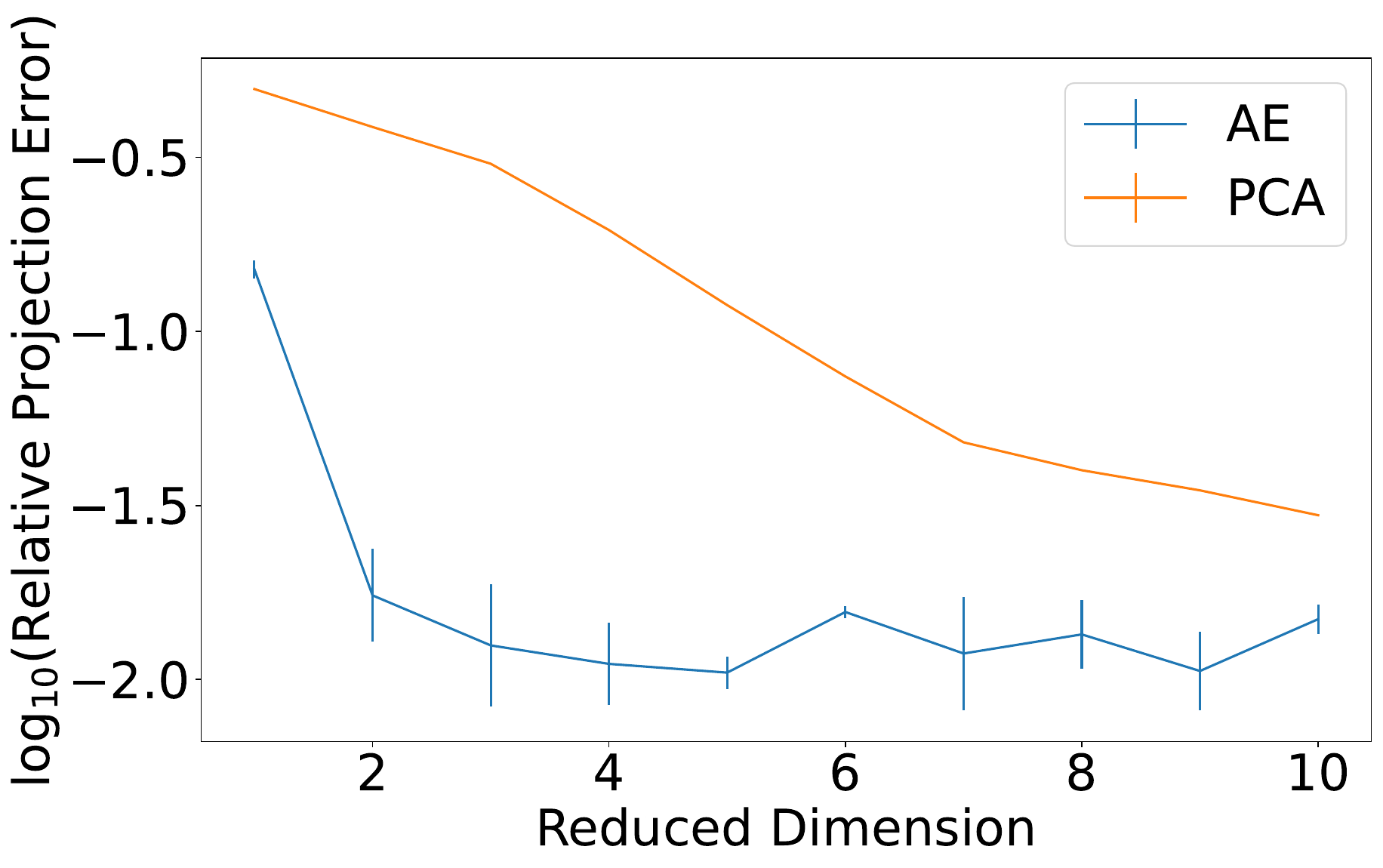}
    \label{fig:transport:projerr}
}
\subfigure[Relative test error]{
    \includegraphics[width=5cm]{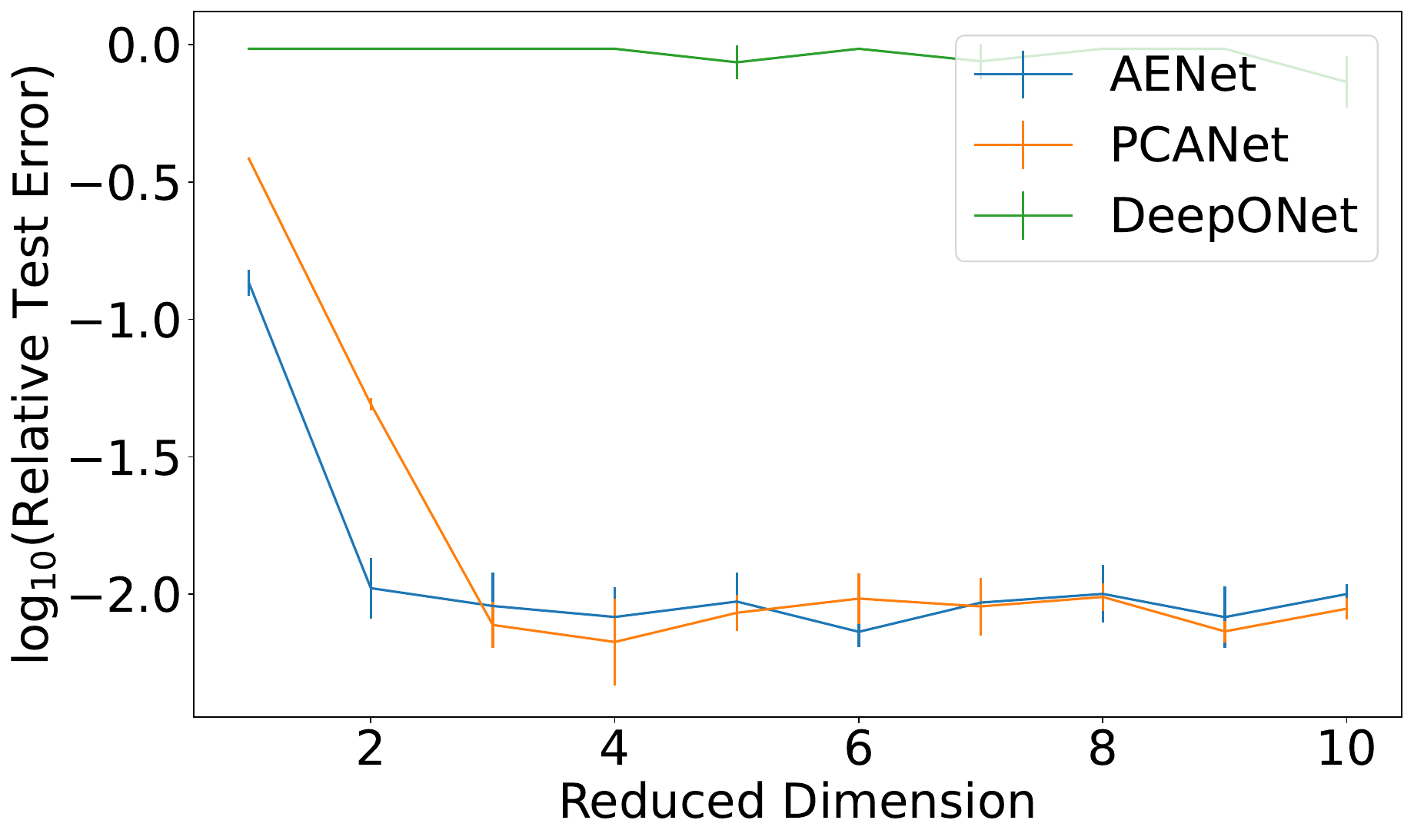}
    \label{fig:transport:testerr}
}
\subfigure[Predicted solutions]{
    \includegraphics[width=5cm]{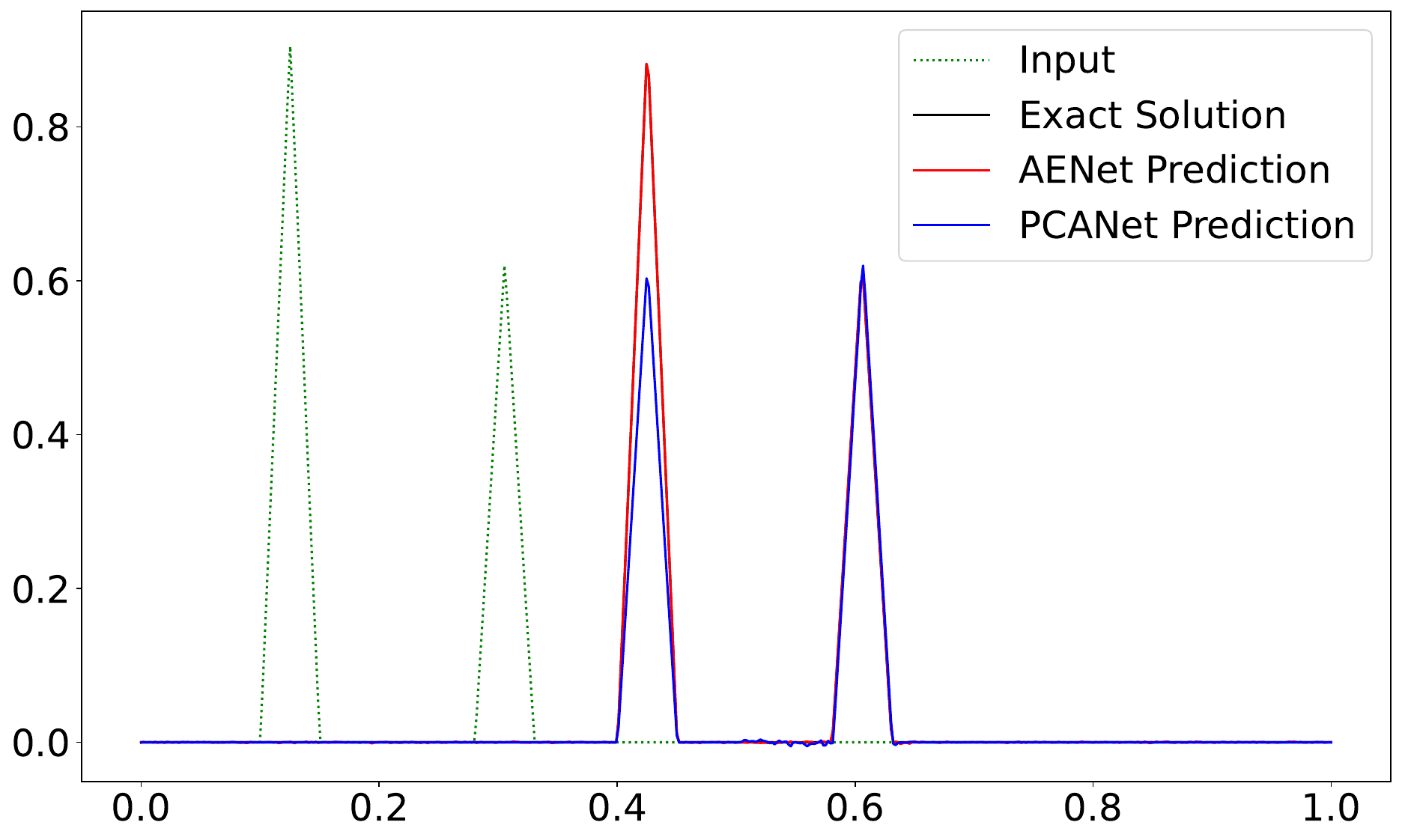}
    \label{fig:transport:examples}
}
\subfigure[Squared test error versus $n$]{
    \includegraphics[width=5cm]{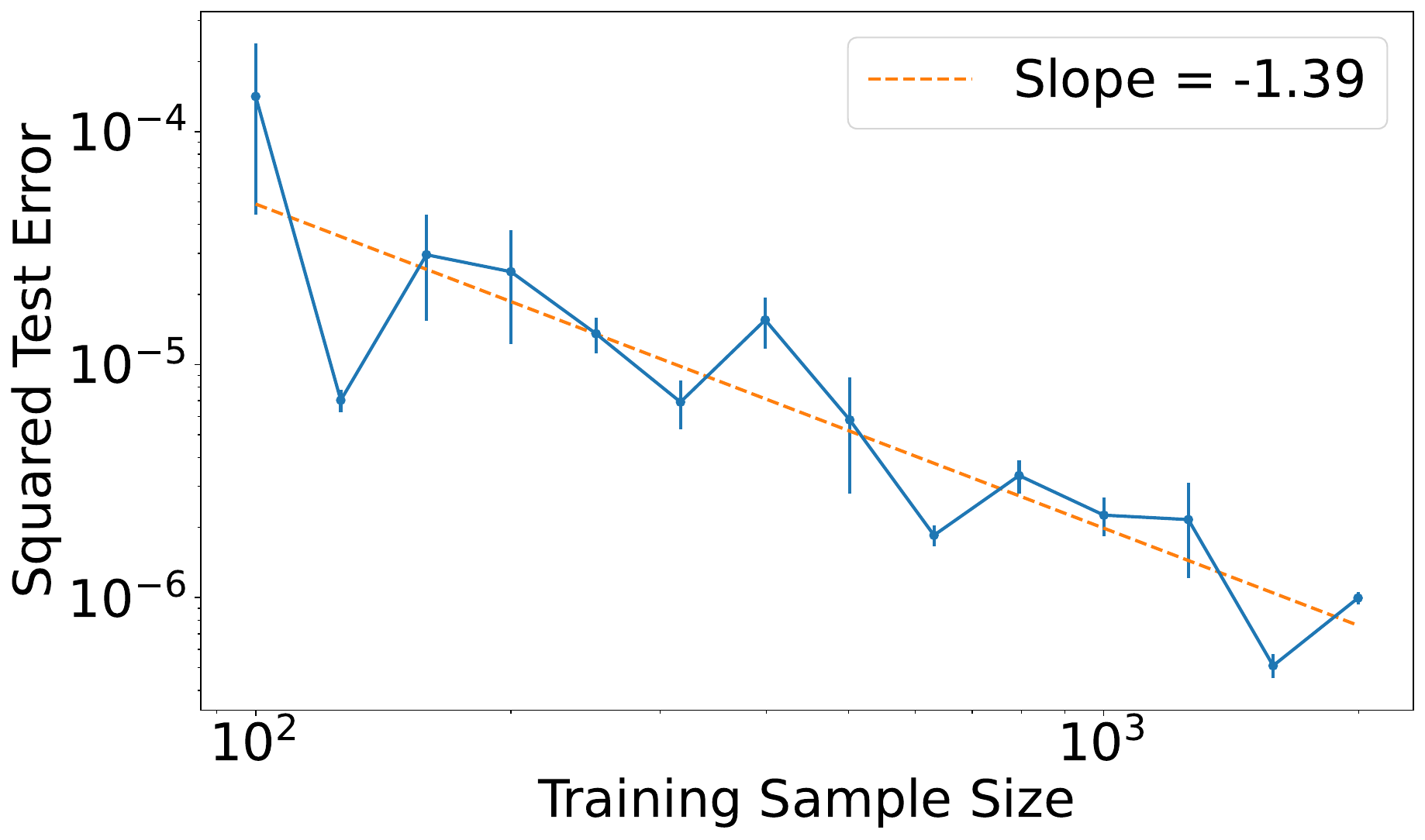}
    \label{fig:transport:generr}
}
\subfigure[Robustness to noise]{
    \includegraphics[width=5cm]{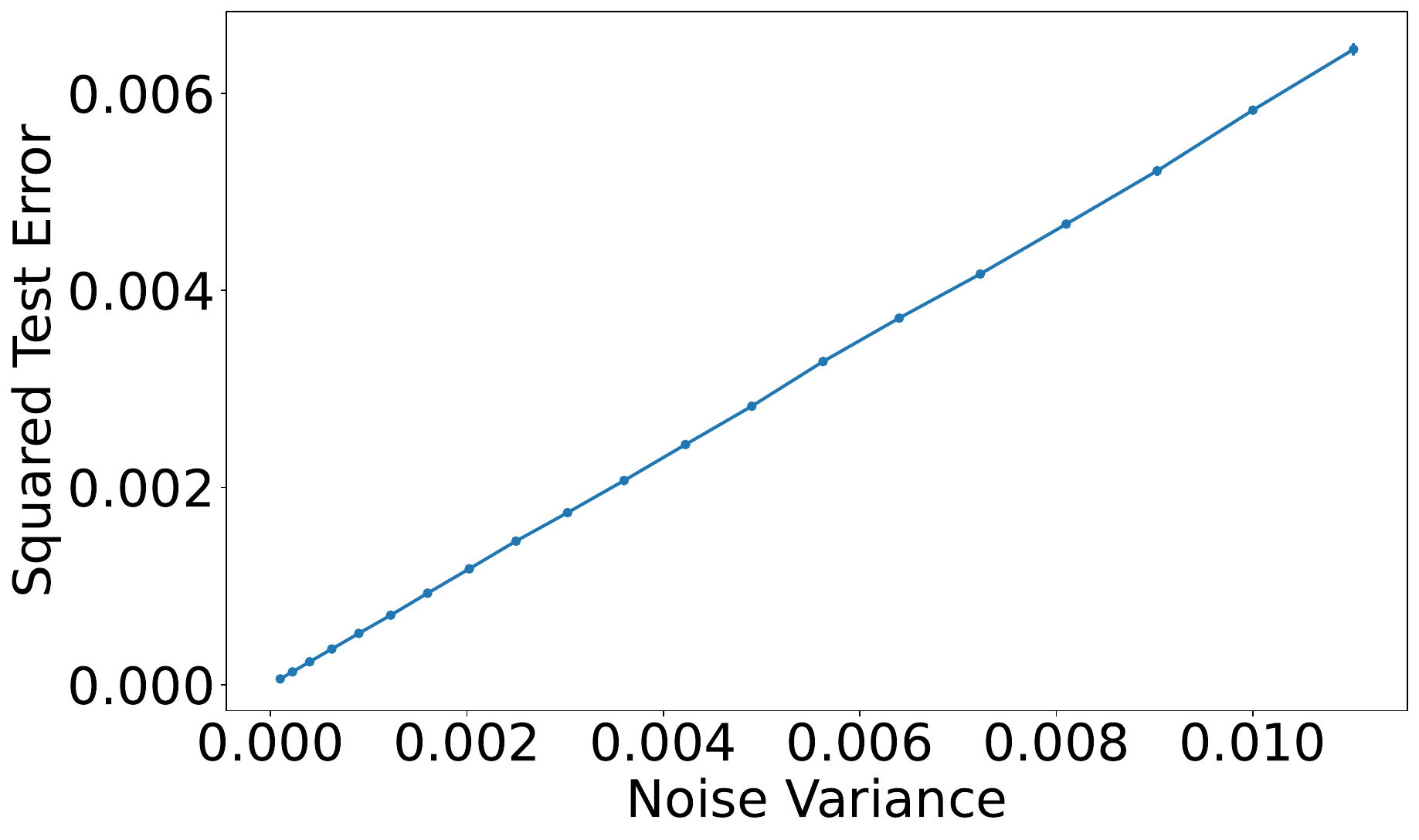}
    \label{fig:transport:noise}
}
\caption{Results of the transport equation.}
\label{fig:transport}
\end{figure}

Figure \ref{fig:transport:samples} shows a sample $u \sim \gamma_{transport}$, as well as $\Psi(u)$. Before learning the operator $\Psi$, we compare the projection error of Auto-Encoder and PCA on a test sample from $\gamma$ in Figure \ref{fig:transport:projerr}.  In Figure \ref{fig:transport:projerr}, Auto-Encoder is trained three times with different initilizations, and the average squared test error is shown with standard deviation error bar. 
Auto-Encoder yields a significantly smaller projection error than PCA for the same reduced dimension.
Figure \ref{fig:transport:testerr} shows the relative test error of AENet, PCANet, and DeepONet (after learning $\Psi$) as functions of the reduced dimension.
We further show the comparison of relative test error (as a percent) in Table \ref{tab:table}(a). AENet outperforms PCANet when the reduced dimension is the intrinsic dimension $2$, and they are comparable when the reduced dimension is bigger than $2$. 
Finally, Figure \ref{fig:transport:examples} shows an example of the predicted solution at $t=1$ for AENet and PCANet with input reduced dimension 2.

To validate our theory in Theorem \ref{thm:generalization}, we show a log-log plot of the absolute squared test error versus training sample size in Figure \ref{fig:transport:generr} for AENet. The curve is almost linear, depicting the theorized exponential relationship. 
To show robustness to noise, we plot the squared test error versus the variance of Gaussian noise added to the output data in Figure \ref{fig:transport:noise}, depicting the theorized relationship. In Figure \ref{fig:transport:noise}, the latent dimension of AENet is taken as $2$. In Figure \ref{fig:transport:generr} and \ref{fig:transport:noise}, we perform three experimental runs, and show the mean with standard deviation error bar.

\subsection{Burgers' equation}
\label{sec:burgers}

We consider the viscous Burgers' equation with periodic boundary conditions given by (for fixed viscosity $\nu = 10^{-3}$)
\begin{align}
\label{eq:burgers}
\begin{aligned}
u_t = \nu u_{xx} - u u_x,\quad &x \in [0, 1), t \in (0, 1] 
\end{aligned}
\end{align}
with a periodic boundary condition and the initial condition $u(x, 0) = g(x), x \in [0, 1). $ 
We seek to approximate the operator $\Psi$ which takes $g(x) = u(x, 0)$ as input and outputs $u(x, 1)$ from the solution of \eqref{eq:burgers} at $t=1$. 

\begin{figure}[t]
    \centering
    \includegraphics[height=3.8cm]{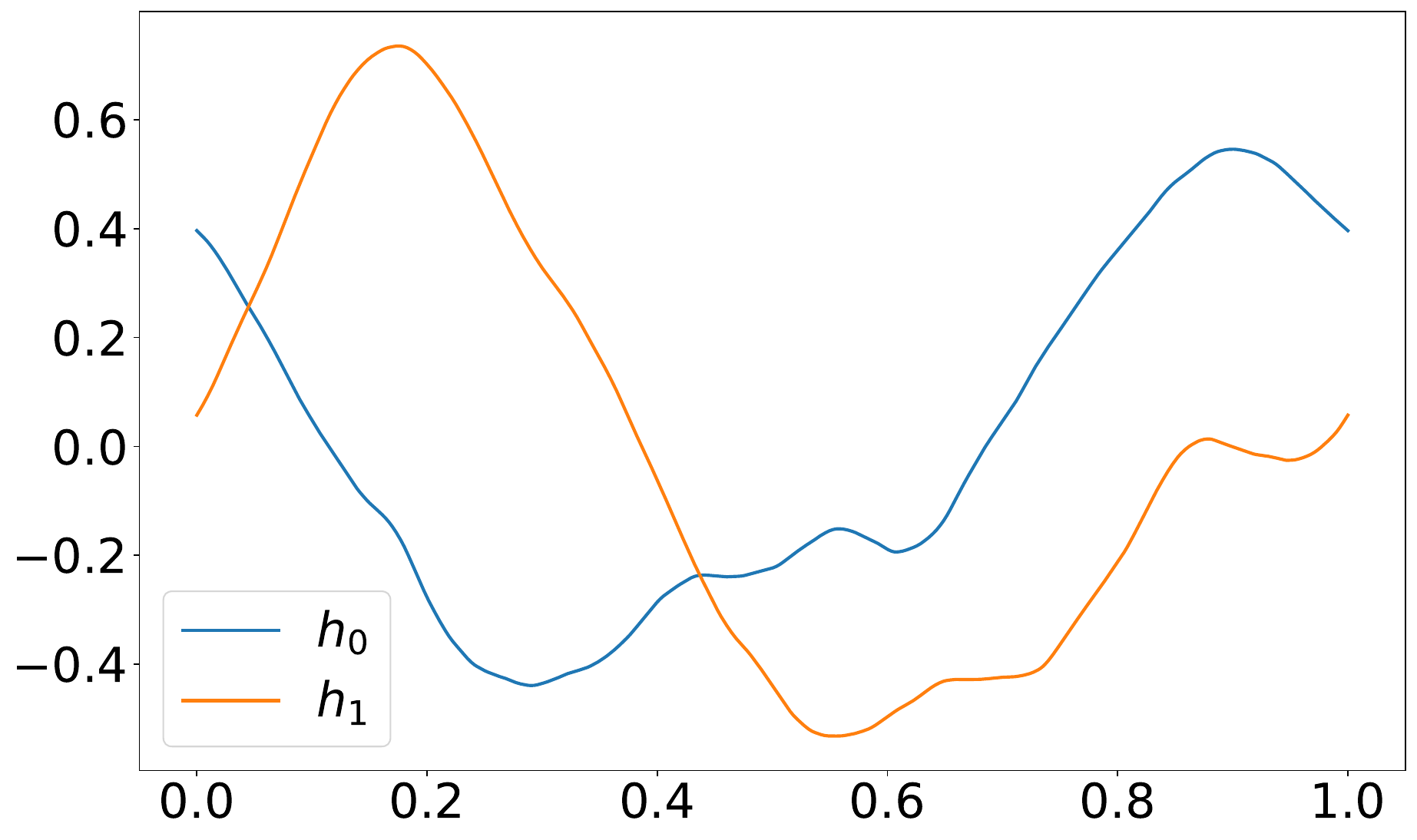}
    \caption{$w_0$ and $w_1$ used for the Burgers' example}
    \label{fig:grfsamples}
\end{figure}

 Let $w_0$ and $w_1$ be two functions sampled from the probability measure $N\left(0, 7^4(-\frac{d^2}{dx^2}+7^2I)^{-2.5}\right)$ on $[0, 1)$, which is considered in  \cite{bhattacharya2020model}. Figure \ref{fig:grfsamples} shows a plot of the $b_0$ and $b_1$ used for the results in this section. For any $a \in [-0.9, 0.9]$ and $h \in [0, 1]$, define 
 \begin{equation} g_{a, h}(x) = a w_0(x - h) + \sqrt{1 - a^2} w_1(x - h). 
 \label{eq:fburgers}
 \end{equation}
Our sampling measure $\gamma_{burg}$ is defined on $\cM := \{g_{a, h} : a \in [-0.9, 0.9], h \in [0, 1]\}$ by sampling $a$ and $h$ uniformly and then constructing $g_{a, h}$ restricted to $x \in [0, 1)$. 

Figure \ref{fig:burgerssvd} plots the singular values in descending order of the input data sampled from $\gamma_{burg}$. The slow decay of the singular values indicates that a non-linear encoder would be a better choice than a linear encoder for this problem. Figure \ref{fig:burgerscomponents} shows the non-linearity of the data, when we project the data into the top $2$ principal components. Figure \ref{fig:burgerscomponents3} and \ref{fig:burgerspca6} further show the  projections of this data set to the 1st-6th principal components. This data set is nonlinearly parametrized by 2 intrinsic parameters, but the top 2 principal components are not sufficient to represent the data, as shown in Figure \ref{fig:burgerscomponents}. 

\begin{figure}[h]
\centering
\subfigure[Singular values]{
    \includegraphics[height=3.5cm]{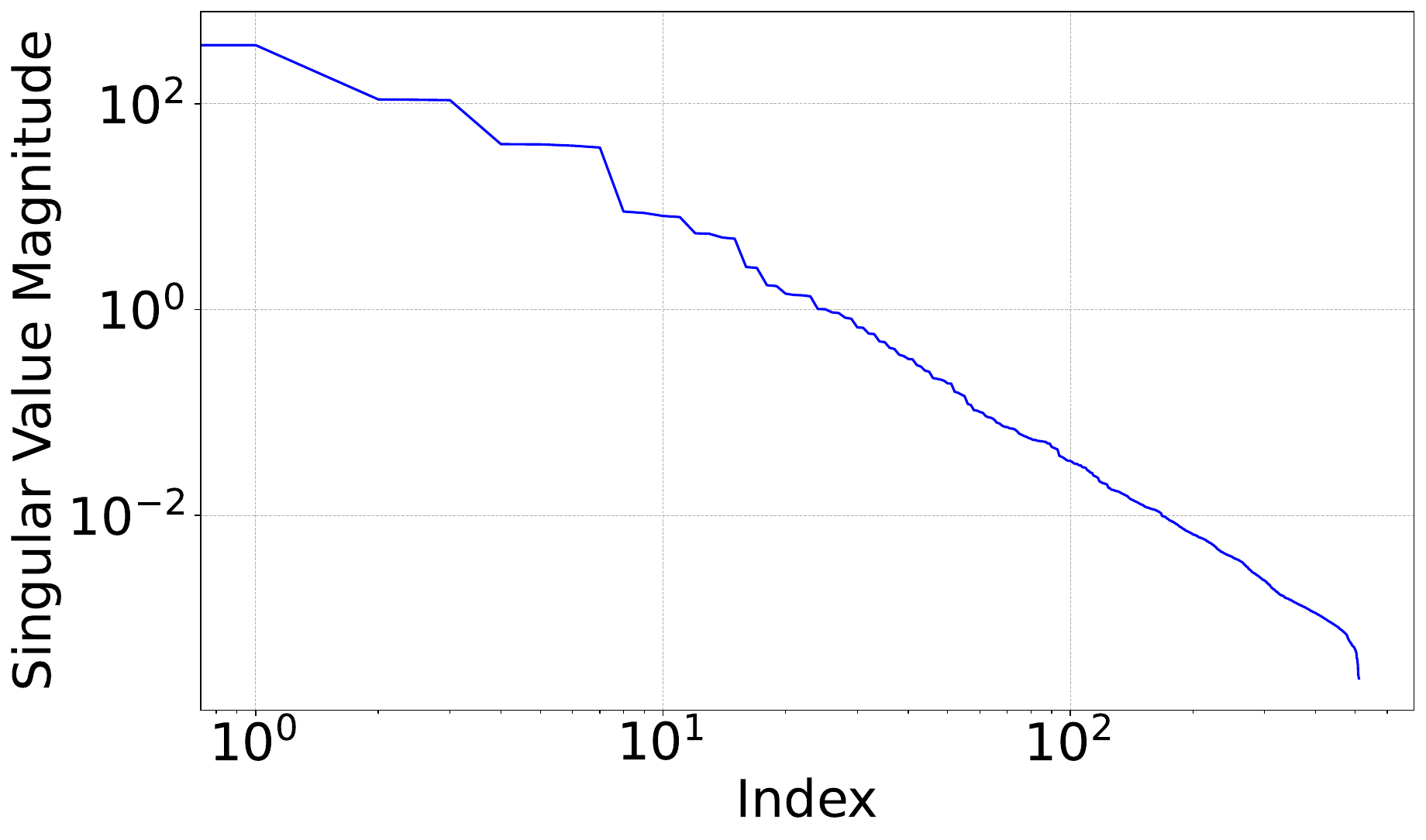}
    \label{fig:burgerssvd}
}
\subfigure[Projection to top 2 principal components]{
    \includegraphics[height=4cm]{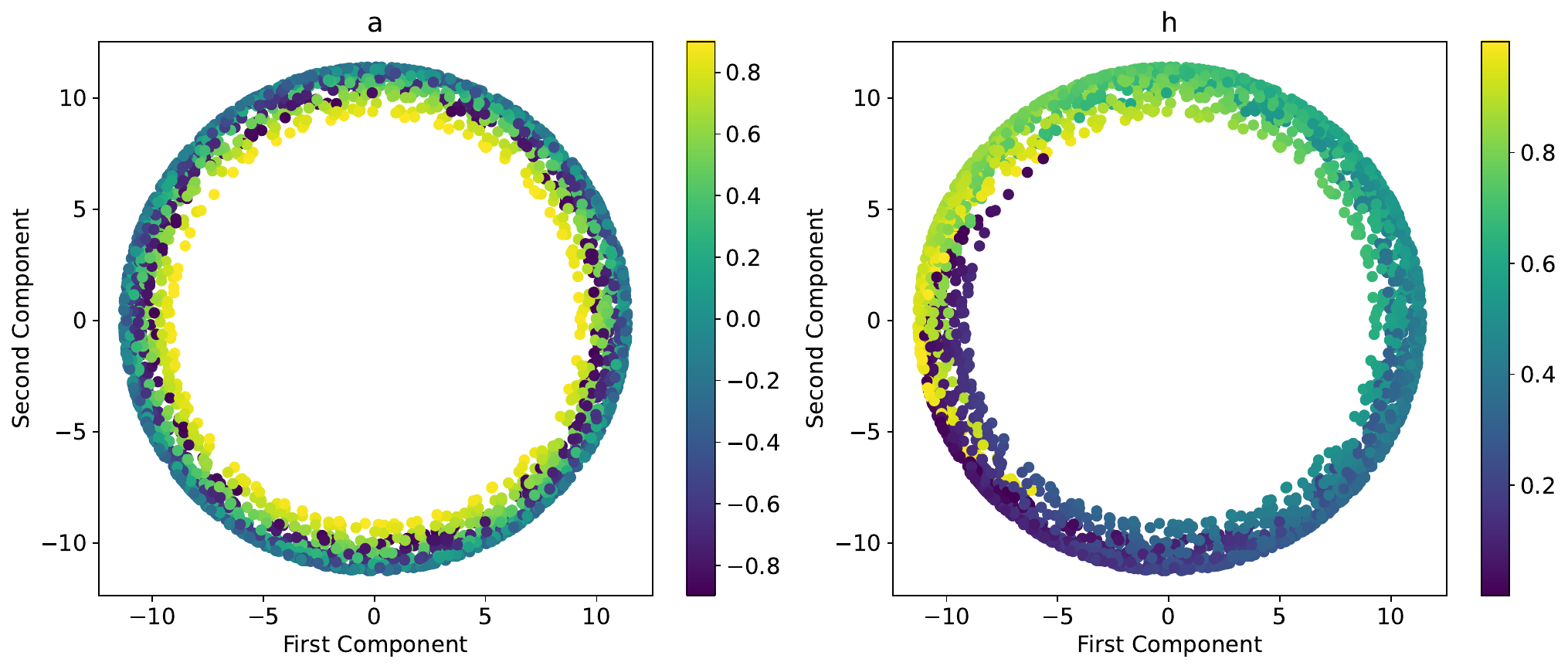}
    \label{fig:burgerscomponents}
}
\subfigure[Projection to top 3 principal components]{
    \includegraphics[height=3.5cm]{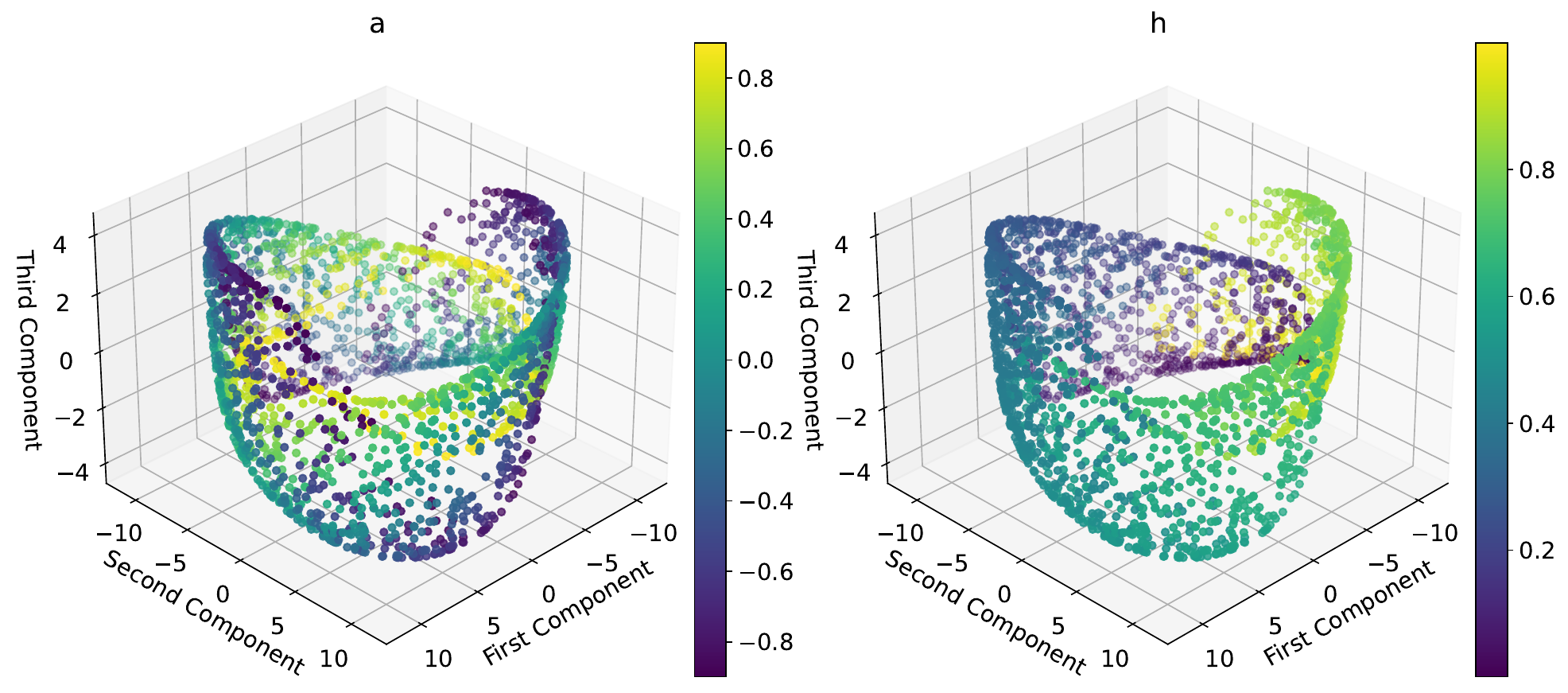}
    \label{fig:burgerscomponents3}
}
\hspace{-0.4cm}
\subfigure[Projection to 4th-6th principal components]{
    \includegraphics[height=3.5cm]{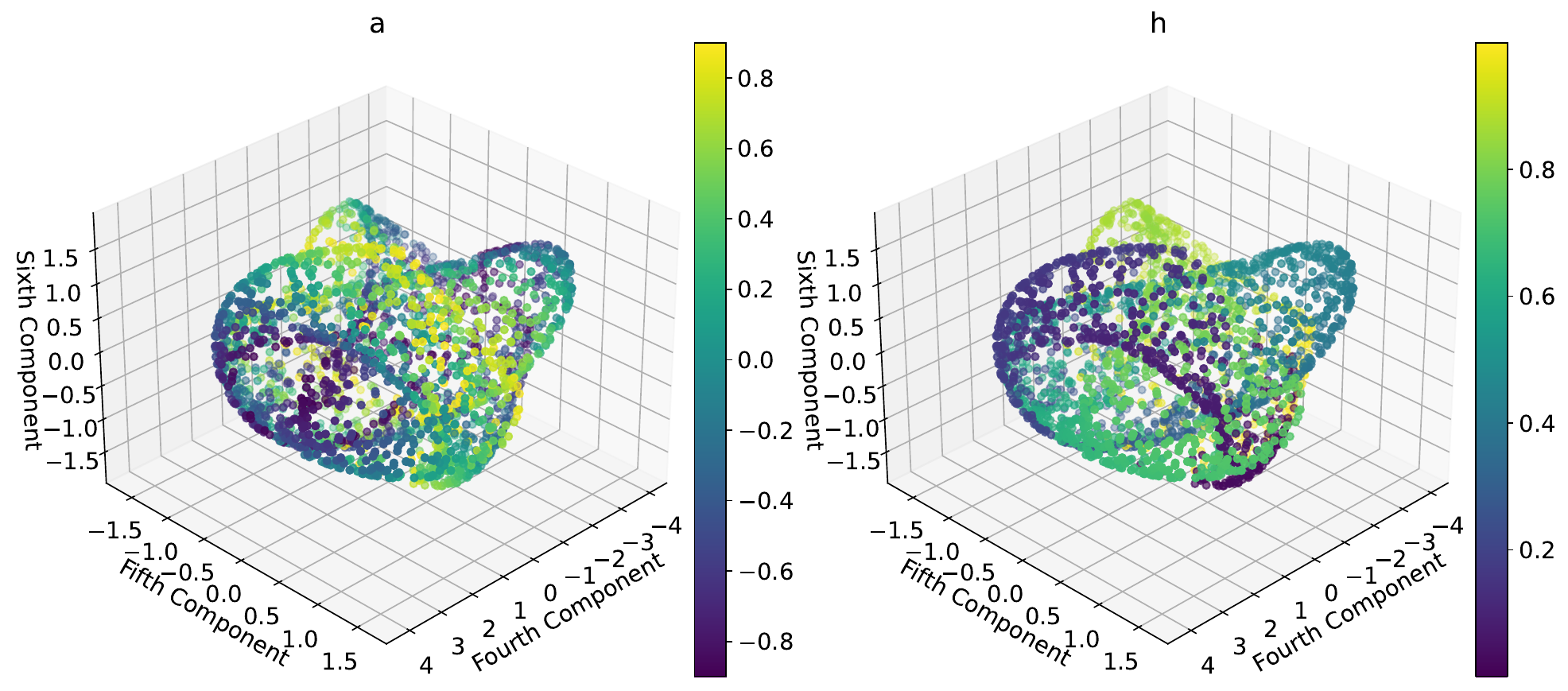}
    \label{fig:burgerspca6}
}
\caption{Nonlinearity of the initial conditions for the Burgers' equation. (a) shows the singular values of the data matrix. (b) shows the projection of data to the top 2 principal components and (c) shows the projection to top 3 principal components. (d) shows the projection to the 4th-6th principal components. In (b), (c) and (d), the projections are colored according to the $a$ parameter in the left subplots and according to the $h$ parameter in the right subplots.}
\label{fig:burgersnonlinear}
\end{figure}

On the other hand, we can use  Auto-Encoder for nonlinear dimension reduction. Figure \ref{fig:burgersproject} shows the projection of the training data by the encoder of a trained Auto-Encoder with reduced dimension $2$. The latent parameters reveals the geometry of an  annulus, i.e. the Cartesian product of an interval and a circle. This matches the distribution of parameters in $\cM$, because the first parameter $a$ varies on a closed interval, and the second parameter $b$ represents translation on the periodic domain which represents a circle.

\begin{figure}
    \centering
    \includegraphics[height=5cm]{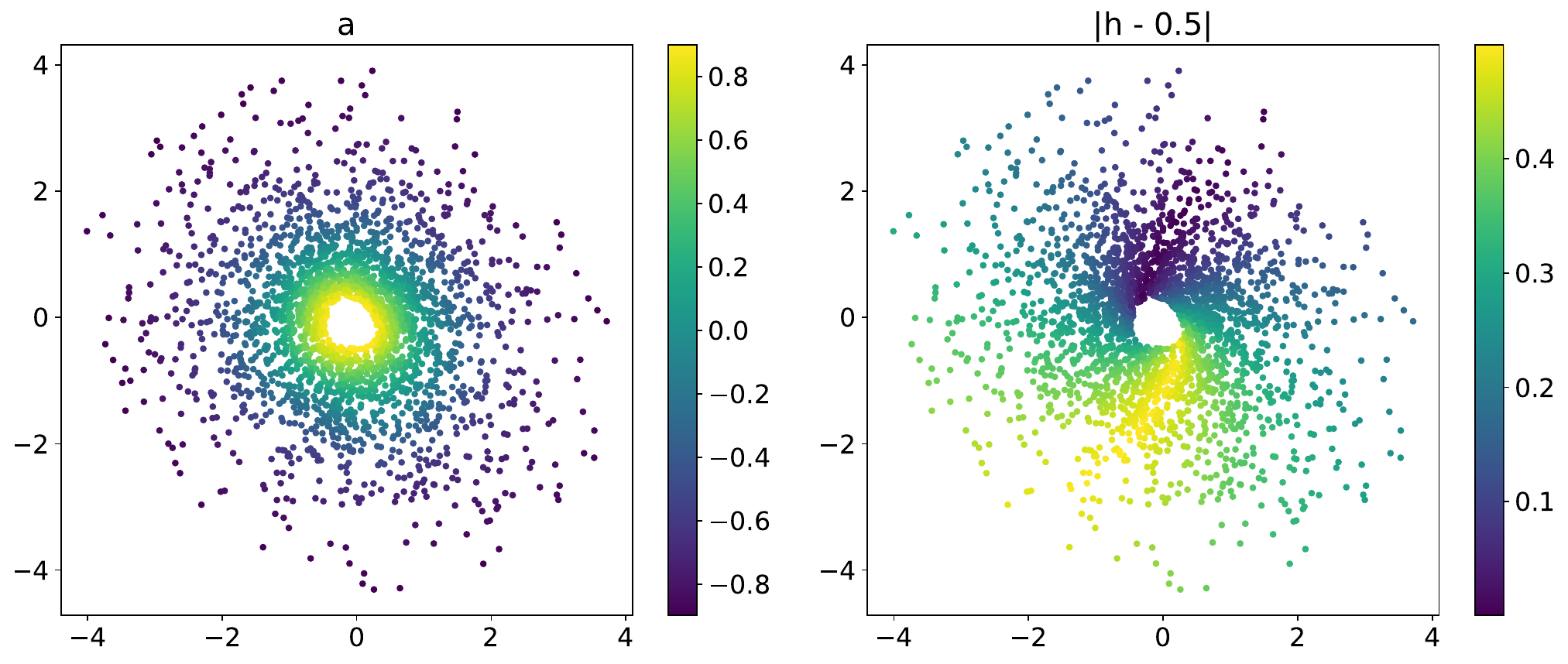}
    \caption{
    Latent features of the initial conditions $g_{a,h}$ (Burgers') in \eqref{eq:fburgers} given by the Auto-Encoder. The left plot is colored according to $a$ and the right plot is colored according to $|h-0.5|$. 
    }
    \label{fig:burgersproject}
\end{figure}

\begin{figure}[h]
\centering
\subfigure[One draw from $\gamma_{burg}$]{
    \includegraphics[width=5cm]{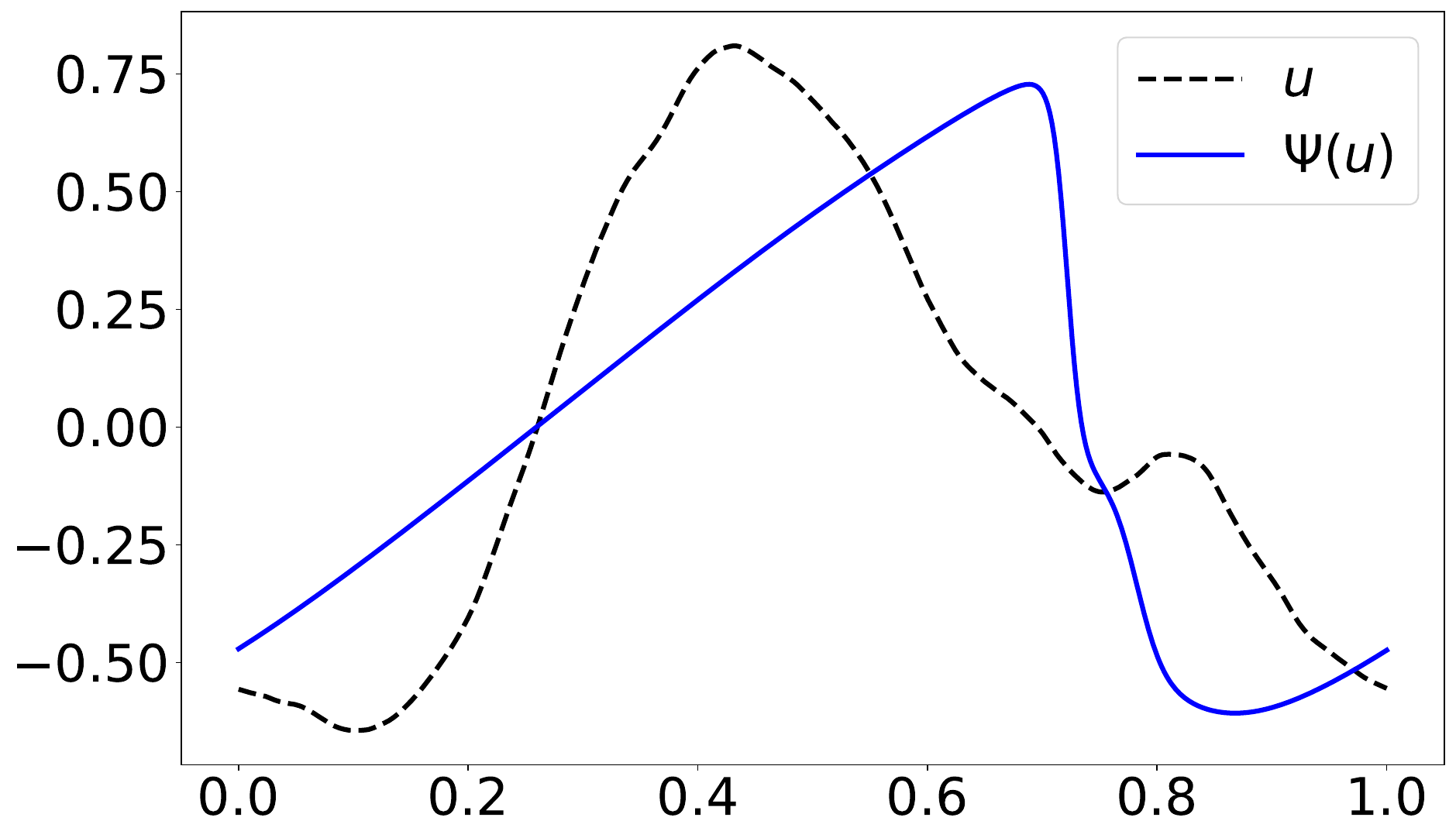}
    \label{fig:burgers:samples}
}
\subfigure[Relative projection error]{
    \includegraphics[width=5cm]{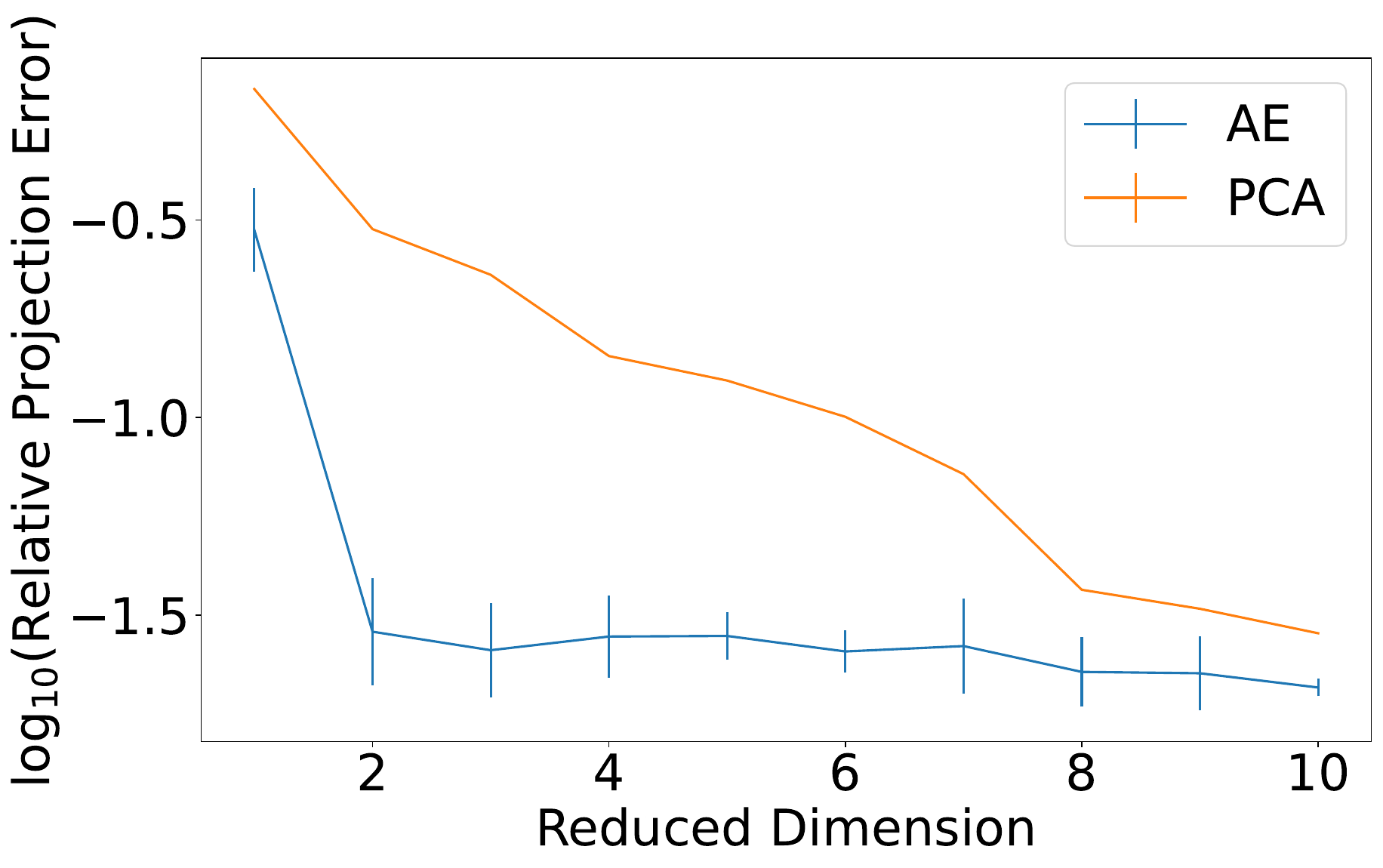}
    \label{fig:burgers:projerr}
}
\subfigure[Relative testing error]{
    \includegraphics[width=5cm]{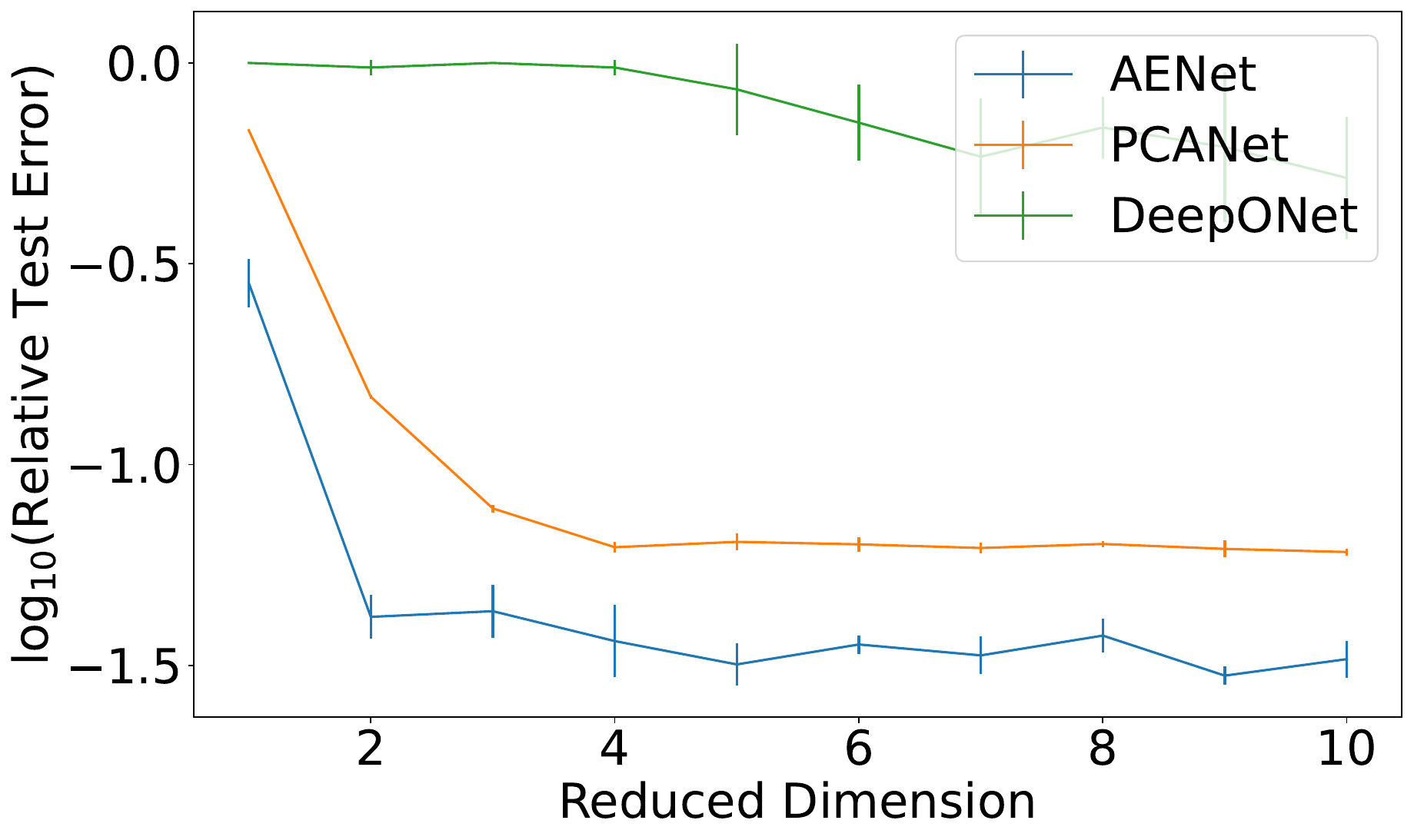}
    \label{fig:burgers:testerr}
} \vspace*{-0.3cm}\\

\subfigure[Predicted solutions]{
    \includegraphics[width=5cm]{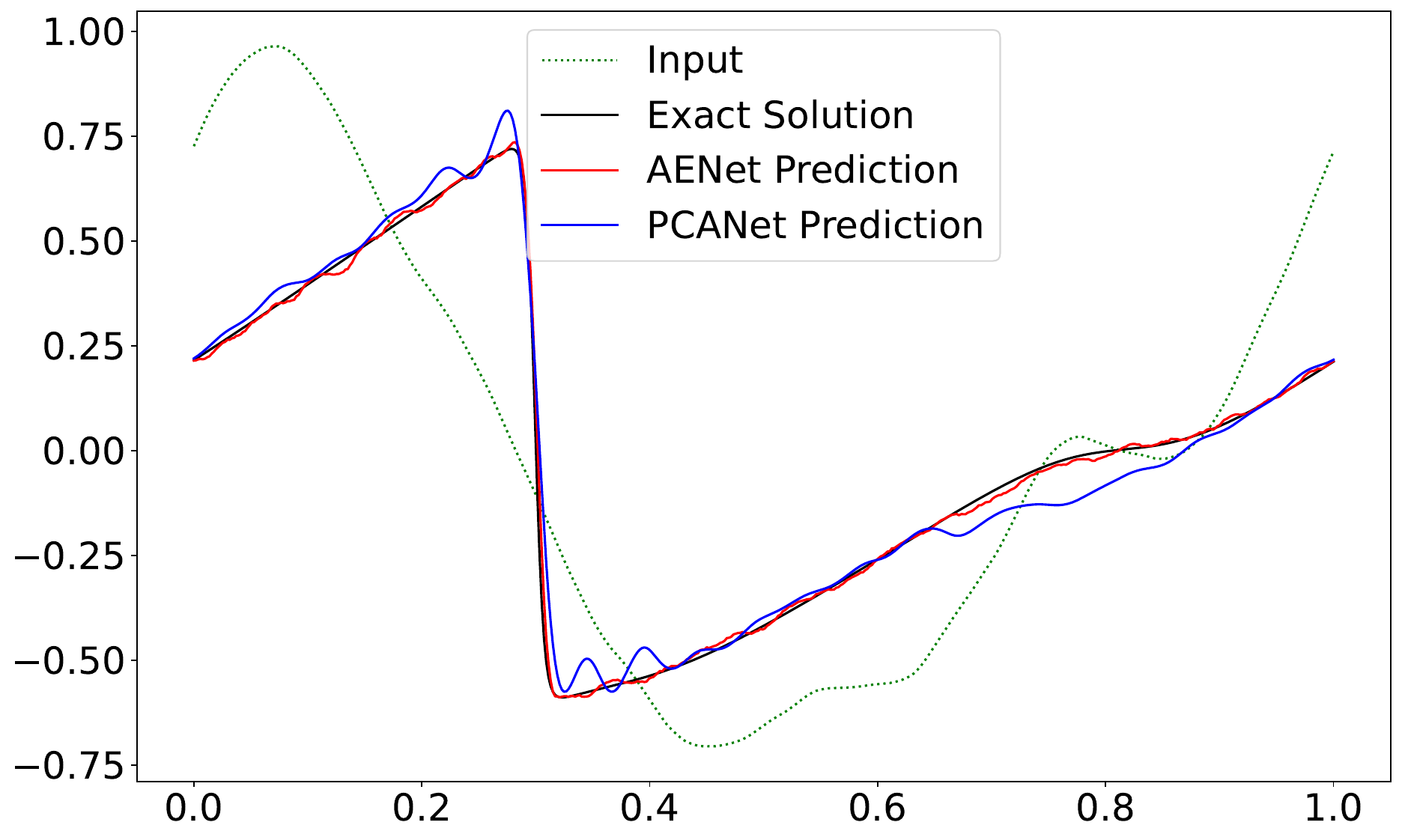}
    \label{fig:burgers:examples}
}
\subfigure[Squared test error versus $n$]{
    \includegraphics[width=5cm]{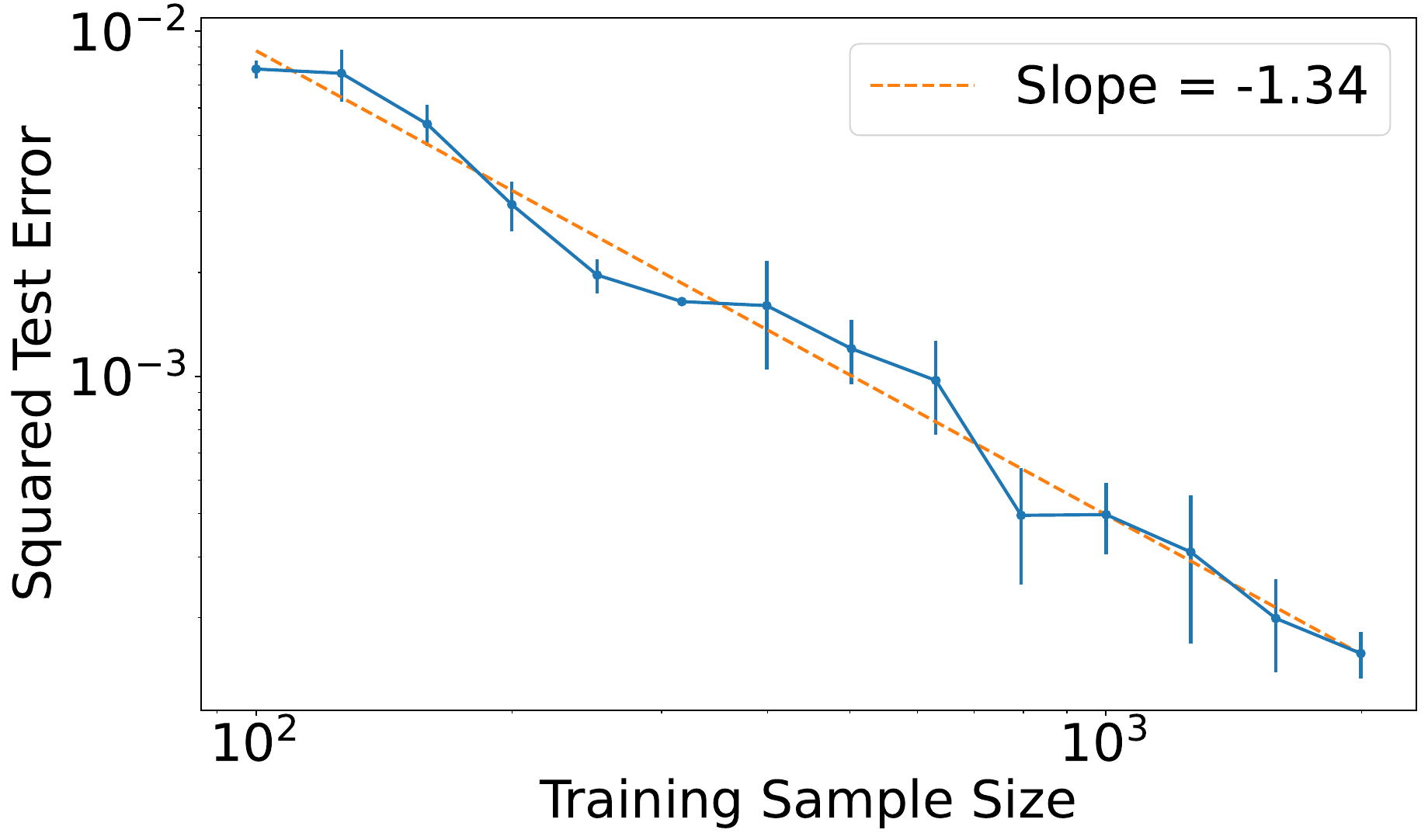}
    \label{fig:burgers:generr}
}
\subfigure[Robustness to noise]{
    \includegraphics[width=5cm]{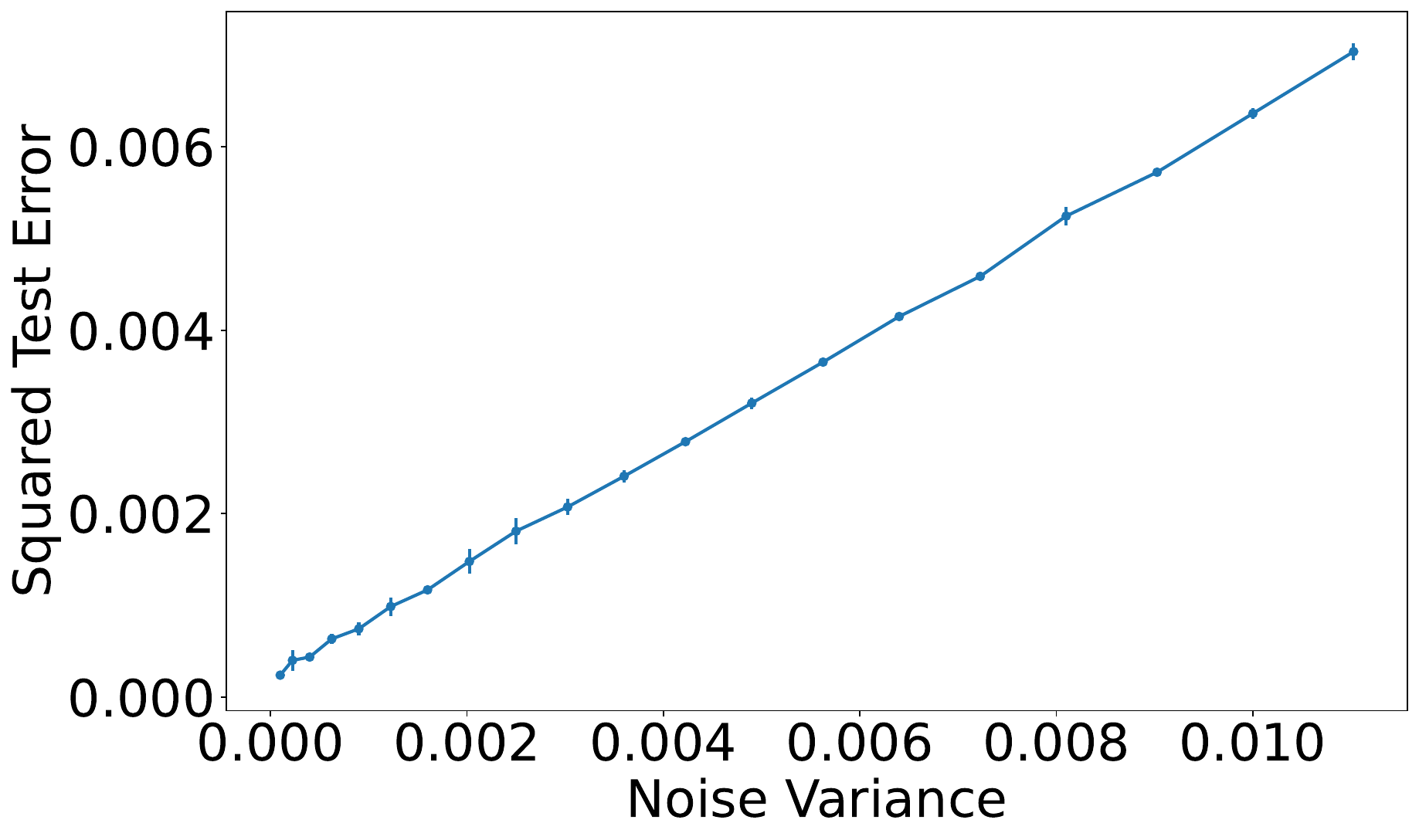}
    \label{fig:burgers:noise}
}

\caption{Results of the Burgers' equation. }
\label{fig:burgers}
\end{figure}

Figure \ref{fig:burgers} contains various plots comparing AENet, PCANet, and DeepONet for the Burgers' equation, analogous to the role Figure \ref{fig:transport} plays for the transport equation. In Figure \ref{fig:burgers:testerr} AENet outperforms PCANet for all reduced dimensions. 
Figure \ref{fig:burgers:examples} compares AENet with reduced dimension 2 to PCANet with reduced dimension 2 for domain and 40 for range. 
Figure \ref{fig:burgers:generr} and \ref{fig:burgers:noise} display the absolute squared test error. Figures \ref{fig:burgers:generr} and \ref{fig:burgers:noise} are also generated using AENet with reduced dimension 2. The comparison of relative test error (as a percent) is further shown in Table \ref{tab:table}(b).

\subsection{Korteweg–De Vries (KdV) equation}
\label{sec:kdv}

We consider one dimensional KdV equation given by 
\begin{equation}
\label{eq:kdv}
\begin{aligned}
u_t = - u_{xxx} - u u_x,\quad &x \in [0, 6], t \in (0, 0.01] 
\end{aligned}
\end{equation}
with initial condition $u(x, 0) = f(x), x \in [0, 6].$ We seek to approximate the operator which takes $u(x, 0)$ as input and outputs $u(x, 0.01)$ from the solution of \eqref{eq:kdv}.

For any $a \in [6, 18]$ and $h \in [0, 3]$, consider the function
\begin{align}
g_{a, h}(x) = \frac{a^2}{2}\text{sech}\left(\frac{a}{2}(x - 1))\right)^2 + \frac{6^2}{2} \text{ sech} \left(\frac{6^2}{2}(x - 2 - h))\right)^2.
\label{eq:fkdv}
\end{align}
Our sampling measure $\gamma_{kdv}$ is defined on \begin{equation}
\cM := \{g_{a, h} : a \in [6, 18], h \in [0, 3]\}
\label{kdvinitialnarrow}
\end{equation}
by sampling $a$ and $h$ uniformly and then constructing $g_{a, h}$ restricted to $x \in [0, 6]$. 

Figure \ref{fig:kdvsvd} plots the singular values in descending order of the input data sampled from $\gamma_{kdv}$. The slow decay of the singular values indicates that a non-linear encoder would be a better choice than a linear encoder for this problem. Figure \ref{fig:kdvcomponents} further shows the non-linearity of the data, when we project the data into the top $2$ principal components. Figure \ref{fig:kdvcomponents3} and \ref{fig:kdvpca6} show the  projections of this data set to the 1st-6th principal components. This data set is nonlinearly parametrized by 2 intrinsic parameters, but the top 2 principal components are not sufficient to represent the data, as shown in Figure \ref{fig:kdvcomponents}. Figure \ref{fig:kdvcomponents3} shows that the top 3 linear principal components yield a better representation of the data, since the coloring by $a$ and $h$ is well recognized.

Figure \ref{fig:kdv:project} shows latent parameters of the training data given by the Auto-Encoder with reduced dimension $2$. The intrinsic parameters $a$ and $h$ are well represented in the latent space. 

\begin{figure}[h]
\centering
\subfigure[Singular values]{
    \includegraphics[height=3.5cm]{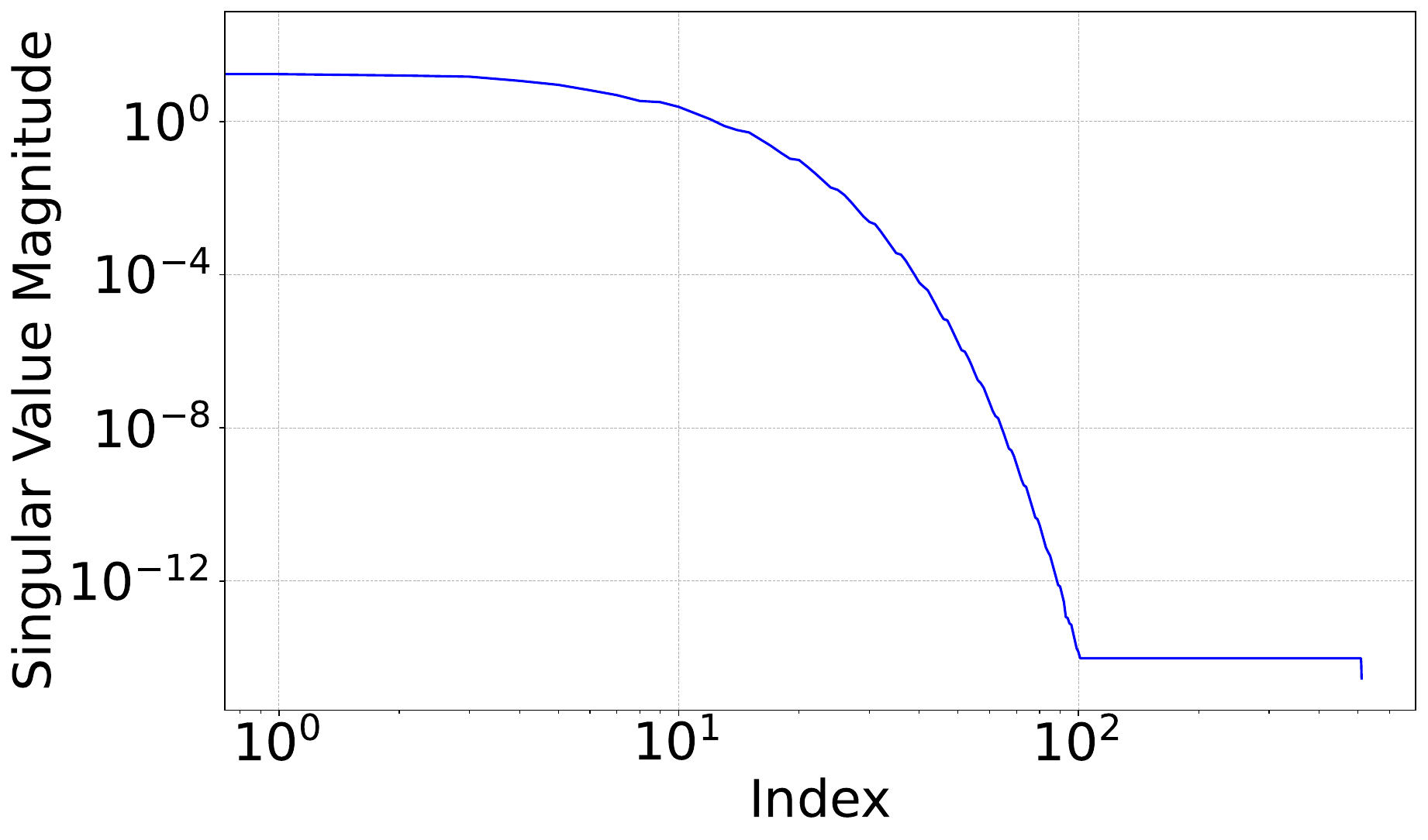}
    \label{fig:kdvsvd}
}
\subfigure[Projection to 2 principal components]{
    \includegraphics[height=4cm]{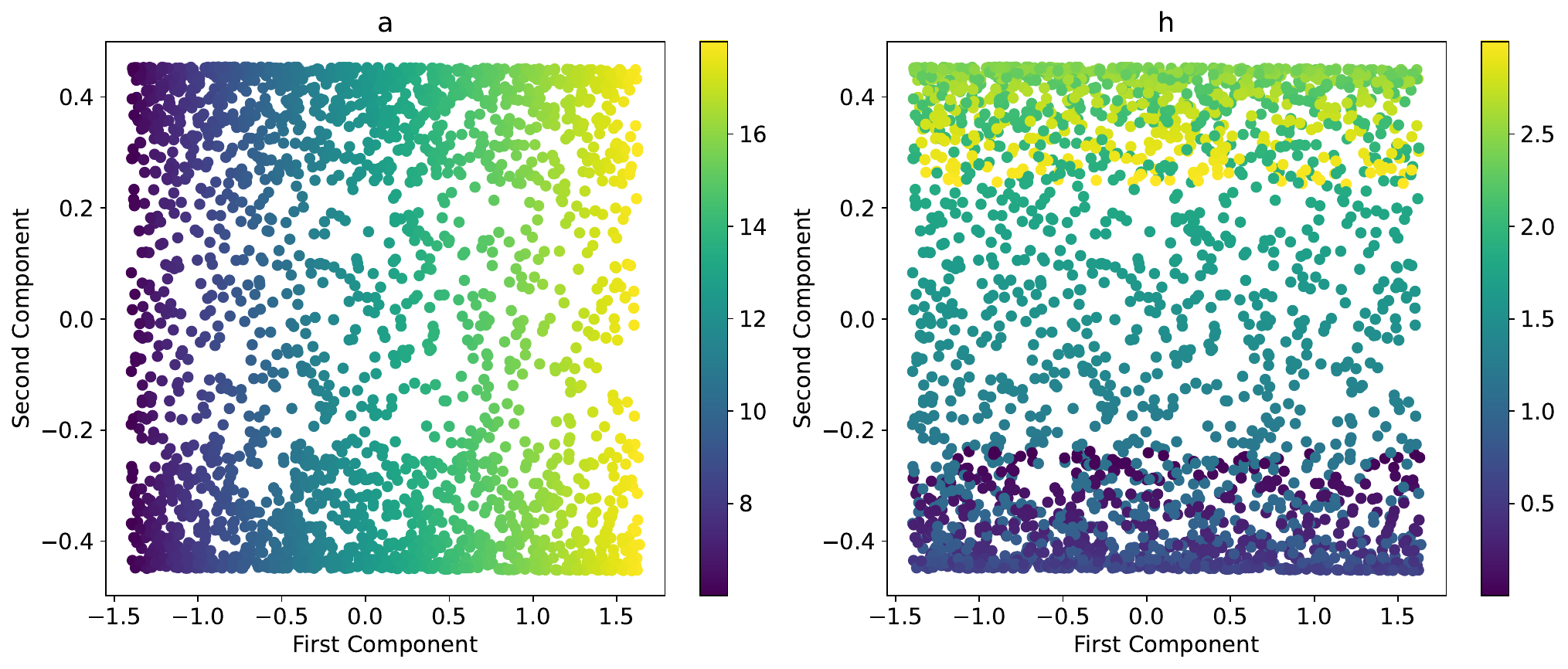}
    \label{fig:kdvcomponents}
}
\subfigure[Projection to 3 principal components]{
    \includegraphics[height=3.5cm]{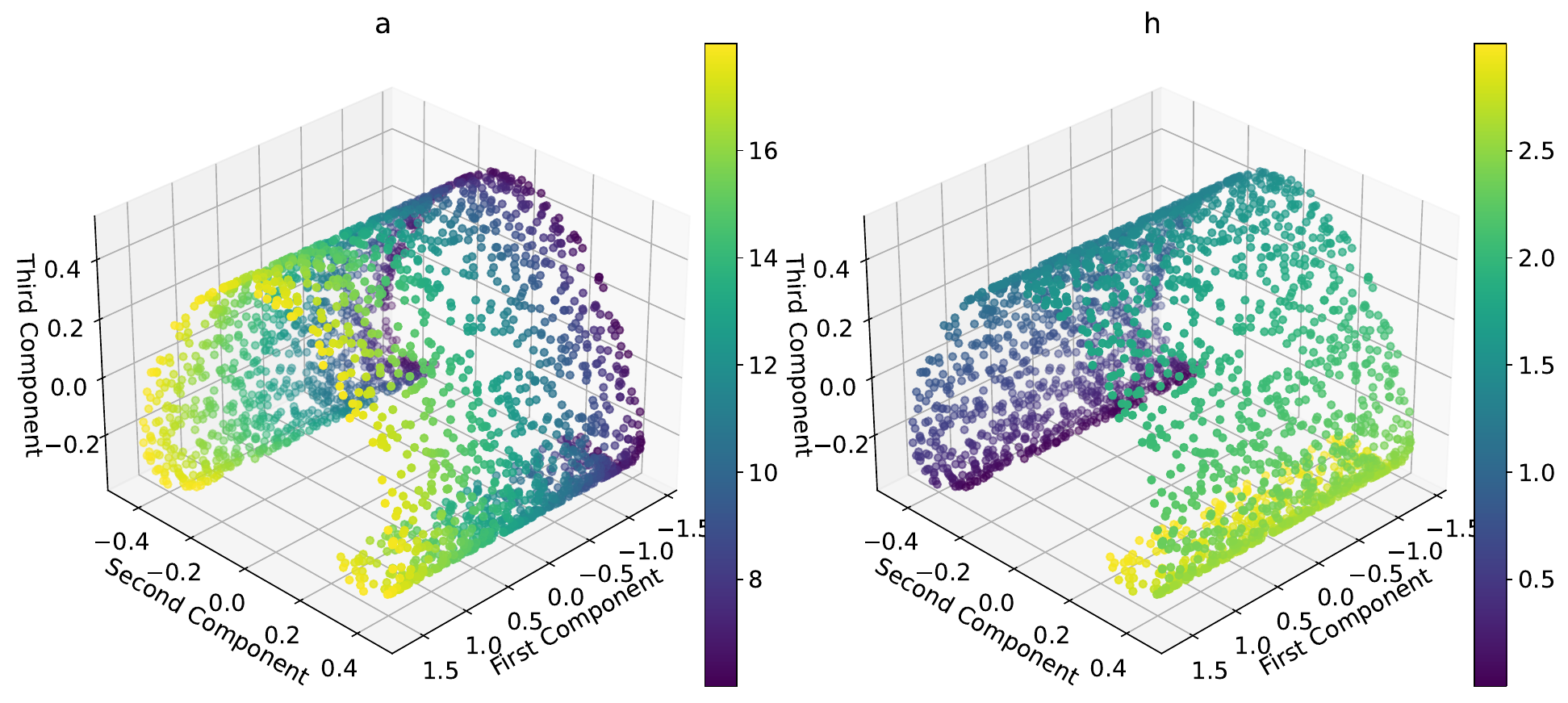}
    \label{fig:kdvcomponents3}
}
\subfigure[Projection to 4th-6th principal components]{
    \includegraphics[height=3.5cm]{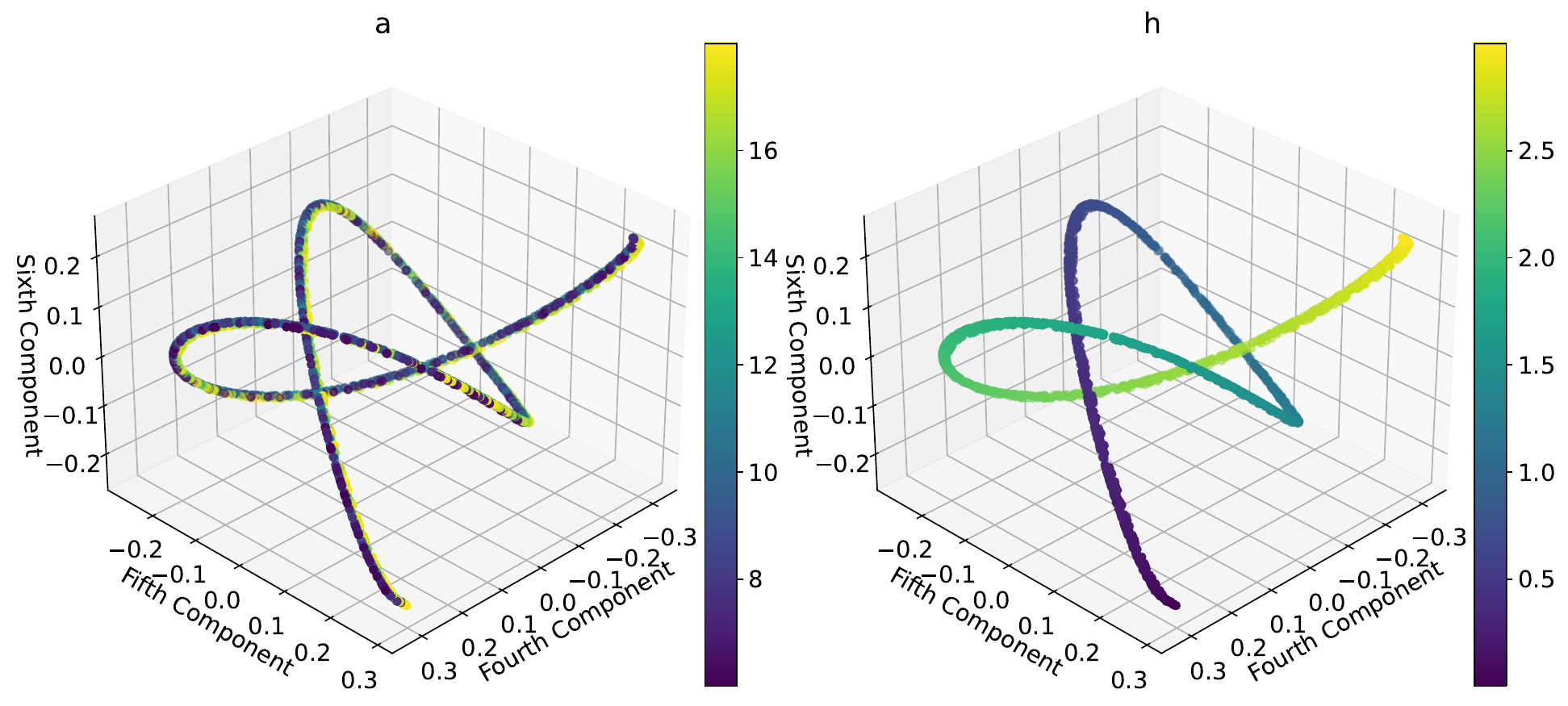}
    \label{fig:kdvpca6}
}
\caption{Nonlinearity of the initial conditions for the Burgers' equation. (a) shows the singular values of the data matrix. (b) shows the projection of data to the top 2 principal components and (c) shows the projection to top 3 principal components. (d) shows the projection to the 4th-6th principal components. In (b), (c) and (d), the projections are colored according to the $a$ parameter in the left subplots and according to the $h$ parameter in the right subplots.}
\label{fig:kdvnonlinear}
\end{figure}

\begin{figure}[h]
    \centering
    \includegraphics[height=4cm]{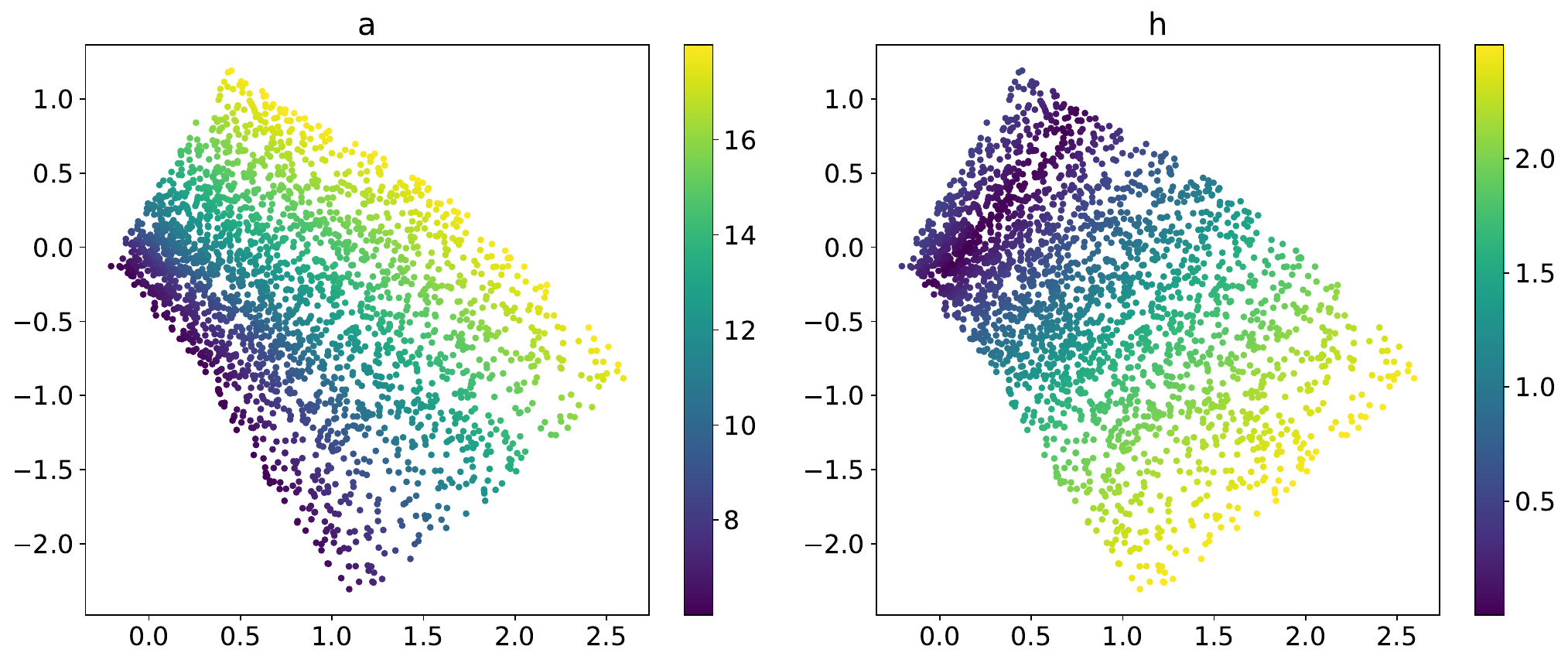}
    \caption{
    Latent features of the initial conditions $g_{a,h}$ (KdV) in \eqref{eq:fkdv} given by the Auto-Encoder. The left plot is colored according to $a$ and the right plot is colored according to $h$.
    }
    \label{fig:kdv:project}
\end{figure}

Figure \ref{fig:kdv} contains various plots comparing AENet, PCANet, and DeepONet for the KdV equation, analogous to the role Figure \ref{fig:transport} plays for the transport equation. Figure \ref{fig:kdv:examples} compares AENet with reduced dimension 2 to PCANet with reduced dimension 2 for domain and 40 for range.
Figure \ref{fig:kdvexperiment:generr} and \ref{fig:kdvexperiment:noise} display the absolute squared test error with respect to $n$ (in log-log plot) and noise variance respectively. 
Figures \ref{fig:kdvexperiment:generr} and \ref{fig:kdvexperiment:noise} are  generated using AENet with reduced dimension 2.
We further compare the relative test error (as a percent) of AENet, PCANet, and DeepONet in Table \ref{tab:table}(c).

\begin{figure}[h]
\centering
\subfigure[One draw from $\gamma_{kdv}$]{
    \includegraphics[width=5cm]{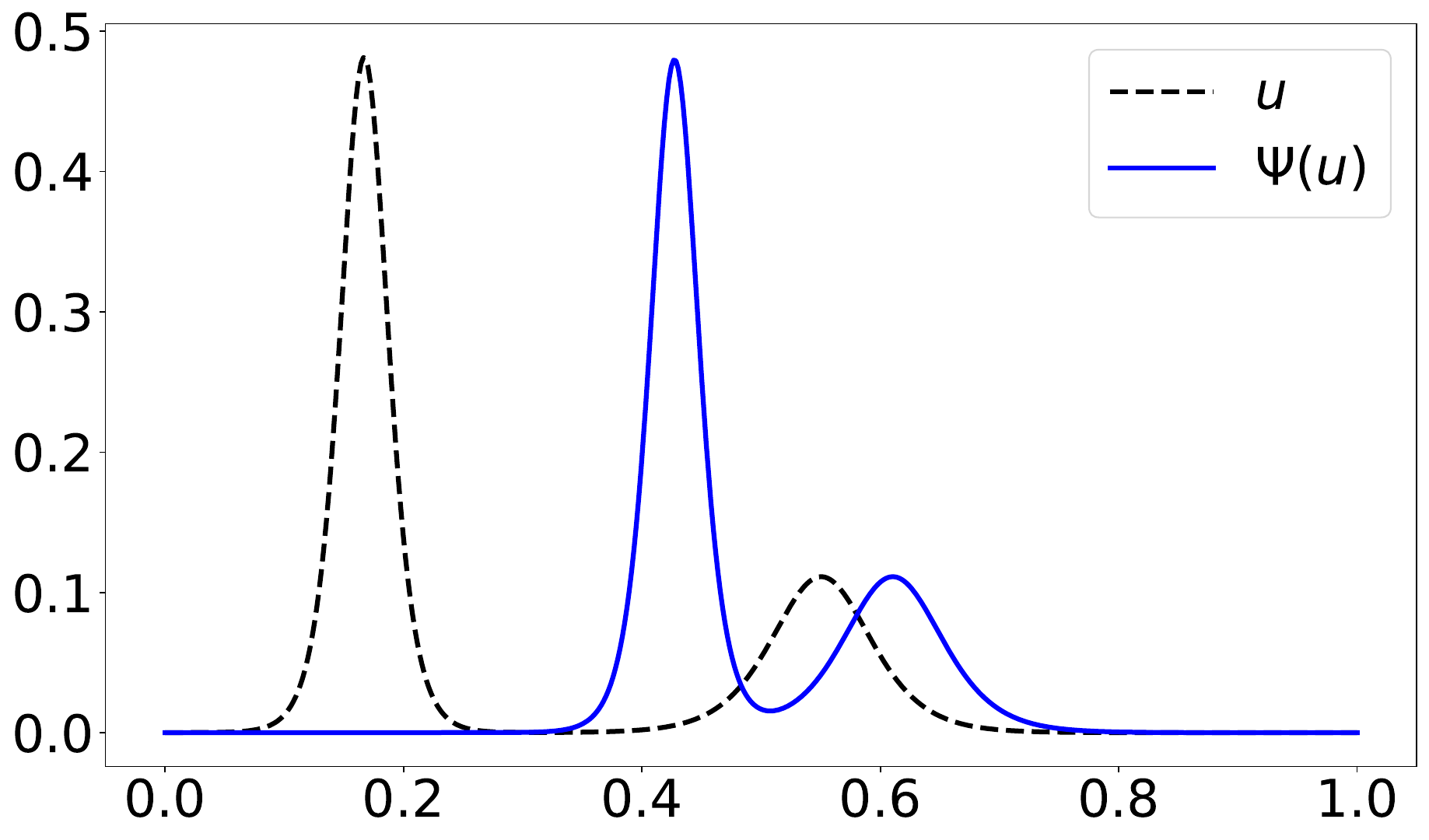}
    \label{fig:kdv:samples}
}
\subfigure[Relative projection error]{
    \includegraphics[width=5cm]{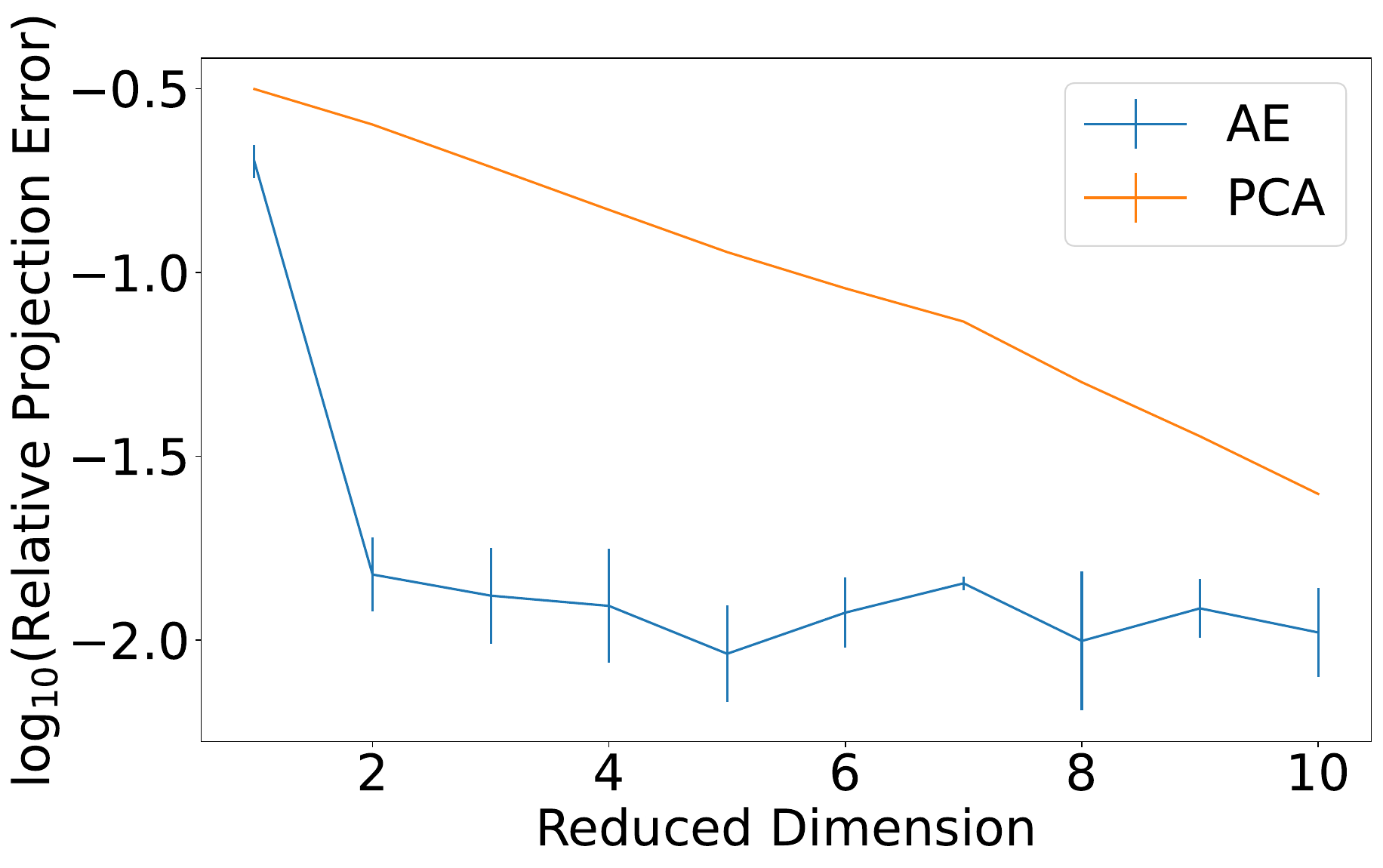}
    \label{fig:kdv:projerr}
}
\subfigure[Relative testing error]{
    \includegraphics[width=5cm]{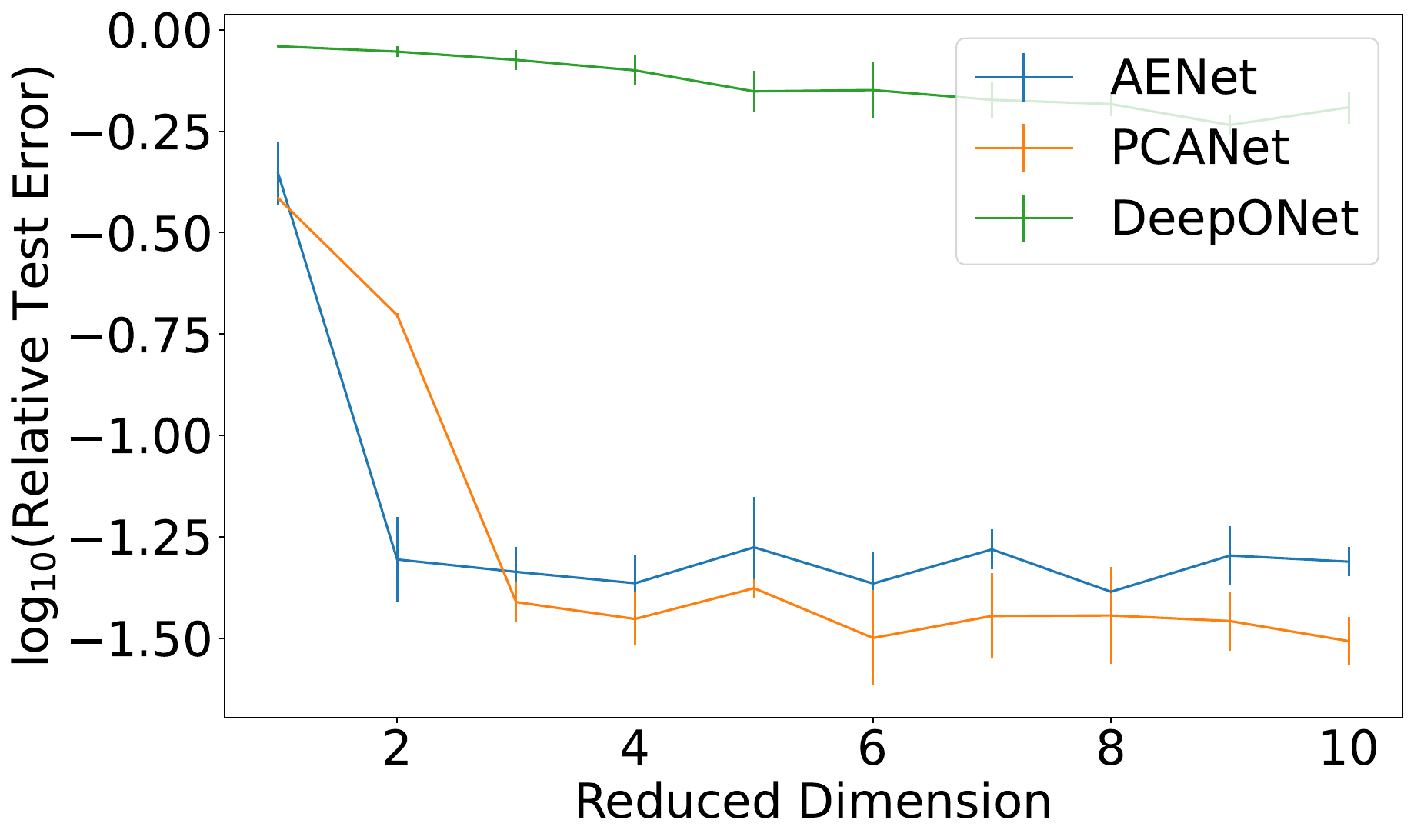}
    \label{fig:kdv:testerr}
}
\vspace*{-0.3cm}\\

\subfigure[Predicted solutions]{
    \includegraphics[width=5cm]{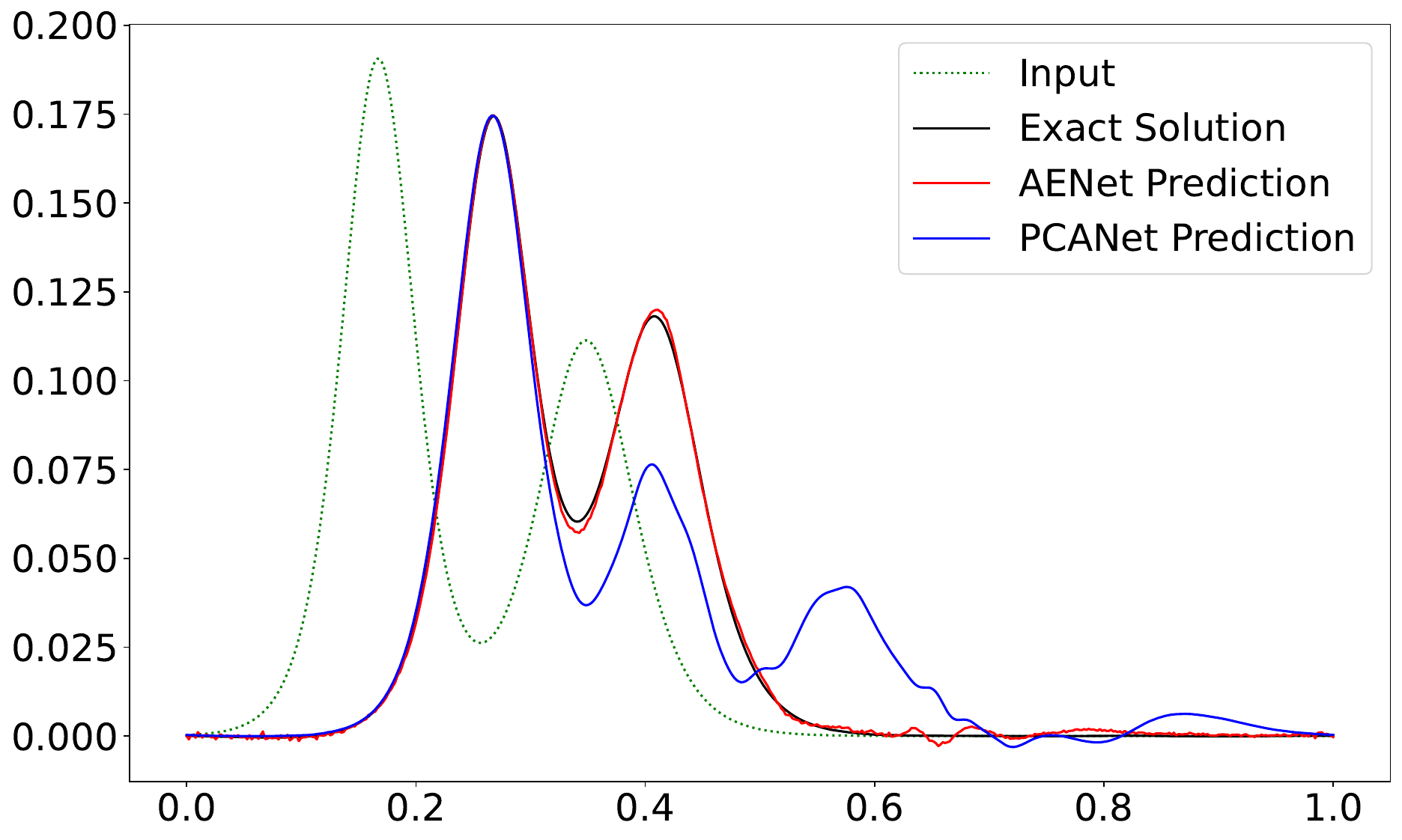}
    \label{fig:kdv:examples}
}
\subfigure[Squared test error versus $n$]{
    \includegraphics[width=5cm]{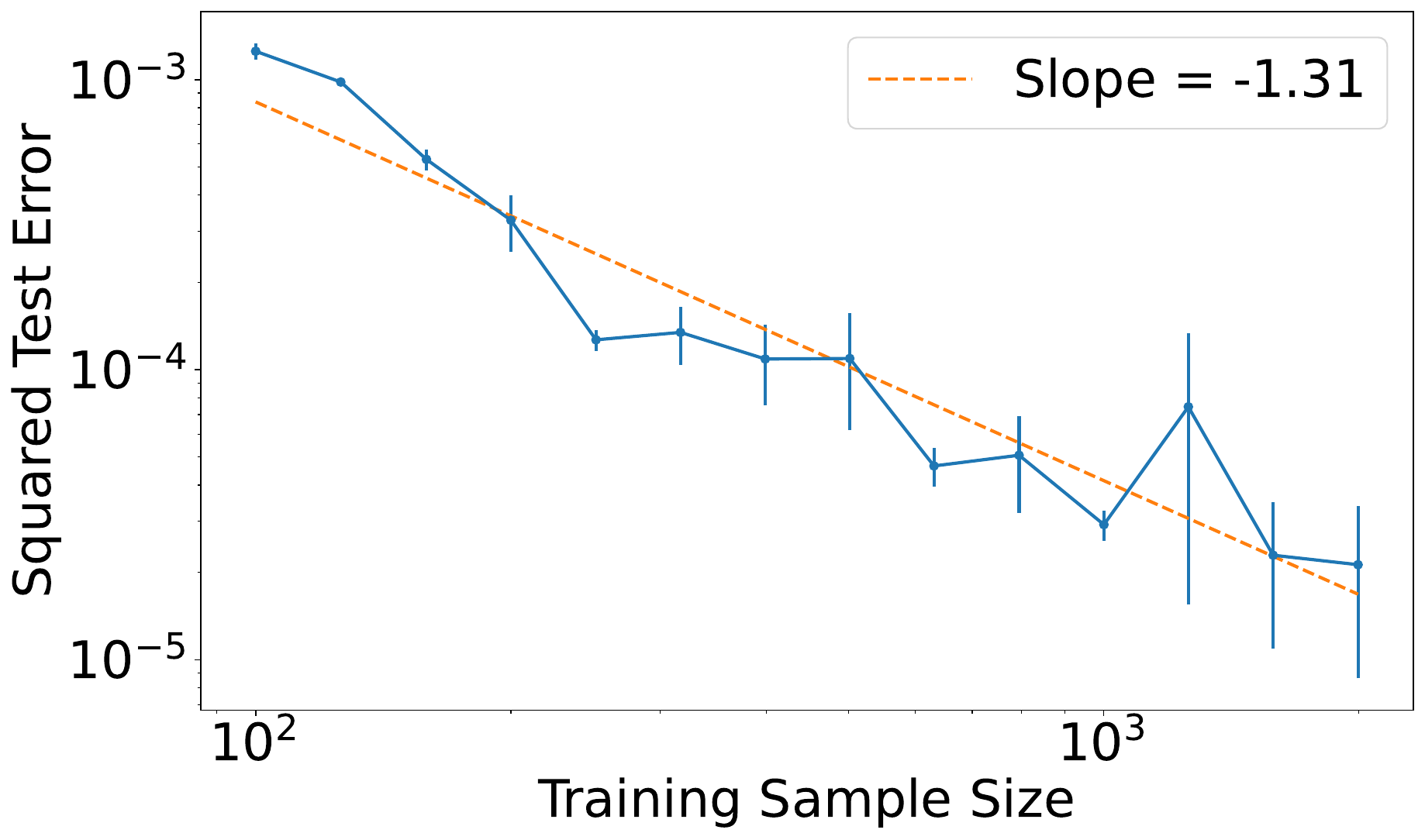}
    \label{fig:kdvexperiment:generr}
}
\subfigure[Robustness to noise]{
    \includegraphics[width=5cm]{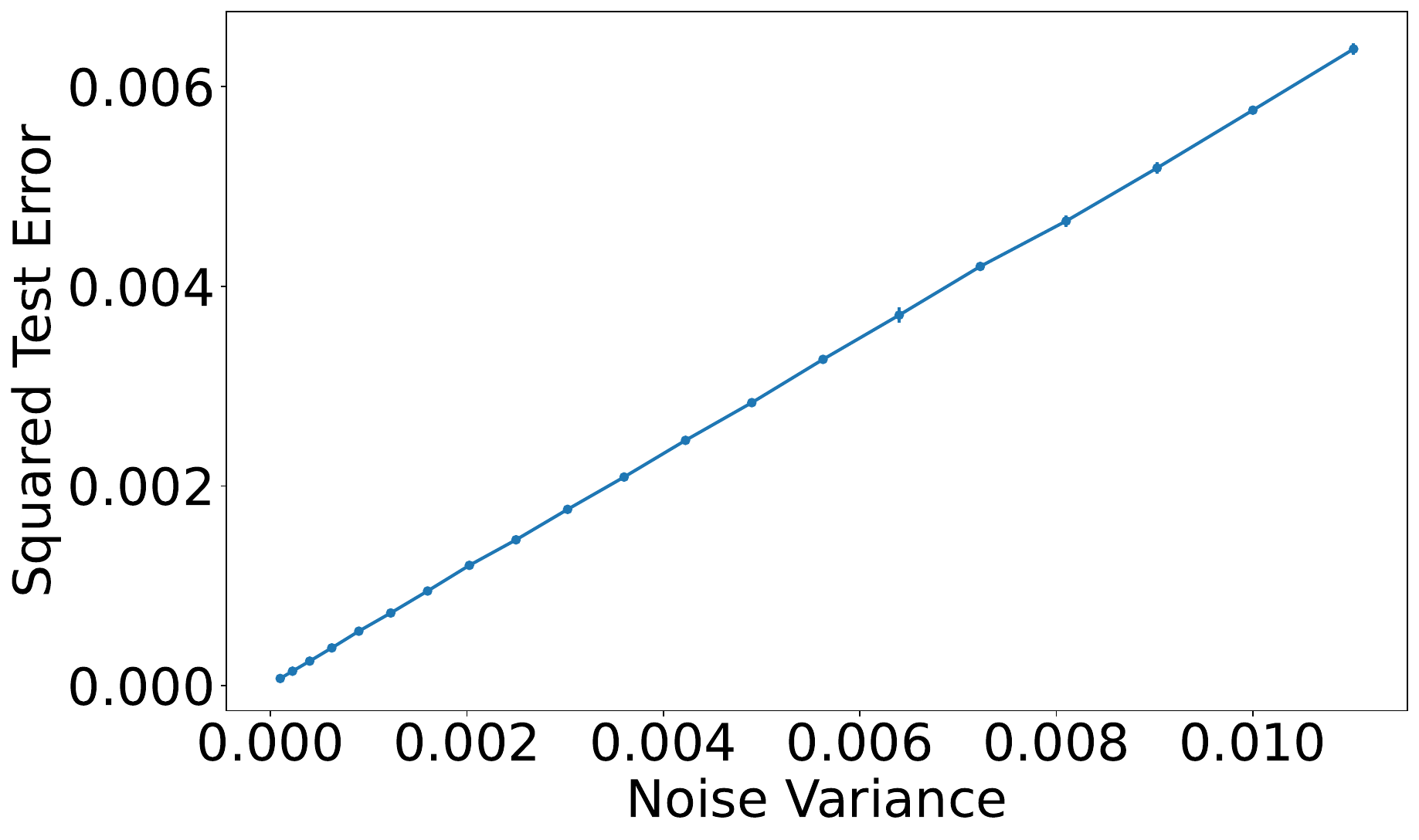}
    \label{fig:kdvexperiment:noise}
}
\caption{Results of the KdV equation.
}
\label{fig:kdv}
\end{figure}

\subsection{Comparison of relative test error}
\label{app:experiments:table}

We compare AENet, PCANet, and DeepONet with various reduced dimensions for the input on the three PDEs mentioned above (using a reduced dimension of 40 for the output of PCANet). We repeat the experiments 3 times and report the mean relative test error along with the standard deviation among the runs in Table \ref{tab:table}(a)-(c). DeepONet becomes successful when the reduced dimension is $100$ or more.

\begin{table}[h!]
    \centering
    \caption{Comparison table of the relative test error \textbf{(as a percent)} when the reduced dimension (in the first row) for the input  varies. Each column corresponds to a reduced dimension for the input. The first number is the mean, and in parenthesis is the standard deviation among 3 trials. \vhalf}
    \label{tab:table}

\subfigure[Transport equation]{
\setlength{\tabcolsep}{0.3em}
\centering
\scriptsize
\begin{tabular}{|l|l|l|l|l|l|l|l|l|l|}
\hline
\textbf{Method} & \textbf{1} & \textbf{2} & \textbf{4} & \textbf{6} & \textbf{8} & \textbf{10} & \textbf{20} & \textbf{40} & \textbf{100} \\\hline
\textbf{AENet} & 12.1 (0.2) & 1.1 (0.5) & 1.3 (0.7) & 1.0 (0.2) & 0.9 (0.1) & 1.5 (0.4) & 0.8 (0.2) & 0.9 (0.2) & 1.0 (0.1) \\\hline
\textbf{PCANet} & 38.6 (0.0) & 5.1 (0.2) & 0.9 (0.3) & 1.0 (0.2) & 1.1 (0.1) & 0.9 (0.0) & 1.1 (0.2) & 0.8 (0.3) & 0.7 (0.0) \\\hline
\textbf{DeepONet} & 96.4 (0.0) & 96.4 (0.0) & 96.4 (0.0) & 96.4 (0.0) & 96.4 (0.0) & 66.0 (6.2) & 33.7 (26.0) & 18.5 (8.3) & 5.0 (1.6) \\\hline
\end{tabular}
}

\subfigure[Burgers' equation]{
\setlength{\tabcolsep}{0.3em}
\centering
\scriptsize
\begin{tabular}{|l|l|l|l|l|l|l|l|l|l|}
\hline
\textbf{Method}   & \textbf{1}  & \textbf{2} & \textbf{4} & \textbf{6}  & \textbf{8}  & \textbf{10} & \textbf{20} & \textbf{40} & \textbf{100} \\ \hline
\textbf{AENet}    & 28.6 (4.2)  & 4.2 (0.6)  & 3.7 (0.8)  & 3.6 (0.2)   & 3.8 (0.4)   & 3.3 (0.4)   & 3.1 (0.5)   & 3.1 (0.2)   & 3.2 (0.2)    \\ \hline
\textbf{PCANet}   & 68.0 (0.1)  & 14.7 (0.2) & 6.2 (0.2)  & 6.3 (0.3)   & 6.3 (0.1)   & 6.1 (0.1)   & 6.1 (0.1)   & 6.5 (0.2)   & 6.6 (0.3)    \\ \hline
\textbf{DeepONet} & 100.0 (0.0) & 97.5 (4.3) & 97.5 (4.3) & 72.8 (16.7) & 70.1 (12.5) & 55.2 (21.0) & 31.2 (6.2)  & 18.2 (2.6)  & 10.7 (0.9)   \\ \hline
\end{tabular}
}
\subfigure[KdV equation]{
\setlength{\tabcolsep}{0.3em}
\centering
\scriptsize
\begin{tabular}{|l|l|l|l|l|l|l|l|l|l|}
\hline
\textbf{Method} & \textbf{1} & \textbf{2} & \textbf{4} & \textbf{6} & \textbf{8} & \textbf{10} & \textbf{20} & \textbf{40} & \textbf{100} \\\hline
\textbf{AENet} & 51.9 (0.6) & 5.7 (0.9) & 4.7 (0.5) & 5.2 (1.2) & 3.9 (0.1) & 4.4 (0.7) & 4.7 (1.0) & 4.8 (0.0) & 5.0 (0.3) \\\hline
\textbf{PCANet} & 38.7 (0.0) & 19.9 (0.4) & 4.0 (0.5) & 3.9 (0.4) & 4.8 (1.4) & 3.4 (0.2) & 3.7 (0.2) & 4.2 (0.7) & 3.8 (0.2) \\\hline
\textbf{DeepONet} & 91.1 (0.1) & 91.1 (0.1) & 68.9 (3.1) & 68.0 (2.2) & 62.6 (3.4) & 56.0 (2.9) & 42.5 (5.2) & 20.9 (0.7) & 9.1 (0.8) \\\hline
\end{tabular}
}
\end{table}

\begin{figure}[h]
    \centering
    \subfigure[Transport]{\includegraphics[width=5cm,height=3.2cm]{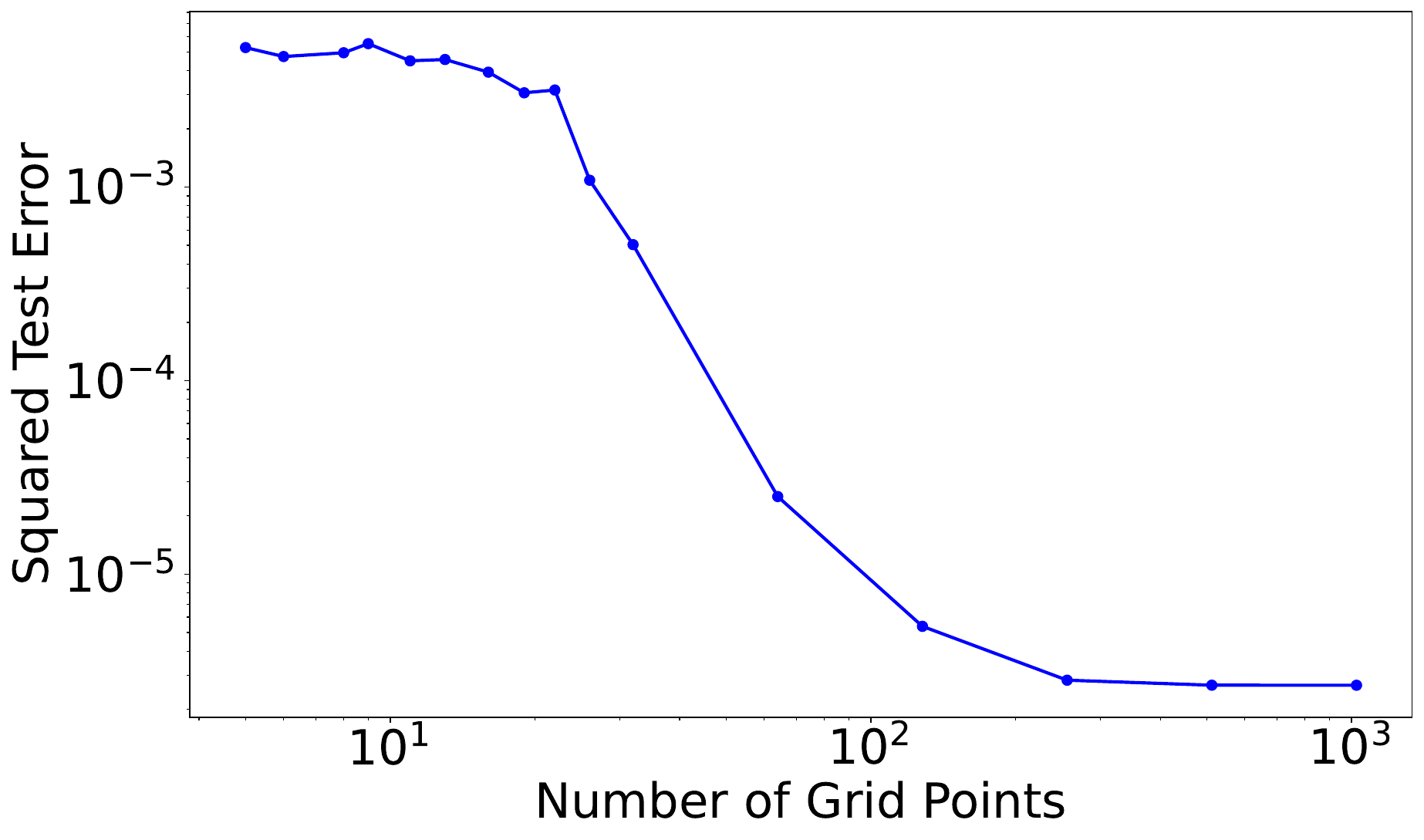}
    \label{fig:transport:interpolation}
    }
    \subfigure[Burgers']{\includegraphics[width=5cm,height=3.2cm]{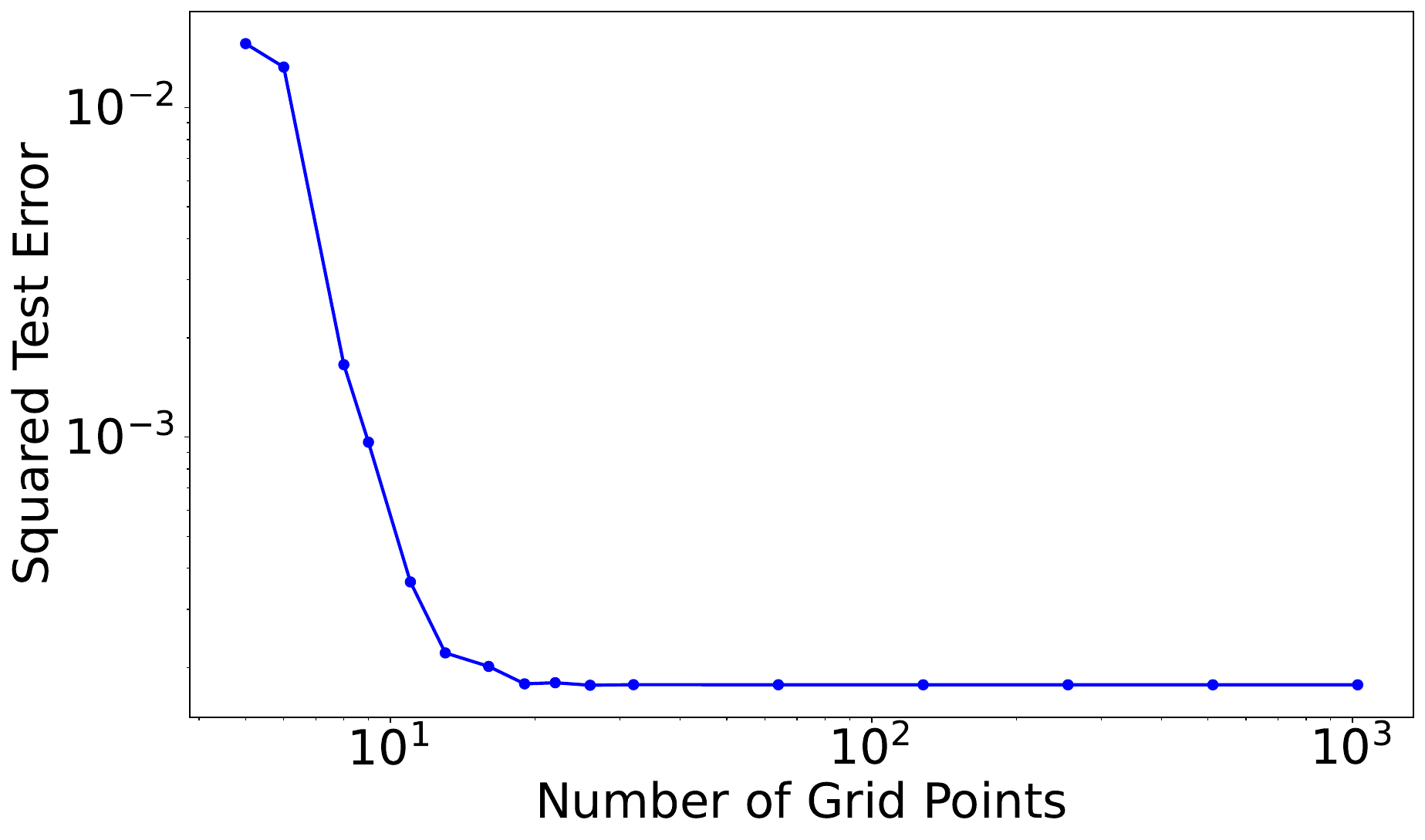}
    \label{fig:burgers:interpolation}
    }
    \subfigure[KdV]{\includegraphics[width=5cm,height=3.2cm]{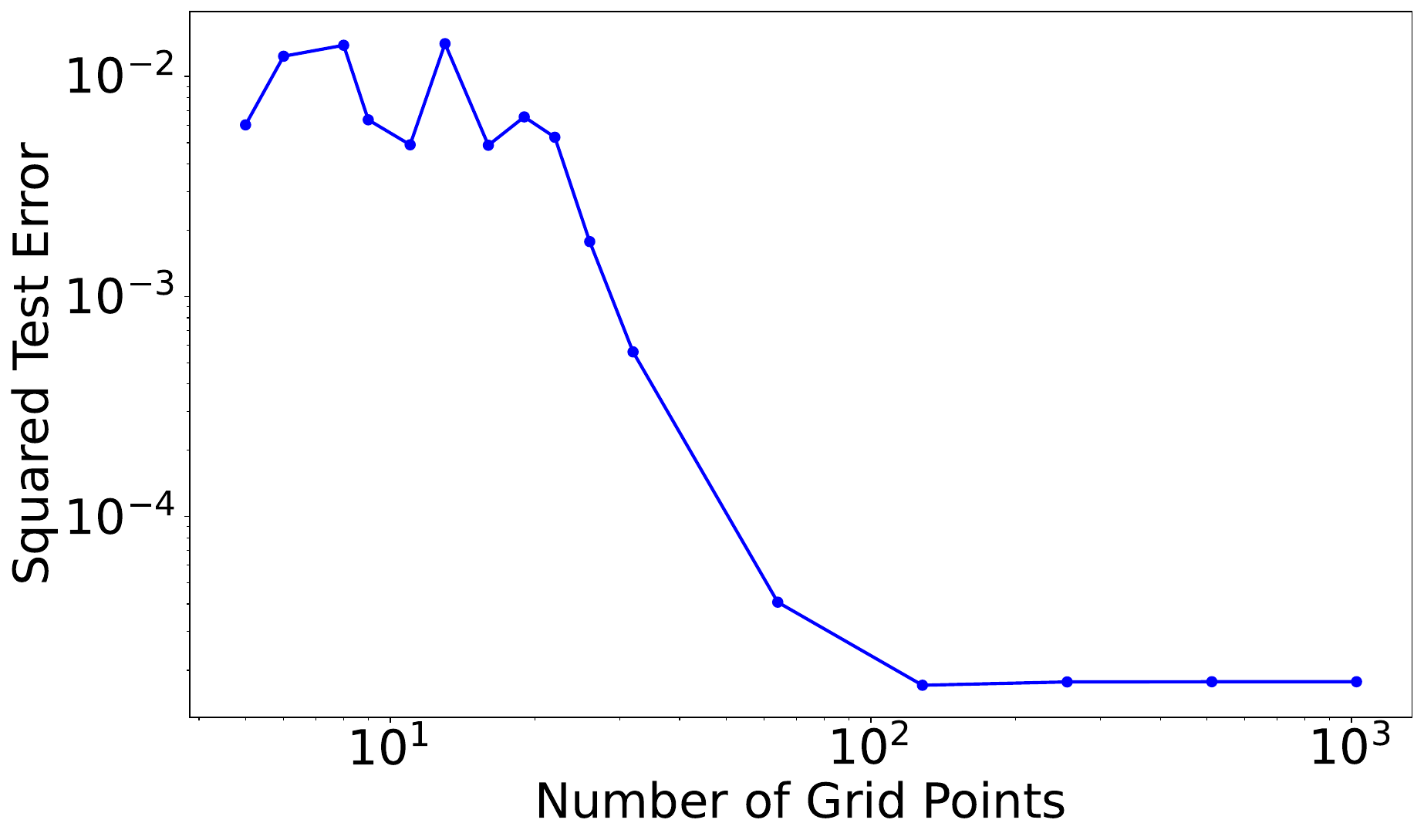}
    \label{fig:kdv:interpolation}
    }
    \caption{Squared test error of AENet versus  grid size of the test data for the transport equation in Subsection \ref{sec:transport}, the Burgers' equation in Subsection \ref{sec:burgers} and the KdV equation in Subsection \ref{sec:kdv}. To evaluate the operator for a test sample, we interpolate the test sample to the same grid as the training data. }
    \label{fig:interpolation}
\end{figure}

\subsection{Test data on a different grid as the training data}

Finally, we show the robustness of our method when the test data are sampled on a different grid from that of the training data, as shown in Remark \ref{remark.approx.interp} and \ref{remark.gene.interp}. Our training data are sampled on a uniform grid with $256$ grid points. When the test data are sampled on a different grid, we interpolate the test data to the same grid as the training data and then evaluate the operator. Figure \ref{fig:interpolation} shows the squared test error for operator prediction  when the test data are sampled on a different grid size for the  transport equation in Subsection \ref{sec:transport}, the Burgers' equation in Subsection \ref{sec:burgers} and the KdV equation in Subsection \ref{sec:kdv}. We used cubic interpolation for all equations. The operator prediction on test samples by this simple interpolation technique is almost discretization invariant as long as the test samples have a sufficient resolution. For the transport equation, the squared test error is almost the same when the grid size is more than $10^2$. For the Burgers' equation and the KdV equation, the squared test error is almost the same when the grid size is more than $10^1$ and $10^2$, respectively.

\commentout{
\begin{figure}
    \centering
    \includegraphics[height=3.8cm]{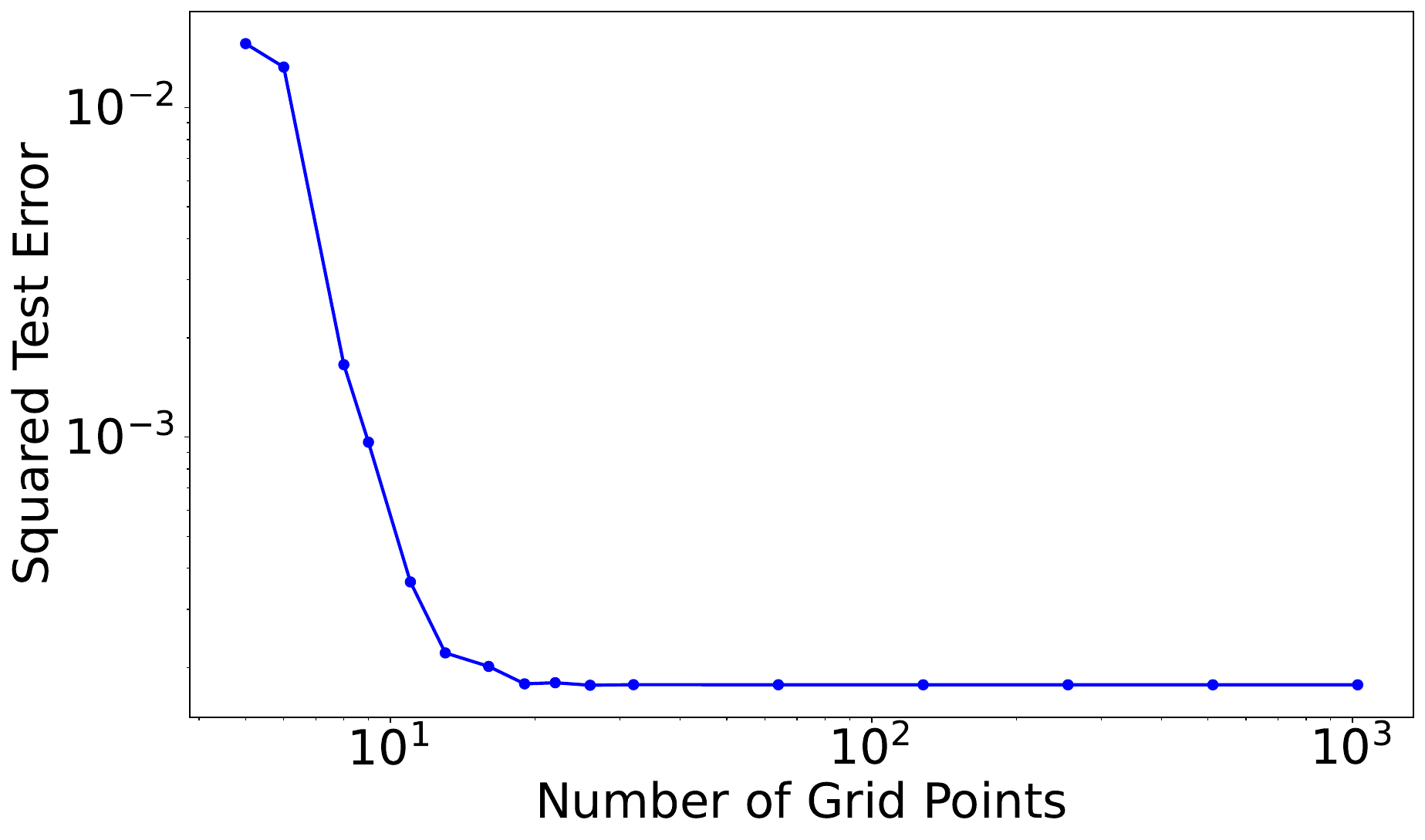}
    \caption{Squared generalization error of AENet versus grid size (via cubic interpolation) for Burgers}
\end{figure}
\begin{figure}
    \centering
    \includegraphics[height=3.8cm]{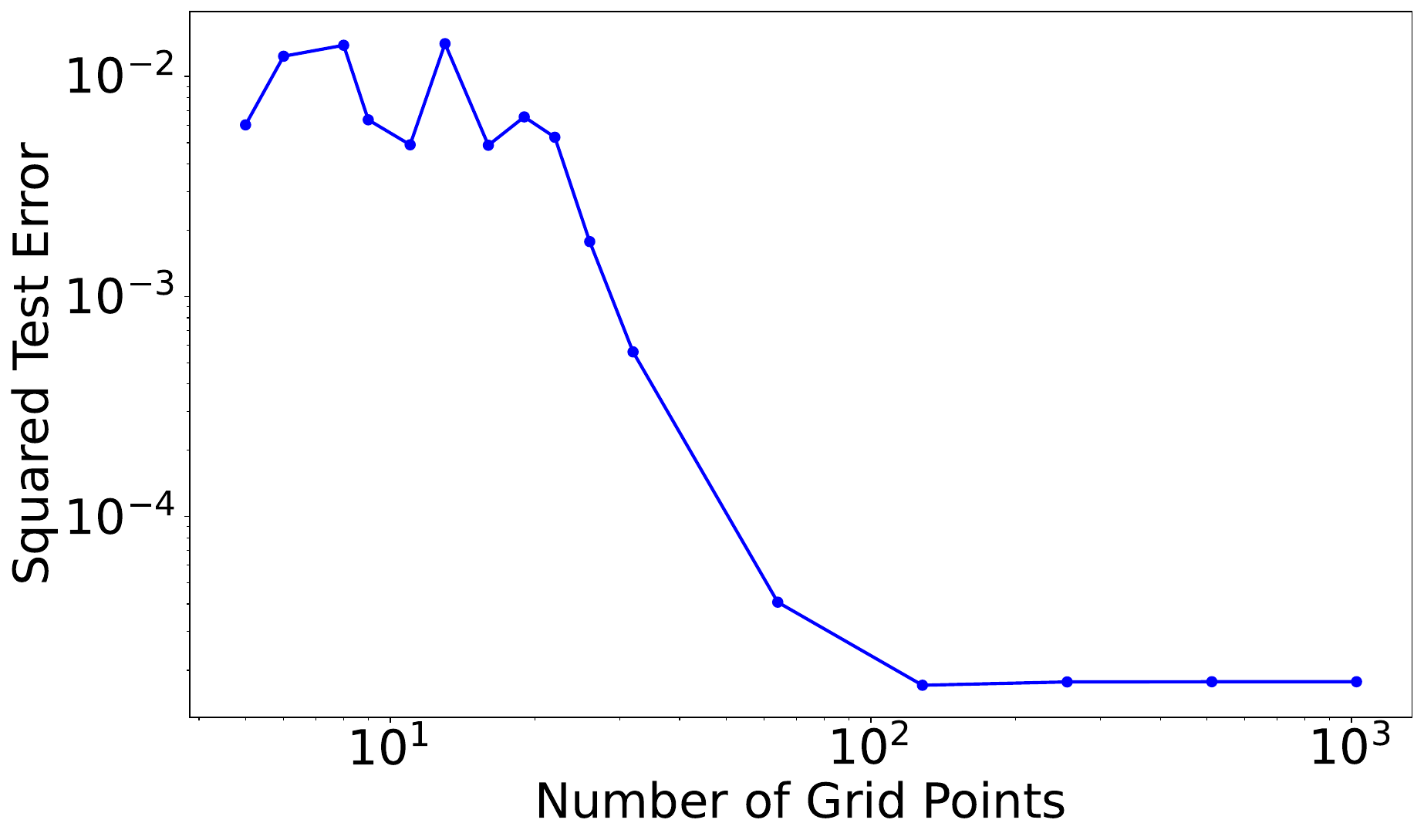}
    \caption{Squared generalization error of AENet versus grid size (via cubic interpolation) for KdV}
\end{figure}
}

\section{Proof of main results}
\label{sec:proof}

In this section, we present the proof of our main results: Theorem \ref{thm:approximation}, Corollary \ref{coro:approximation} and Theorem \ref{thm:generalization}. Proofs of lemmas are given in Appendix \ref{applemma}.

\subsection{Proof of the approximation theory in Theorem \ref{thm:approximation}}
\label{proof:thm:approximation}
\begin{proof}[Proof of Theorem \ref{thm:approximation}]
We will prove the approximation theory for the Auto-Encoder of the input $\widetilde{u}$ and the transformation $\Gamma$ in order.

\noindent $\bullet$ {\bf Approximation theory of Auto-Encoder for the input $\widetilde{u}$:} 
    We first prove an approximation theory for the Auto-Encoder of  $\widetilde{u} =\cSX(u)$. We will show that $\widetilde{\fb}$ and $\widetilde{\gb}$ can be well approximated  by neural networks.
    Note that $\widetilde{\fb}$ is only defined on $\widetilde{\cM}$. The following lemma shows that it can be extended to the cubical domain $[-R_{\cX},R_{\cX}]^{D_1}$ while keeping the same Lipschitz constant.
\begin{lemma}[Kirszbraun theorem \citep{kirszbraun1934zusammenziehende}]
\label{kem:extension}
    If $E\subset \RR^D$
, then any Lipschitz function $\fb: E\rightarrow \RR^d$
 can be extended to the whole $\RR^D$
 keeping the Lipschitz constant of the original function.
\end{lemma}
In the rest of this paper, without other specification, we use $\widetilde{\fb}$ to denote the extended function.

    For the network construction to approximate $\widetilde{\fb}$, we will use the following neural network approximation result on a set in $\RR^{D}$ with Minkowski dimension $d<D$, which is a variant of \cite[Theorem 5]{nakada2020adaptive} (see a proof in Appendix \ref{proof:Ex:0}):
    \begin{lemma}
    \label{lem:Ex:0}
        Let $D,d$ be positive integers with $d<D$, $M>0$ and $\Xi\subset [0,1]^D$. Suppose $d\geq d_M \Xi$. For any $\varepsilon>0$, consider a network class $\cF_{\rm NN}(D,1,L,p,K,\kappa,M)$ with
        \begin{align*}
    L=O(\log \varepsilon^{-1}),\ p=O(\varepsilon^{-d}),  \ K=O(\varepsilon^{-d}), \ \kappa=O(\varepsilon^{-3-4(1+\lceil\log_2 M \rceil}).
        \end{align*}
        Then for $\varepsilon\in (0,1/4)$ and any Lipschitz function $f^*: [0,1]^{D} \rightarrow [-M,M]$ with function value and Lipschitz constant bounded by $M$, there exists a network $f_{\rm NN} \in \cF_{\rm NN}(D,1,L,p,K,\kappa,M)$ satisfying
        \begin{align*}
            \|f_{\rm NN}-f^*\|_{L^\infty(\Xi)} \leq \varepsilon.
        \end{align*}
        The constant hidden in $O(\cdot)$ depends on $d,M$ and is only polynomial in $D$. 
    \end{lemma}

To use Lemma \ref{lem:Ex:0} to derive an approximation result of $\widetilde{\fb}$, we need an upper bound on $\sup_{\widetilde{u}\in [-R_{\cX},R_{\cX}]^{D_1}}\|\widetilde{\fb}(\widetilde{u})\|_{\infty}$. Note that $\widetilde{\fb}$ is the extended function from $\widetilde{\cM}$ to $[-R_{\cX},R_{\cX}]^{D_1}$. 
Even though $\|\widetilde{\fb}(\widetilde{u})\|_{\infty}$ is bounded by $1$ for $\widetilde u \in \widetilde{\cM}$, $\|\widetilde{\fb}(\widetilde{u})\|_{\infty}$ may exceed $1$ for $\widetilde{u} \in [-R_{\cX},R_{\cX}]^{D_1}$.
Since $\widetilde{\fb}(\widetilde{u})\in [-1,1]^d$ for any $\widetilde{u}\in \widetilde{\cM}$ and we are building the approximation theory of $\widetilde{\fb}$ on $\widetilde{\cM}$, 
we can clip the value of $\widetilde{\fb}(\widetilde u)$ for  $\widetilde u \in [-R_{\cX},R_{\cX}]^{D_1}$ to $[-1,1]^d$. Specifically, we introduce the clipping operator  
\begin{align}
    \CL(\widetilde{\fb})=\min\{ \max\{ \widetilde{\fb},-1\}, 1\},
    \label{eq:clip}
\end{align}
where $\min,\max$ are applied element-wisely. The operator $\CL(\widetilde{\fb})$ clips the outputs of $\widetilde{\fb}$ to $[-1,1]^d$. It is easy to show that the functions $\CL(\widetilde{\fb})$ is Lipschitz with the same Lipschitz constant as $\widetilde{\fb}$.

Now, we are ready to conduct approximation analysis on $\widetilde{\fb}:\widetilde{\mathcal{M}}\rightarrow [-1,1]^d$. Denote $\widetilde{\fb}=[\widetilde{f}_{1},...,\widetilde{f}_{d}]^{\top}$. For $k=1,...,d$, by Lemma \ref{lemma:f1g1lip}, each $\CL(\widetilde{f}_{k})$ is a function from $[-R_{\cX},R_{\cX}]^{D_1}$  to $[-1,1]$ and with Lipschitz constant $2L_{\fb}$. According to Lemma \ref{lemma:tildemmindim}, we have $d_M \cSX(\cM) \le d$. For any $\varepsilon_1>0$, by Lemma \ref{lem:Ex:0} with a proper scaling and shifting, there exists a network architecture $\cF_{\rm NN}(D_1,1,L_{4},p_{4},K_{4},\kappa_{4},M_{4})$ with 
    \begin{align*}
    &L_{4}=O(\log \varepsilon^{-1}),\ p_{4}=O(\varepsilon_1^{-d}), \ K_{4}=O(\varepsilon_1^{-d}), \ \kappa_{4}=O(\varepsilon_1^{-7}), \ M_{4}=1
        \end{align*} 
    so that for each $\widetilde{f}_{k}$, there exists $\widetilde{f}_{{\rm NN},k}\in \cF_{\rm NN}(D_1,1,L_4,p_4,K_4,\kappa_4,M_4)$ satisfying
    \begin{align*}
        \|\widetilde{f}_{{\rm NN},k}- \CL(\widetilde{f}_{k})\|_{L^\infty(\wtcM)}\leq \varepsilon.
    \end{align*}
    The constant hidden in $O(\cdot)$ depends on $d,L_{\fb},R_{\cX}$ and polynomially in $D_1$.
    
    Define the network $\widetilde{\fb}_{\rm NN}=[\widetilde{f}_{{\rm NN},1},...,\widetilde{f}_{{\rm NN},d}]^{\top}$ as the concatenation of $\widetilde{f}_{{\rm NN},k}$'s. We have  
    \begin{align*}
        \sup_{\widetilde{u}\in \widetilde{\cM}}\|\widetilde{\fb}_{\rm NN}(\widetilde{u})- \widetilde{\fb}(\widetilde{u})\|_{\infty}=\sup_{\widetilde{u}\in \widetilde{\cM}}\|\widetilde{\fb}_{\rm NN}(\widetilde{u})- \CL\circ\widetilde{\fb}(\widetilde{u})\|_{\infty}\leq \varepsilon_1
    \end{align*} 
    since $\CL\circ\widetilde{\fb}(\widetilde{u})=\widetilde{\fb}(\widetilde{u})\in [-1,1]^d$ for any $u\in \widetilde{\cM}$.
    
  Furthermore, we have
    $\widetilde{\fb}_{\rm NN}\in \cF_{\rm NN}^{E_{\cX}}=\cF_{\rm NN}(D_1,d,L_{1},p_{1},K_{1},\kappa_{1},M_{1})$  with
    \begin{align*}
    &L_{1}=L_4=O(\log \varepsilon^{-1}),\ p_{1}=dp_4=O(\varepsilon_1^{-d}),\ K_{1}=dK_4=O(\varepsilon_1^{-d}), \ \kappa_{1}=\kappa_4=O(\varepsilon_1^{-7}), \ M_{1}=1.
        \end{align*}

    For the network approximation of $\widetilde{\gb}$, we will use the following result:
    \begin{lemma}[Theorem 1 of \cite{yarotsky2017error}]
    \label{lem:Sobolev}
        For any $\varepsilon\in(0,1)$, there is a network architecture $\cF_{\rm NN}(D,1,L,p,K,\kappa,M)$ with 
        \begin{align*}
            &L=O(\log \varepsilon^{-1}), \ p=O(\varepsilon^{-D}), \ K=O(\varepsilon^{-D}\log \varepsilon^{-1}), \ \kappa= O(\varepsilon^{-1}), \ M=1
        \end{align*}
        so that for any Lipschitz function $f^*:[0,1]^D\rightarrow[-1,1]$ with Lipschitz constant bounded by 1, there exists a network $f_{\rm NN}\in \cF_{\rm NN}(D,1,L,p,K,\kappa,M)$ satisfying
        \begin{align*}
            \|f_{\rm NN}-f^*\|_{L^\infty([0,1]^D)}\leq \varepsilon.
        \end{align*}
        The constants hidden in $O(\cdot)$ depends on $D$.
    \end{lemma}
    
    Denote $\widetilde{\gb}=[\widetilde{g}_{1},...,\widetilde{g}_{D_1}]$. By Lemma \ref{lemma:f1g1lip}, each $\widetilde{g}_{k}$ is Lipschitz with Lipschitz constant $2L_{\gb}$. By Lemma \ref{lem:Sobolev} with a proper scaling and shifting, for any $\varepsilon_2\in(0,1)$, there exists a network architecture $\cF_{\rm NN}(d,1,L_5,p_5,K_5,\kappa_5,M_5)$ with
    \begin{align*}
            &L_5=O(\log (1/\varepsilon_2)),\  p_5=O(\varepsilon_2^{-d}), \ K_5=O(\varepsilon_2^{-d}\log (1/\varepsilon_2)), \ \kappa_5= O(\varepsilon_2^{-1}), \ M_5=R_{\cX}
        \end{align*}
    so that for each $g_{1,k}$, there is a $\widetilde{g}_{{\rm NN},k}\in \cF_{\rm NN}(d,1,L_5,p_5,K_5,\kappa_5,M_5)$ satisfying
    \begin{align}
        \|\widetilde{g}_{{\rm NN},k}-\widetilde{g}_{k}\|_{L^\infty([-1,1]^d)}\leq \varepsilon_2.
        \label{eqapproxg1}
    \end{align}
    The constant hidden in $O(\cdot)$ depends on $d,L_{\gb}$ and $R_{\cX}$.

    Define $\widetilde{\gb}=[\widetilde{g}_{{\rm NN},1},...,\widetilde{g}_{{\rm NN},D_1}]^{\top}$ as the concatenation of $\widetilde{g}_{k}$'s. According to \eqref{eqapproxg1}, we have 
    \begin{align*}
        \sup_{\zb\in[-1,1]^d}\|\widetilde{\gb}_{\rm NN}(\zb)-\widetilde{\gb}(\zb)\|_{\infty}\leq \varepsilon_2,
    \end{align*}
    and $\widetilde{\gb}_{\rm NN}\in \cF_{\rm NN}^{D_{\cX}}=\cF_{\rm NN}(d,D_1,L_{2},p_{2},K_{2},\kappa_{2},M_{2})$ with
    \begin{align*}
        &L_{2}=L_5=O(\log (1/\varepsilon_2)), \ p_{2}=D_1p_5=O(\varepsilon_2^{-d}), \\ &K_{2}=D_1K_5=O(\varepsilon_2^{-d}\log (1/\varepsilon_2)), \ \kappa_{2}=\kappa_5= O(\varepsilon_2^{-1}),\  M_{2}=R_{\cX}.
    \end{align*}
    The constant hidden in $O(\cdot)$ depends on $d,L_{\gb},R_{\cX}$ and is linear in $D_1$.

    As a result, we have the following for any $u \in \cM$:
    \begin{align}
& \|\widetilde{\gb}_{\rm NN}\circ\widetilde{\fb}_{\rm NN}\circ\cS_{\cX}(u)- \widetilde{\gb}\circ\widetilde{\fb}\circ \cS_{\cX}(u)\|_{\infty} \nonumber\\
        \leq & \|\widetilde{\gb}_{\rm NN}\circ\widetilde{\fb}_{\rm NN}\circ\cS_{\cX}(u)- \widetilde{\gb}\circ\widetilde{\fb}_{\rm NN}\circ \cS_{\cX}(u)\|_{\infty}  + \|\widetilde{\gb}\circ\widetilde{\fb}_{\rm NN}\circ\cS_{\cX}(u)- \widetilde{\gb}\circ\widetilde{\fb}\circ \cS_{\cX}(u)\|_{\infty} \nonumber\\
        \leq& \varepsilon_2+ 2\sqrt{d}L_{\gb}\varepsilon_1.
    \end{align}

\noindent $\bullet$ {\bf Approximation theory for the transformation $\Gamma$ from the input latent variable to the output:} 
     The network approximation of the transformation $\Gamma$ can be proved using Lemma \ref{lem:Sobolev}. 
    First, we notice that the operator $\Gamma$ is Lipschitz (see a proof in Appendix \ref{proof:GammaLip}).
    \begin{lemma}\label{lem:GammaLip}
    The operator $\Gamma$ defined in \eqref{eqGamma} is Lipschitz with Lipschitz constant $2L_{\Psi}L_{\gb}$:
    \begin{align}
        \|\Gamma(\zb_1)-\Gamma(\zb_2)\|_{\cS_{\cY}}\leq 2L_{\Psi}L_{\gb}\|\zb_1-\zb_2\|_2
    \end{align}
    for any $\zb_1,\zb_2\in [-1,1]^d$.
\end{lemma}

     Denote $\Gamma=[\Gamma_1,...,\Gamma_{D_2}]^{\top}$. According to Lemma \ref{lem:GammaLip}, each $\Gamma_k$ is Lipschitz with Lipschitz constant $2L_{\Psi}L_{\gb}$. By Lemma \ref{lem:Sobolev} with proper scaling and shifting, for any $\varepsilon_3\in(0,1)$, there exists a network architecture $\cF_{\rm NN}(d,1,L_6,p_6,K_6,\kappa_6,M_6)$ with
    \begin{align*}
            &L_6=O(\log (1/\varepsilon_3)), \ p_6=O(\varepsilon_3^{-d}), \ K_6=O(\varepsilon_3^{-d}\log (1/\varepsilon_3)), \ \kappa_6= O(\varepsilon_3^{-1}), M_6=R_{\cY}
        \end{align*}
    so that for each $\Gamma_{k}$, there is a $\Gamma_{{\rm NN},k}\in \cF_{\rm NN}(d,1,L_6,p_6,K_6,\kappa_6,M_6)$ satisfying
    \begin{align*}
        \|\Gamma_{{\rm NN},k}-\Gamma_{k}\|_{L^\infty((-1,1]^d)}\leq \varepsilon_3.
    \end{align*}
    The constant hidden in $O(\cdot)$ depends on $d,L_{\gb},L_{\Psi}$ and $R_{\cY}$. 

    Define the network $\Gamma_{\rm NN}=[\Gamma_{{\rm NN},1},...,\Gamma_{{\rm NN},d}]^{\top}$ as the concatenation of the $\Gamma_{{\rm NN},k}$'s. Then we have  
    \begin{align*}
     \sup_{\zb\in [-1,1]^d}   \|\Gamma_{\rm NN}(\zb)- \Gamma(\zb)\|_{\infty}\leq \varepsilon_3.
    \end{align*} 
Furthermore, we have
    $\Gamma_{\rm NN}\in \cF_{\rm NN}^{\Gamma}=\cF_{\rm NN}(d,D_2,L_{3},p_{3},K_{3},\kappa_{3},M_{3})$ with
    \begin{align*}
        &L_{3}=L_6=O(\log (1/\varepsilon_3)), \ p_{3}=dp_6=O(\varepsilon_3^{-d}), \\
        &K_{3}=dK_6=O(\varepsilon_3^{-d}\log (1/\varepsilon_3)), \ \kappa_{3}=\kappa_6= O(\varepsilon_3^{-1}),\  M_{3}=R_{\cY}.
    \end{align*}
    The constant hidden in $O(\cdot)$ depends on $d,L_{\gb},L_{\Psi}, R_{\cY}$ and is linear in $D_2$.
\end{proof}

\subsection{Proof of Corollary \ref{coro:approximation}}
\label{proof:coro:approximation}
\begin{proof}[Proof of Corollary \ref{coro:approximation}]
    For any $\widetilde{u}\in\widetilde{\cM}$, we have
    \begin{align*}
        \|\Phi_{\rm NN}(\widetilde{u})-\Phi(\widetilde{u})\|_{\infty}
        = & \| \Gamma_{\rm NN}\circ \widetilde{\fb}_{\rm NN}(\widetilde{u})-\Gamma\circ\widetilde{\fb}(\widetilde{u})\|_{\infty}\\
        \leq &\|\Gamma_{\rm NN}\circ \widetilde{\fb}_{\rm NN}(\widetilde{u})-\Gamma\circ \widetilde{\fb}_{\rm NN}(\widetilde{u})\|_{\infty}+\|\Gamma\circ \widetilde{\fb}_{\rm NN}(\widetilde{u})-\Gamma\circ \widetilde{\fb}(\widetilde{u})\|_{\infty}\\
        \leq& \|\Gamma_{\rm NN}-\Gamma\|_{\infty} +2L_{\Psi}L_{\gb}\|\widetilde{\fb}_{\rm NN}-\widetilde{\fb}\|_{\infty}
        \leq \varepsilon_3+2\sqrt{d}L_{\Psi}L_{\gb}\varepsilon_1 
        = \varepsilon.
    \end{align*}
\end{proof}

\subsection{Proof of the generalization theory in Theorem \ref{thm:generalization}}
\label{proof:generalization}
In this section, we first give an upper bound of the generalization error of Stage I in Section \ref{appgeneralizationstagei}. The generalization error combining Stage I and II is analyzed in Section \ref{appgeneralizationstageii}.
\subsubsection{An upper bound for the generalization error in Stage I}
\label{appgeneralizationstagei}
In Stage I, the encoder $E_{\cX}^n$ and decoder $D_{\cX}^n$ are learned based on the first half data $\cJ_1$. We expect that $D_{\cX}^n\circ E_{\cX}^n(\widetilde{u})$ is close to $\widetilde{u}$ for any $\widetilde{u}\in \widetilde{\cM}$. We study the generalization error
\begin{align}
    \EE_{\cJ_1}\EE_{u\sim \gamma} \|D_{\cX}^n\circ E_{\cX}^n(\widetilde{u})- \widetilde{u}\|_{\infty}^2.
    \label{eq:gene:stage1}
\end{align}

Let $\cF=\cF_{\rm NN}(d_1,d_2,L,p,K,\kappa,M)$ be a network class from $[-B,B]^{d_1}$ to $[-R,R]^{d_2}$ for some $B,R>0$. We denote $\cN(\delta,\cF,\|\cdot\|_{L^{\infty,\infty}})$ as the $\delta$-covering number of $\cF$, where the norm $\|\cdot\|_{L^{\infty,\infty}}$ is defined as $\|F_{\rm NN}\|_{L^{\infty,\infty}}=\sup_{\xb\in [-B,B]^{d_1}} \|F_{\rm NN}(\xb)\|_{\infty}$ for any $F_{\rm NN}\in \cF$.

An upper bound of the generalization error in (\ref{eq:gene:stage1}) is given in the following lemma
(see a proof in Appendix \ref{proof:autoXGene}).
    \begin{lemma}\label{lem:autoXGene}
        Suppose Assumption \ref{assumption:samplinglip} and \ref{assumptiong:globalparametrization} hold. For any $\varepsilon_1\in(0,1/4),\varepsilon_2\in(0,1)$, set the network architectures $ 
    \cF_{\rm NN}^{E_{\cX}}=\cF_{\rm NN}(D_1,d,L_{1},p_{1},K_{1},\kappa_{1},M_{1})$ with
    \begin{align*}
    &L_{1}=O(\log \varepsilon_1^{-1}),\ p_{1}=O(\varepsilon_1^{-d}),\
    K_{1}=O(\varepsilon_1^{-d}), \ \kappa_{1}=O(\varepsilon_1^{-7}), \ M_{1}=1,
    \end{align*}
    and 
    $\cF_{\rm NN}^{D_{\cX}}=\cF_{\rm NN}(d,D_1,L_{2},p_{2},K_{2},\kappa_{2},M_{2})$ with
    \begin{align*}
        &L_{2}=O(\log \varepsilon_2^{-1}), \ p_{2}=O(\varepsilon_2^{-d}), \ K_{2}=O(\varepsilon_2^{-d}\log  \varepsilon_2^{-1}), \ \kappa_{2}= O(\varepsilon_2^{-1}),\  M_{2}=R_{\cX}.
    \end{align*}
    Denote the network architecture 
    \begin{align*}
    \cF_{\rm NN}^G=\{G_{\rm NN}:\RR^{D_1}\rightarrow \RR^{D_1}| G_{\rm NN}=D_{\rm NN}\circ E_{\rm NN} \ \mbox{ for }\ D_{\rm NN}\in \cF_{\rm NN}^{D_{\cX}}, \ E_{\rm NN}\in \cF_{\rm NN}^{E_{\cX}}\}.
\end{align*}
    Let $E_{\cX}^n,D_{\cX}^n$ be the learned autoencdoer in Stage I given by \eqref{eq:autoencoder:loss}. For any $\delta>0$, we have
    \begin{align}
        \EE_{\cJ_1}\EE_{u\sim \gamma} \left[\|D_{\cX}^n\circ E_{\cX}^n(\widetilde{u})-\widetilde{u}\|_{\infty}^2\right]\leq &16dL_{\gb}^2\varepsilon_1^2+ 4\varepsilon_2^2 \nonumber\\
        &+ \frac{48R_{\cX}^2}{n}\log\cN\left(\frac{\delta}{4R_{\cX}},\cF_{\rm NN}^G,\|\cdot\|_{L^{\infty,\infty}}\right)+ 6\delta.
    \end{align}
    \end{lemma}
In Lemma \ref{lem:autoXGene}, $\varepsilon_1,\varepsilon_2$ correspond to the approximation error of $\cF_{\rm NN}^{E_{\cX}}$, $\cF_{\rm NN}^{D_{\cX}}$ in Theorem \ref{thm:approximation}, respectively. 
For any $\varepsilon\in(0,1/4)$, by Theorem \ref{thm:approximation}, we choose $\varepsilon_1=\varepsilon_2=\varepsilon$ and set the network architectures $ 
    \cF_{\rm NN}^{E_{\cX}}=\cF_{\rm NN}(D_1,d,L_{1},p_{1},K_{1},\kappa_{1},M_{1})$ with
    \begin{align}
    &L_{1}=O(\log \varepsilon^{-1}),\ p_{1}=O(\varepsilon^{-d}),\
    K_{1}=O(\varepsilon^{-d}), \ \kappa_{1}=O(\varepsilon^{-7}), \ M_{1}=1,
    \label{eq.thm2.FE}
    \end{align}
    and 
    $\cF_{\rm NN}^{D_{\cX}}=\cF_{\rm NN}(d,D_1,L_{2},p_{2},K_{2},\kappa_{2},M_{2})$ with
    \begin{align}
        &L_{2}=O(\log (1/\varepsilon)), \ p_{2}=O(\varepsilon^{-d}), \ K_{2}=O(\varepsilon^{-d}\log (1/\varepsilon)), \ \kappa_{2}= O(\varepsilon^{-1}),\  M_{2}=R_{\cX}.
        \label{eq.thm2.FD}
    \end{align}
There exist $\widetilde{\fb}_{\rm NN}\in \cF_{\rm NN}^{E_{\cX}}, \widetilde{\gb}_{\rm NN}\in \cF_{\rm NN}^{D_{\cX}}$ satisfying
    \begin{align*}
        \sup_{\widetilde{u}\in \widetilde{\cM}}\|\widetilde{\fb}_{\rm NN}(\widetilde{u})- \widetilde{\fb}(\widetilde{u})\|_{\infty}\leq \varepsilon,\quad 
        \sup_{\zb\in[-1,1]^d}\|\widetilde{\gb}_{\rm NN}(\zb)- \widetilde{\gb}(\zb)\|_{\infty}\leq \varepsilon.
    \end{align*} 
    The constant hidden in $O(\cdot)$ depends on $d,L_{\fb},L_{\gb},R_{\cX}$ and is polynomial in $D_1$. 

    By Lemma \ref{lem:autoXGene}, we have
    \begin{align}
        \EE_{\cJ_1}\EE_{u\sim \gamma} \left[\|D_{\cX}^n\circ E_{\cX}^n(\widetilde{u})-\widetilde{u}\|_{\infty}^2\right]\leq &(16dL_{\gb}^2+ 4)\varepsilon^2 \nonumber\\
        &+ \frac{48R_{\cX}^2}{n}\log\cN\left(\frac{\delta}{4R_{\cX}},\cF_{\rm NN}^G,\|\cdot\|_{L^{\infty,\infty}}\right)+ 6\delta.
        \label{eq:autoXGen:err:1}
    \end{align}
    The following lemma gives an upper bound of the covering number of any given network architecture:
\begin{lemma}[Lemma 5.3 of \cite{chen2022nonparametric}]
\label{lem:covering}
    Let $\cF_{\rm NN}(d_1,d_2,L,p,K,\kappa,M)$ be a network architecture from $[-B,B]^{d_1}$ to $[-R,R]^{d_2}$ for some $B,R>0$. We have
    \begin{align*}
        \cN(\delta,\cF_{\rm NN},\|\cdot\|_{L^{\infty,\infty}})\leq \left( \frac{2L^2(pB+2)\kappa^L p^{L+1}}{\delta}\right)^{d_2K}.
    \end{align*}
\end{lemma}
According to the definition of $\cF_{\rm NN}^{G}$ , we have $\cF_{\rm NN}^{G}\subset \cF_{\rm NN}(D_1,D_1,L_4,p_4,K_4,\kappa_4,M_4)$ with
\begin{align}
    L_4=O(\log \varepsilon^{-1}), \ p_4=O(\varepsilon^{-d}), \ K_4=O(\varepsilon^{-d}\log \varepsilon^{-1}),\ \kappa_4=O(\varepsilon^{-7}), \ M_4=R_{\cX}.
    \label{eq:err:G:para}
\end{align}
Substituting (\ref{eq:err:G:para}) into Lemma \ref{lem:covering}, we get
\begin{align}
    &\log \cN\left(\frac{\delta}{4R_{\cX}},\cF_{\rm NN}^G,\|\cdot\|_{L^{\infty,\infty}}\right)\leq C_1 \varepsilon^{-d}\log^2 \varepsilon^{-1} (\log \varepsilon^{-1}+\log \delta^{-1}),
    \label{eq:err:G:logCover}
\end{align}
for some $C_1$ depending on $d,L_{\fb},L_{\gb},R_{\cX},|\Omega_{\cX}|,$ and is polynomial in $D_1$.

Substituting (\ref{eq:err:G:logCover}) into (\ref{eq:autoXGen:err:1}) gives rise to
\begin{align}
    &\EE_{\cJ_1}\EE_{u\sim \gamma} \|D_{\cX}^n\circ E_{\cX}^n\circ \cSX(u)-\cSX(u)\|^2_{\infty} \nonumber\\
    \leq& (16dL_{\gb}^2+ 4)\varepsilon^2+\frac{C_148R_{\cX}^2}{n}\varepsilon^{-d}\log^2 \frac{1}{\varepsilon} (\log \frac{1}{\varepsilon}+\log\frac{1}{\delta}+D_1)+6 \delta. 
    \label{eq:auXGene:error:2}
\end{align}
By balancing the terms in (\ref{eq:auXGene:error:2}), we set $\varepsilon=n^{-\frac{1}{2+d}}, \delta=\frac{1}{n}$. The bound in (\ref{eq:auXGene:error:2}) reduces to
\begin{align*}
    \EE_{\cJ_1}\EE_{u\sim \gamma} \|D_{\cX}^n\circ E_{\cX}^n\circ \cSX(u)-\cSX(u)\|^2_{\infty} \leq C_2 n^{-\frac{2}{2+d}}\log^3 n
\end{align*}
for some $C_2$ depending on $d,L_{\fb},L_{\gb},R_{\cX}$ and is polynomial in $D_1$. 

By Markov inequality, we further have the probability bound
\begin{align}
    \PP(\| D_{\cX}^n\circ E_{\cX}^n\circ \cSX(u)- \cSX(u)\|_{\infty}^2\geq t)\leq \frac{\EE_{\cJ_1}\EE_{u\sim \gamma} \|D_{\cX}^n\circ E_{\cX}^n\circ \cSX(u)-\cSX(u)\|^2_{\infty}}{t}  \leq \frac{C_2 n^{-\frac{2}{2+d}}\log^3 n}{t}
    \label{eq:autoXGene:markov}
\end{align}
for $t>0$ and $u\sim \gamma$.

\subsubsection{Proof of Theorem \ref{thm:generalization}}
\label{appgeneralizationstageii}
\begin{proof}[Proof of Theorem \ref{thm:generalization}]
    Recall that the dataset is evenly splitted into $\cJ_1$ and $\cJ_2$, which are used in Stage I and Stage II, respectively. We decompose the error as
    \begin{align}
    &\EE_{\cJ}\EE_{u\sim \gamma} \left[\|\Phi^n_{\rm NN}\circ \cSX(u) -\cSY\circ\Psi(u)\|^2_{\cSY} \right] \nonumber \\
    =&\EE_{\cJ}\EE_{u\sim \gamma} \left[\|\Phi^n_{\rm NN}(\widetilde{u}) -\widetilde{v}\|_{\cSY}^2\right] \nonumber\\
        =& \underbrace{\EE_{\cJ_1}\Bigg[ 2\EE_{\cJ_2} \left[ \frac{1}{n} \sum_{i=n+1}^{2n} \|\Phi^n_{\rm NN}(\widetilde{u})_i-\widetilde{v}_i\|_{\cSY}^2\bigg|\cJ_1\right]\Bigg]}_{\rm T_1}\nonumber \\
        &+\underbrace{\EE_{\cJ_1}\Bigg[\EE_{\cJ_2}\EE_{u\sim \gamma} \left[\|\Phi^n_{\rm NN}(\widetilde{u})-\widetilde{v}\|_{\cSY}^2 \Big|\cJ_1\right]-2\EE_{\cJ_2} \left[ \frac{1}{n} \sum_{i=n+1}^{2n} \|\Phi^n_{\rm NN}(\widetilde{u}_i)-\widetilde{v}_i\|_{\cS_{\cY}}^2\bigg|\cJ_1\right]\Bigg]}_{\rm T_2}.
        \label{eq:gene:err}
    \end{align}
    The term ${\rm T_1}$ captures the bias of network approximation. The term ${\rm T_2}$ captures the variance. We will derive the upper bound for each term in the rest of the proof.

    \paragraph{Bounding ${\rm T_1}$.} 
    We deduce 
    \begin{align}
        {\rm T_1}=& \EE_{\cJ_1} \left[2\EE_{\cJ_2}\left[\frac{1}{n}\sum_{i=n+1}^{2n}\|\Gamma_{\rm NN}^n\circ E_{\cX}^n\circ \cS_{\cX}(u_i) -\cSY\circ\Psi(u_i)\|^2_{\cSY} \Big| \cJ_1\right]\right] \nonumber\\
        =&\EE_{\cJ_1} \left[2\EE_{\cJ_2}\left[\frac{1}{n}\sum_{i=n+1}^{2n}\|\Gamma_{\rm NN}^n\circ E_{\cX}^n\circ \cSX(u_i) -\cSY\circ\Psi(u_i)- \cSY(\epsilon_i)+\cSY(\epsilon_i)\|^2_{\cSY} \Big| \cJ_1\right] \right] \nonumber\\
        =& \EE_{\cJ_1} \left[\frac{2}{n}\EE_{\cJ_2}\left[\sum_{i=n+1}^{2n}\|\Gamma_{\rm NN}^n\circ E_{\cX}^n\circ \cSX(u_i) -\cSY(\widehat{v}_i)+\cSY(\epsilon_i)\|^2_{\cSY} \Big| \cJ_1\right] \right] \nonumber\\
        =& \EE_{\cJ_1}\bigg[ \frac{2}{n}\EE_{\cJ_2}\bigg[\sum_{i=n+1}^{2n} \Big(\|\Gamma_{\rm NN}^n\circ E_{\cX}^n\circ \cSX(u_i) -\cSY(\widehat{v}_i)\|^2_{\cSY}+ \|\cSY(\epsilon_i)\|^2_{\cSY} \nonumber \\
        &+ 2\left\langle \Gamma_{\rm NN}^n\circ E_{\cX}^n \circ \cSX(u_i)-\cSY\circ\Psi(u_i)-\cSY(\epsilon_i), \cSY(\epsilon_i)\right\rangle_{\cSY}\Big) \Big| \cJ_1 \bigg] \bigg]\nonumber\\
        =& \EE_{\cJ_1}\bigg[  \frac{2}{n}\EE_{\cJ_2}\bigg[\sum_{i=n+1}^{2n} \Big(\|\Gamma_{\rm NN}^n\circ E_{\cX}^n\circ \cSX(u_i) -\cSY(\widehat{v}_i)\|^2_{\cSY}- \|\cSY(\epsilon_i)\|^2_{\cSY} \nonumber\\
        &+ 2\left\langle \Gamma_{\rm NN}^n\circ E_{\cX}^n\circ \cSX(u_i), \cSY(\epsilon_i)\right\rangle_{\cSY}\Big) \Big| \cJ_1\bigg] \bigg] \nonumber\\
        =& \EE_{\cJ_1}\bigg[  2\EE_{\cJ_2} \bigg[\inf_{\Gamma'_{\rm NN}\in \cF_{\rm NN}^{\Gamma}} \frac{1}{n}  \sum_{i=n+1}^{2n} \bigg(\|\Gamma'_{\rm NN}\circ E_{\cX}^n\circ \cSX(u_i) -\cSY(\widehat{v}_i)\|^2_{\cSY}- \|\cSY(\epsilon_i)\|^2_{\cSY}\bigg)\Big| \cJ_1\bigg] \nonumber\\
        &+\EE_{\cJ_2} \bigg[ \frac{4}{n} \sum_{i=n+1}^{2n}\left\langle \Gamma_{\rm NN}^n\circ E_{\cX}^n\circ \cSX(u_i), \cSY(\epsilon_i)\right\rangle_{\cSY} \Big| \cJ_1 \bigg]\bigg] \nonumber\\
        \leq & \EE_{\cJ_1} \bigg[ 2\inf_{\Gamma'_{\rm NN}\in \cF_{\rm NN}^{\Gamma}} \EE_{\cJ_2} \bigg[\frac{1}{n} \sum_{i=n+1}^{2n} \bigg(\|\Gamma'_{\rm NN}\circ E_{\cX}^n\circ\cSX(u_i) -\cSY(\widehat{v}_i)\|^2_{\cSY}- \|\cSY(\epsilon_i)\|^2_{\cSY}\bigg)\Big| \cJ_1\bigg] \nonumber\\
        &+\EE_{\cJ_2} \bigg[ \frac{4}{n} \sum_{i=n+1}^{2n}\left\langle \Gamma_{\rm NN}^n\circ E_{\cX}^n\circ \cSX(u_i), \cSY(\epsilon_i)\right\rangle_{\cSY} \Big| \cJ_1 \bigg]\bigg] \nonumber\\
        =& \EE_{\cJ_1}\bigg[  2\inf_{\Gamma'_{\rm NN}\in \cF_{\rm NN}^{\Gamma}} \EE_{\cJ_2} \bigg[\frac{1}{n} \sum_{i=n+1}^{2n} \bigg(\|\Gamma'_{\rm NN}\circ E_{\cX}^n \circ \cSX(u_i) -\cSY(v_i)-\cSY(\epsilon_i)\|^2_{\cSY}- \|\cSY(\epsilon_i)\|^2_{\cSY}\bigg) \Big| \cJ_1\bigg]\nonumber\\
        &+\EE_{\cJ_2} \bigg[ \frac{4}{n} \sum_{i=n+1}^{2n}\left\langle \Gamma_{\rm NN}^n\circ E_{\cX}^n\circ \cSX(u_i), \cSY(\epsilon_i)\right\rangle_{\cSY} \Big| \cJ_1 \bigg] \bigg]\nonumber\\
        =&\EE_{\cJ_1}\bigg[  2\inf_{\Gamma'_{\rm NN}\in \cF_{\rm NN}^{\Gamma}} \EE_{\cJ_2} \bigg[\frac{1}{n} \sum_{i=n+1}^{2n} \bigg(\|\Gamma'_{\rm NN}\circ E_{\cX}^n \circ\cSX(u_i) -\cSY\circ\Psi(u_i)\|_{\cSY}^2 \nonumber \\
        &-2\left\langle \Gamma_{\rm NN}\circ E_{\cX}^n \circ \cSX(u_i) -\cSY(v_i), \cSY(\epsilon_i) \right\rangle_{\cSY}\bigg) \Big| \cJ_1 \bigg]
        +\EE_{\cJ_2} \bigg[ \frac{4}{n} \sum_{i=n+1}^{2n}\left\langle \Gamma_{\rm NN}^n\circ E_{\cX}^n\circ \cSX(u_i), \cSY(\epsilon_i)\right\rangle_{\cSY} \Big| \cJ_1 \bigg]\bigg] \nonumber\\
        =&\underbrace{\EE_{\cJ_1}\bigg[ 2\inf_{\Gamma'_{\rm NN}\in \cF_{\rm NN}^{\Gamma}} \EE_{\rm u\sim \gamma} \bigg[ \|\Gamma'_{\rm NN}\circ E_{\cX}^n \circ \cSX(u) -\cSY\circ\Psi(u)\|_{\cSY}^2 \Big| \cJ_1 \bigg]\bigg]}_{\rm T_{1a}} \nonumber\\
        & +\underbrace{\EE_{\cJ} \bigg[ \frac{4}{n} \sum_{i=n+1}^{2n}\left\langle \Phi^n_{\rm NN}\circ \cSX(u_i), \cSY(\epsilon_i)\right\rangle_{\cSY}  \bigg]}_{\rm T_{1b}}.
        \label{eq:T1:err:1}
    \end{align}
    Term ${\rm T_{1a}}$ captures the approximation error and term ${\rm T_{1b}}$ captures the stochastic error. We will analyze ${\rm T_{1a}}$ and ${\rm T_{1b}}$ separately.

\paragraph{Bounding ${\rm T_{1a}}$.}
We first focus on ${\rm T_{1a}}$ in (\ref{eq:T1:err:1}) and derive an upper bound using approximation results of $\widetilde{\fb},\widetilde{\gb}$ and $\Gamma$. We define the $\xi$-neighbourhood of a set as follows:
\begin{definition}
        For a set $\Xi\in [-M,M]^D$ for some $M>0$, the $\xi$-neighbourhood containing $\Xi$ is defined as the set
        \begin{align}
        T_{\xi}(\Xi)=\{\zb|\inf_{\zb_1\in\Xi}\|\zb-\zb_1\|_{\infty}\leq \xi\}.
        \end{align}
    \end{definition}

We next show that for any  $\xi\geq0$, the function $\widetilde{\fb}$ defined in Lemma \ref{lemma:f1g1lip} extended to $\RR^{D_1}$ according to Lemma \ref{kem:extension} can be well approximated by a network on $T_{\xi}(\widetilde{\cM})$.

 The following lemma shows that we can approximate $\CL(\widetilde{\fb})$ well on a $\xi$-neighbourhood of $\widetilde{\cM}$, where $\CL(\cdot)$ is the clipping operator defined in (\ref{eq:clip}). 

    \begin{lemma}\label{lem:Ex:1}
        For any $\varepsilon\in(0,1/4)$ and $\xi\geq 0$, 
        set network architectures $ 
    \cF_{\rm NN}(D_1,d,L_{1},p_{1},K_{1},\kappa_{1},M_{1})$ with
    \begin{align*}
    &L_{1}=O(\log \varepsilon^{-1}),\ p_{1}=O(\varepsilon_1^{-d}),\
    K_{1}=O(\varepsilon_1^{-d}), \ \kappa_{1}=O(\varepsilon_1^{-7}), \ M_{1}=1.
    \end{align*}
    For any Lipschitz function $\widetilde{\fb}:\widetilde{\cM}\rightarrow [-1,1]^d$ extended to domain $[-R_{\cX},R_{\cX}]^D$ according to Lemma \ref{kem:extension}, with Lipscthiz constant bounded by $L_{\fb}$, there exists $\widetilde{\fb}_{\rm NN}\in \cF_{\rm NN}(D_1,d,L_{1},p_{1},K_{1},\kappa_{1},M_{1})$ such that
    \begin{align}
        &\sup_{\widetilde{u}\in T_{\xi}(\widetilde{\cM})}\|\widetilde{\fb}_{\rm NN}(\widetilde{u})- \CL\circ\widetilde{\fb}(\widetilde{u})\|_{\infty}\leq CD(\varepsilon+\xi),
    \end{align}
    for some absolute constant $C$. The constant hidden in $O(\cdot)$ depends on $d, L_{\fb}, \xi,R_{\cX}$ and is polynomial in $D_1$. 
    \end{lemma}
    Lemma \ref{lem:Ex:1} is proved in Appendix \ref{proof:Ex:1}.

    Let $\varepsilon_1, \xi\in(0,1/4)$.
    By Lemma \ref{lem:Ex:1}, set the network architecture $ 
    \bar{\cF}_{\rm NN}^{E_{\cX}}=\cF_{\rm NN}(D_1,d,L_{5},p_{5},K_{5},\kappa_{5},M_{5})$ with
    \begin{align*}
    &L_{5}=O(\log \varepsilon_1^{-1}),\ p_{5}=O(\varepsilon_1^{-d}),\
    K_{5}=O(\varepsilon_1^{-d}), \ \kappa_{5}=O(\varepsilon_1^{-7}), \ M_{5}=1.
    \end{align*}
There exists $\bar{\fb}_{\rm NN}\in \bar{\cF}_{\rm NN}^{E_{\cX}}$ satisfying
    \begin{align}
        \sup_{\widetilde{u}\in T_{\xi}(\widetilde{\cM})}\|\bar{\fb}_{\rm NN}(\widetilde{u})- \CL\circ\widetilde{\fb}(\widetilde{u})\|_{\infty}\leq C_3(\varepsilon_1+\xi)
        \label{eq:fgp:err}
    \end{align} 
    for some $C_3$ linear in $D_1$, where $\CL(\cdot)$ is the clipping operator defined in (\ref{eq:clip}).

    By Theorem \ref{thm:approximation}, set the network architecture $\bar{\cF}_{\rm NN}^{\Gamma}=\cF_{\rm NN}(d,D_2,L_{6},p_{6},K_{6},\kappa_{6},M_{6})$ with
    \begin{align*}
        &L_{6}=O(\log \varepsilon_1^{-1}), \ p_{6}=O(\varepsilon_1^{-d}), \ K_{6}=O(\varepsilon_1^{-d}\log \varepsilon_1^{-1}), \ \kappa_{6}= O(\varepsilon_1^{-1}),\  M_{6}=R_{\cY}.
    \end{align*}
    There exists $\bar{\Gamma}_{\rm NN}\in \bar{\cF}_{\rm NN}^{\Gamma}$
    satisfying 
    \begin{align*}
        \sup_{\zb\in[-1,1]^d}\|\bar{\Gamma}_{\rm NN}(\zb)- \Gamma(\zb)\|_{\infty}\leq \varepsilon_1.
    \end{align*} 
    The constant hidden in $O(\cdot)$ depends on $d,L_{\Psi},L_{\gb}$ and is linear in $D_2$.

    Define the network
    \begin{align*}
        \cF_{\rm NN}^{\Gamma}=\big\{\Gamma'_{\rm NN}:[-1,1]^d\rightarrow[-R_{\cY},R_{\cY}]^{D_2}|&\ \Gamma'_{\rm NN}=\bar{\Gamma}'_{\rm NN}\circ\bar{\fb}'_{\rm NN}\circ\widetilde{\gb}'_{\rm NN} \ \\
        &\mbox{ for } \ \bar{\Gamma}'_{\rm NN}\in \bar{\cF}_{\rm NN}^{\Gamma}, \ \bar{\fb}'_{\rm NN}\in \bar{\cF}_{\rm NN}^{E_{\cX}}, \ \widetilde{\gb}'_{\rm NN}\in \cF_{\rm NN}^{D_{\cX}}\big\}
    \end{align*}
    where $\cF_{\rm NN}^{D_{\cX}}$ is given in (\ref{eq.thm2.FD}).
    
    We have $\cF_{\rm NN}^{\Gamma}\in \cF_{\rm NN}(d,D_2,L_7,p_7,K_7,\kappa_7,M_7)$ with
    \begin{align*}
         &L_7=O\left(\log (\varepsilon^{-1}+\varepsilon_1^{-1})\right), \ p_7=O\left(\varepsilon^{-d}+\varepsilon_1^{-d}\right), \ K_7=O\left((\varepsilon^{-d}+\varepsilon_1^{-d})\log (\varepsilon^{-1}+\varepsilon_1^{-1}) \right),\\
         &\kappa_7=O\left(\varepsilon^{-1}+ \varepsilon_1^{-7}\right), \ M_7=R_{\cY}
    \end{align*}
    and $\bar{\Gamma}_{\rm NN}\circ \bar{\fb}_{\rm NN}\circ D_{\cX}^n\in \cF_{\rm NN}^{\Gamma}$. We thus have  
    \begin{align}
        {\rm T_{1a}}=&\EE_{\cJ_1}\bigg[ 2\inf_{\Gamma'_{\rm NN}\in \cF_{\rm NN}^{\Gamma}} \EE_{\rm u\sim \gamma} \bigg[ \|\Gamma'_{\rm NN}\circ E_{\cX}^n \circ \cSX(u) -\cSY\circ\Psi(u)\|_{\cSY}^2 \Big| \cJ_1 \bigg] \nonumber\\
        \leq & \EE_{\cJ_1}\left[2\EE_{u\sim \gamma} \left[ \|\bar{\Gamma}_{\rm NN}\circ \bar{\fb}_{\rm NN}\circ D_{\cX}^n\circ E_{\cX}^n\circ\cSX(u)-\cSY\circ\Psi(u)\|_{\cS_{\cY}}^2\Big|\cJ_1 \right]\right] \nonumber\\
        = & \EE_{\cJ}\left[2\EE_{u\sim \gamma} \left[ \|\bar{\Gamma}_{\rm NN}\circ \bar{\fb}_{\rm NN}\circ D_{\cX}^n\circ E_{\cX}^n\circ \cSX(u)-\Gamma\circ \widetilde{\fb}\circ \cSX(u)\|_{\cS_{\cY}}^2 \right] \right] \nonumber\\
        \leq& \EE_{\cJ}\Big[2\EE_{u\sim \gamma} \Big[ 3\|\bar{\Gamma}_{\rm NN}\circ \bar{\fb}_{\rm NN}\circ D_{\cX}^n\circ E_{\cX}^n\circ \cSX(u)-\Gamma\circ \bar{\fb}_{\rm NN}\circ D_{\cX}^n\circ E_{\cX}^n\circ \cSX(u)\|_{\cS_{\cY}}^2 \nonumber\\
        &+ 3\|\Gamma\circ \bar{\fb}_{\rm NN}\circ D_{\cX}^n\circ E_{\cX}^n\circ \cSX(u)-\Gamma\circ (\CL\circ\widetilde{\fb})\circ D_{\cX}^n\circ E_{\cX}^n\circ \cSX(u)\|_{\cS_{\cY}}^2 \nonumber\\
        & + 3\|\Gamma\circ (\CL\circ\widetilde{\fb})\circ D_{\cX}^n\circ E_{\cX}^n\circ \cSX(u)-\Gamma\circ \widetilde{\fb}\circ \cSX(u)\|_{\cS_{\cY}}^2\Big]\Big] \nonumber\\
        \leq & \EE_{\cJ}\Big[2\EE_{u\sim \gamma} \Big[ 3|\Omega_{\cY}|\varepsilon_1^2 \nonumber\\
        &+ 12L_{\Psi}^2L_{\gb}^2\|\bar{\fb}_{\rm NN}\circ D_{\cX}^n\circ E_{\cX}^n\circ \cSX(u)-(\CL\circ\widetilde{\fb})\circ D_{\cX}^n\circ E_{\cX}^n\circ \cSX(u)\|_{2}^2 \nonumber\\
        & + 12L_{\Psi}^2L_{\gb}^2\| (\CL\circ\widetilde{\fb})\circ D_{\cX}^n\circ E_{\cX}^n\circ \cSX(u)-\widetilde{\fb}\circ \cSX(u)\|_{2}^2\Big]\Big], 
        \label{eq:T1:term1:1}
    \end{align}
    where we used Lemma \ref{lem:GammaLip} in the last inequality.

To derive an upper bound of (\ref{eq:T1:term1:1}), denote 
\begin{align*}
    &{\rm I}=12L_{\Psi}^2L_{\gb}^2\|\bar{\fb}_{\rm NN}\circ D_{\cX}^n\circ E_{\cX}^n\circ \cSX(u)-(\CL\circ\widetilde{\fb})\circ D_{\cX}^n\circ E_{\cX}^n\circ \cSX(u)\|_{2}^2,\\
    &{\rm II}=12L_{\Psi}^2L_{\gb}^2\| (\CL\circ\widetilde{\fb})\circ D_{\cX}^n\circ E_{\cX}^n\circ \cSX(u)-\widetilde{\fb}\circ \cSX(u)\|_{2}^2.
\end{align*}
We have
\begin{align*}
    {\rm I}\leq &12L_{\Psi}^2L_{\gb}^2(64dR_{\cX}^2L_{\fb}^2)= 768dL_{\Psi}^2L_{\gb}^2L_{\fb}^2R_{\cX}^2,\\
    {\rm II}\leq & 48L_{\Psi}^2L_{\gb}^2d.
\end{align*}
When $\| D_{\cX}^n\circ E_{\cX}^n\circ \cSX(u)- \cSX(u)\|_{\infty}^2\leq \xi^2$, we have $D_{\cX}^n\circ E_{\cX}^n\circ \cSX(u)\in T_{\xi}(\widetilde{\cM})$, 
and by (\ref{eq:fgp:err}), 
\begin{align*}
    {\rm I}&\leq 24C_3^2L_{\Psi}^2L_{\gb}^2(\varepsilon_1^2+\xi^2),\\
    {\rm II}&\leq12L_{\Psi}^2L_{\gb}^2\| (\CL\circ\widetilde{\fb})\circ D_{\cX}^n\circ E_{\cX}^n\circ \cSX(u)-\CL\circ\widetilde{\fb}\circ \cSX(u)\|_{2}^2\\
    &\leq96L_{\Psi}^2L_{\gb}^2L_{\fb}^2
    \| D_{\cX}^n\circ E_{\cX}^n\circ \cSX(u)-\cSX(u)\|_{\cSX}^2  \leq 96L_{\Psi}^2L_{\gb}^2L_{\fb}^2\xi^2.
\end{align*}

We thus have
\begin{align}
    \EE_{\cJ_1}\EE_{u\sim \gamma} \left[ {\rm I}+ {\rm II} \right] \leq &\PP(\| D_{\cX}^n\circ E_{\cX}^n\circ \cSX(u)- \cSX(u)\|_{\infty}^2\leq \xi^2) (24C_3^2L_{\Psi}^2L_{\gb}^2+ 96L_{\Psi}^2L_{\gb}^2L_{\fb}^2)(\varepsilon_1^2+\xi^2) \nonumber\\
    &+ \PP(\| D_{\cX}^n\circ E_{\cX}^n\circ \cSX(u)- \cSX(u)\|_{\infty}^2\geq \xi^2)(768|\Omega_{\cX}|L_{\Psi}^2L_{\gb}^2L_{\fb}^2R_{\cX}^2+ 48L_{\Psi}^2L_{\gb}^2) \nonumber\\
    \leq& (24C_3^2L_{\Psi}^2L_{\gb}^2+ 96L_{\Psi}^2L_{\gb}^2L_{\fb}^2)(\varepsilon_1^2+\xi^2) \nonumber\\
    &+(768|\Omega_{\cX}|L_{\Psi}^2L_{\gb}^2L_{\fb}^2R_{\cX}^2+ 48L_{\Psi}^2L_{\gb}^2)\frac{\EE\left[ \| D_{\cX}^n\circ E_{\cX}^n\circ \cSX(u)- \cSX(u)\|_{\infty}^2 \right]}{\xi^2}  \nonumber\\
    \leq &(24C_3^2L_{\Psi}^2L_{\gb}^2+ 96L_{\Psi}^2L_{\gb}^2L_{\fb}^2)(\varepsilon_1^2+\xi^2) \nonumber\\
    &+(768|\Omega_{\cX}|L_{\Psi}^2L_{\gb}^2L_{\fb}^2R_{\cX}^2+ 48L_{\Psi}^2L_{\gb}^2)\frac{C_2n^{-\frac{2}{2+d}}\log^3 n}{\xi^2},
    \label{eq:IandII:0}
\end{align}
where we used (\ref{eq:autoXGene:markov}) in the last two inequalities.

Set $\varepsilon_1^2=\xi^2=n^{-\frac{1}{2+d}}$ with $n>4^{2+d}$, (\ref{eq:IandII:0}) becomes
\begin{align}
    \EE_{\cJ_1}\EE_{u\sim \gamma} \left[ {\rm I}+ {\rm II} \right] \leq C_4 n^{-\frac{1}{2+d}} \log^3n,
    \label{eq:IandII}
\end{align}
for some $C_4$ depending on $d,L_{\fb},L_{\gb},L_{\Psi},R_{\cX},R_{\cY},|\Omega_{\cX}|$ and is polynomial in $D_1$ and is linear in $D_2$.

Substituting (\ref{eq:IandII}) and $\varepsilon_1^2=n^{-\frac{1}{2+d}}$ into (\ref{eq:T1:term1:1}) gives rise to
\begin{align}
    {\rm T_{1a}}
        \leq  C_5 n^{-\frac{1}{2+d}}\log^3n,
        \label{eq:T1a:err}
\end{align}
for some $C_5$ depending on $d,L_{\fb},L_{\gb},L_{\Psi},R_{\cX},R_{\cY},|\Omega_{\cX}|,|\Omega_{\cY}|$ and is polynomial in $D_1$ and is linear in $D_2$.

The resulting network architecture is $\cF_{\rm NN}^{\Gamma}\in \cF_{\rm NN}(d,D_2,L_7,p_7,K_7,\kappa_7,M_7)$ with
    \begin{align*}
         &L_7=O\left(\log (\varepsilon^{-1})\right), \ p_7=O\left(\varepsilon^{-d}\right), \ K_7=O\left(\varepsilon^{-d}\log \varepsilon^{-1} \right),\ \kappa_7=O\left( \varepsilon^{-7/2}\right), \ M_7=R_{\cY}
    \end{align*}
    for $\varepsilon=n^{-\frac{2}{2+d}}$.

\paragraph{Bounding ${\rm T_{1b}}$.}

 Define the network architecture
 \begin{align*}
     \cF_{\rm NN}^{\Phi}=\{\Phi'_{\rm NN}| \Phi'_{\rm NN}=\Gamma'_{\rm NN}\circ E'_{\rm NN} \ \mbox{ for } \ \Gamma'_{\rm NN}\in \cF_{\rm NN}^{\Gamma},\ E'_{\rm NN}\in \cF_{\rm NN}^{E_{\cX}} \}
 \end{align*}
 and denote 
 $$
 \|\Phi'_{\rm NN}\|_n^2=\frac{1}{n}\sum_{i=n+1}^{2n} \|\Phi'_{\rm NN}\circ\cSX(u_i)\|_{\cSY}^2. 
 $$ 
An upper bound of ${\rm T_{1b}}$ is given by the following lemma (see a proof in Appendix \ref{sec:T1:var}):
\begin{lemma}\label{lemma:T1:var}
    Under conditions of Theorem \ref{thm:generalization}, for any $\delta>0$, we have 
    \begin{align}
    &\EE_{\cJ} \left[\frac{1}{n}\sum_{i=n+1}^{2n}\left\langle \Phi^n_{\rm NN}\circ \cSX(u_i), \cSY(\epsilon_i)\right\rangle_{\cSY}\right] \nonumber\\
    \leq& 2 \sqrt{|\Omega_{\cY}|}\sigma\left(\sqrt{ \EE_{\cJ}\left[ \| \Phi^n_{\rm NN}-\Phi\|_{n}^2\right]}+ \sqrt{D_2}\delta \right)\sqrt{\frac{4 \log \cN(\delta,\cF_{\rm NN}^{\Phi},\|\cdot\|_{L^{\infty,\infty}}) + 6}{n}} + |\Omega_{\cY}|\delta\sigma.
    \label{eq:T1:var:4}
\end{align}
\end{lemma}

Substituting (\ref{eq:T1a:err}) and (\ref{eq:T1:var:4}) into (\ref{eq:T1:err:1}) gives rise to 
\begin{align}
    {\rm T_1}=&2\EE_{\cJ}\left[\frac{1}{n}\sum_{i=n+1}^{2n}\|\Phi^n_{\rm NN}\circ \cSX(u_i) -\cSY\circ\Psi(u_i)\|^2_{\cSY}\right] 
    =2\EE_{\cJ}\left[ \| \Phi^n_{\rm NN}-\Phi\|_{n}^2\right] \nonumber\\
    \leq &C_5 n^{-\frac{1}{2+d}}\log^3n + 4|\Omega_{\cY}|\delta\sigma+ 8\sqrt{|\Omega_{\cY}|}\sigma \left(\sqrt{ \EE_{\cJ}\left[ \| \Phi^n_{\rm NN}-\Phi\|_{n}^2\right]}+ \sqrt{|\Omega_{\cY}|}\delta \right)\sqrt{\frac{4 \log \cN(\delta,\cF_{\rm NN}^{\Phi},\|\cdot\|_{L^{\infty,\infty}}) + 6}{n}} .
    \label{eq:T1:bound:1}
\end{align}

In (\ref{eq:T1:bound:1}), the term $\EE_{\cJ}\left[ \| \Phi^n_{\rm NN}-\Phi\|_{n}^2\right]$ appears on both sides. We next derive an upper bound of ${\rm T_1}$ by applying some inequality to (\ref{eq:T1:bound:1}). 

Denote 
\begin{align*}
    &\omega=\sqrt{\EE_{\cJ}\left[ \| \Phi^n_{\rm NN}-\Phi\|_{n}^2\right]},\\
    &a=\frac{C_5 n^{-\frac{1}{2+d}}\log^3n}{2}+2|\Omega_{\cY}|\delta\sigma +4|\Omega_{\cY}|\sigma\delta\sqrt{\frac{4 \log \cN(\delta,\cF_{\rm NN}^{\Phi},\|\cdot\|_{L^{\infty,\infty}}) + 6}{n}},\\
    &b=2\sqrt{|\Omega_{\cY}|}\sigma\sqrt{\frac{4 \log \cN(\delta,\cF_{\rm NN}^{\Phi},\|\cdot\|_{L^{\infty,\infty}}) + 6}{n}}.
\end{align*}
Relation (\ref{eq:T1:bound:1}) implies
\begin{align*}
    \omega^2\leq a+2b\omega.
\end{align*}
We deduce 
\begin{align*}
    (\omega-b)^2\leq a+b^2 \Rightarrow |\omega-b|\leq \sqrt{a+b^2}\leq \sqrt{a}+b.
\end{align*}
When $\omega\geq b$, we have
\begin{align*}
    w-b\leq \sqrt{a}+b \Rightarrow w\leq \sqrt{a}+2b \Rightarrow w^2\leq (\sqrt{a}+2b)^2\leq 2a+8b^2.
\end{align*}
When $\omega<b$, we also have $\omega^2\leq 2a+8b^2$. Substituting the expression of $\omega,a,b$ into the relation $\omega^2\leq 2a+8b^2$, we have
\begin{align}
    {\rm T_1}=2\omega^2 \leq &2C_5 n^{-\frac{1}{2+d}}\log^3n+8|\Omega_{\cY}|\delta\sigma +16|\Omega_{\cY}|\sigma\delta\sqrt{\frac{4 \log \cN(\delta,\cF_{\rm NN}^{\Phi},\|\cdot\|_{L^{\infty,\infty}}) + 6}{n}} \nonumber\\
    &+ 128|\Omega_{\cY}|\sigma^2\frac{2 \log \cN(\delta,\cF_{\rm NN}^{\Phi},\|\cdot\|_{L^{\infty,\infty}}) + 3}{n}.
    \label{eq:T1}
\end{align}

\paragraph{Bounding ${\rm T_2}$.} The upper bound of ${\rm T_2}$ can be derived using the covering number $\cN(\delta,\cF_{\rm NN}^{\Phi},\|\cdot\|_{L^{\infty,\infty}})$ and Bernstein-type inequalities. The upper bound is summarized in the following lemma (see a proof in Appendix \ref{proof:T2}).
\begin{lemma}\label{lem:T2}
    Under conditions of Theorem \ref{thm:generalization}, for any $\delta>0$, we have
    \begin{align}
        {\rm T_2}\leq \frac{48R_{\cY}^2|\Omega_{\cY}|}{n}\log\cN\left(\frac{\delta}{4R_{\cY}|\Omega_{\cY}|},\cF_{\rm NN}^{\Phi},\|\cdot\|_{L^{\infty,\infty}}\right)+ 6\delta.
        \label{eq:T2}
    \end{align}
\end{lemma}
\paragraph{Putting ${\rm T_1}$ and ${\rm T_2}$ together.} Substituting (\ref{eq:T1}) and (\ref{eq:T2}) into (\ref{eq:gene:err}), we have
\begin{align}
    &\EE_{\cJ}\EE_{u\sim \gamma} \|\Phi^n_{\rm NN}\circ \cSX(u) -\cSY\circ\Psi(u)\|^2_{\cSY} \nonumber\\
    \leq & 2C_5 n^{-\frac{1}{2+d}}\log^3n+(8|\Omega_{\cY}|+6)\delta\sigma +16|\Omega_{\cY}|\sigma\delta\sqrt{\frac{4 \log \cN(\delta,\cF_{\rm NN}^{\Phi},\|\cdot\|_{L^{\infty,\infty}}) + 6}{n}} +6\delta\nonumber\\
    &+ 64|\Omega_{\cY}|\sigma^2\frac{2 \log \cN(\delta,\cF_{\rm NN}^{\Phi},\|\cdot\|_{L^{\infty,\infty}}) + 3}{n} +\frac{48R_{\cY}^2|\Omega_{\cY}|}{n}\log\cN\left(\frac{\delta}{4R_{\cY}|\Omega_{\cY}|},\cF_{\rm NN}^{\Phi},\|\cdot\|_{L^{\infty,\infty}}\right).
    \label{eq:gene:error:1}
\end{align}
The network architecture satisfies $\cF_{\rm NN}^{\Phi}\subset \cF_{\rm NN}(D_1,D_2,L_8,p_8,K_8,\kappa_8,M_8)$ with
\begin{align}
    L_8=O(\log \varepsilon^{-1}), \ p_8=O(\varepsilon^{-d}), \ K_8=O(\varepsilon^{-d}\log \varepsilon^{-1}),\ \kappa_8=O(\varepsilon^{-7}), \ M_8=R_{\cY}.
    \label{eq:err:Phi:para}
\end{align}
Substituting (\ref{eq:err:Phi:para}) with $\varepsilon=n^{-\frac{1}{2+d}}, \delta=n^{-1}$ into Lemma \ref{lem:covering}, we get
\begin{align}
    &\log \cN\left(\delta,\cF_{\rm NN}^{\Phi},\|\cdot\|_{L^{\infty,\infty}}\right)\leq C_6 \varepsilon^{-d}\log^2 \varepsilon^{-1} (\log \varepsilon^{-1}+\log \delta^{-1})\leq2C_6n^{\frac{d}{2+d}}\log^3 n,
    \label{eq:err:Phi:logCover}
\end{align}
for some $C_6$ depending on $d,L_{\fb},L_{\gb},L_{\Psi},R_{\cX},R_{\cY},|\Omega_{\cX}|,|\Omega_{\cY}|$ and is polynomial in $D_1$ and is linear in $D_2$.

Substituting (\ref{eq:err:Phi:logCover}) into (\ref{eq:gene:error:1}) gives rise to
\begin{align*}
    \EE_{\cJ}\EE_{u\sim \gamma} \|\Phi^n_{\rm NN}\circ \cSX(u) -\cSY\circ\Psi(u)\|^2_{\cSY} \leq C_7(1+\sigma^2) n^{-\frac{1}{2+d}}\log^3 n,
\end{align*}
for some $C_7$ depending on $d,L_{\fb},L_{\gb},L_{\Psi},R_{\cX},R_{\cY},|\Omega_{\cX}|,|\Omega_{\cY}|$ and is polynomial in $D_1$ and is linear in $D_2$.

The resulting network architectures are stated in (\ref{eq:gene:para:1}), (\ref{eq:gene:para:2}) and (\ref{eq:gene:para:3}).

\end{proof}

\section{Conclusion and Discussion}
\label{secdiscussion}
This paper explores the use of Auto-Encoder-based neural network (AENet) for operator learning in function spaces, leveraging Auto-Encoders-based nonlinear model reduction techniques. This approach is particularly effective when input functions are situated on a nonlinear manifold. In such cases, an Auto-Encoder is utilized to identify and represent input functions as latent variables. These latent variables are then transformed during the operator learning process to produce outputs. Our study establishes a comprehensive approximation theory and performs an in-depth analysis of generalization errors. The findings indicate that the efficiency of AENet, measured in terms of sample complexity, is closely linked to the intrinsic dimensionality of the underlying model.

We next discuss some potential applications and improvement of this work.

\noindent
{\bf Network architecture:} In this paper, we have an Auto-Encoder applied on the input functions, instead of two Auto-Encoders applied on the input and output functions, respectively, since in our numerical experiments, training the Auto-Encoder on the output is almost as hard as training the transformation from the input latent variable to the output. In literature, Auto-Encoders are applied on the output in \cite{NEURIPS2022_24f49b2a} and \cite{kontolati2023learning}. For the simulation of high-dimensional PDEs, it may be important to apply an Auto-Encoder on the output to reduce its dimension.
Our proof technique may be extended to Auto-Encoder-based neural networks (AENets) where two Auto-Encoders are applied on the input and output functions, respectively. We will investigate this in our future work.

\noindent
{\bf Generating the solution manifold:} AENet has the advantage of producing the solution manifold of the operator $\Psi$ from low-dimensional latent variables. With the transformation $\Gamma_{\rm NN}^n$ given in \eqref{eq:gamma:loss}, we can express the solution manifold as $\{\Gamma_{\rm NN}^n(\zb): \ \zb \in [-1,1]^d\}$. In other words, AENet not only learns the operator $\Psi$, but also gives rise to the solution manifold. AENet is a potential tool to study the geometric structure of the solution manifold.
\\
\noindent
{\bf Data splitting in Stage I and II:} Our algorithm involves a data splitting in Stage I and II, in order to create data independence in Stage I and II for the proof of the generalization error. This data splitting strategy is only for theory purpose. In experiments, we use all training data in Stage I and II.
\\
\noindent
{\bf Optimality of convergence rate:} This paper provides the first generalization error analysis of nonlinear model reduction by deep neural networks. Theorem \ref{thm:generalization} proves an exponential convergence of the squared generalization error as $n$ increases, and the exponent depends on the intrinsic dimension of the model. The rate of convergence (exponent) in Theorem \ref{thm:generalization} may not be optimal. One of our future works is to improve the rate of convergence.

\bibliographystyle{ims}
\bibliography{refoperator}

\newpage
\appendix
\section*{Appendix}

\section{Example \ref{examplesampling} and the error bound in Remark \ref{remark.approx.interp} and Remark \ref{remark.gene.interp}}
\subsection{Proof of Example \ref{examplesampling}}
\label{applemma:examplesampling}

\begin{proof}[Proof of Example \ref{examplesampling}]
For any $u = \sum_{k=-N}^N a_k e^{2\pi i k x}$, we have
\begin{align*}
    \|u\|_{L^2([0,1])} &= \sqrt{\sum_{k=-N}^N |a_k|^2} \ge a\sqrt{2N+1},
    \\
     \sup_{x \in [0,1]}\left|\frac{du}{dx}\right| &= \sup_{x \in [0,1]}\left|\sum_{k=-N}^N a_k 2\pi i k e^{2\pi i k x}\right| \le 2\pi A\sum_{k=-N}^N |k| = 2\pi A N(N+1).
\end{align*}
Therefore,
$\delta \ge \frac{a\sqrt{2N+1}}{2\pi A N(N+1)}$.

\end{proof}

\subsection{Derivation of (\ref{eq.approx.inter.x})}
\label{sec:approx.interpolation.x}
We have
\begin{align*}
    &\sup_{u \in \cM} \|\Phi_{\rm NN}\circ \cSX(P_{{\rm intp},\cX}(\cSX'(u)))-\Phi(\widetilde{u})\|_{\infty}\\
    \leq& \sup_{u \in \cM} \left[\|\Phi_{\rm NN}\circ \cSX(u)-\cS_{\cY}\circ\Psi(u)\|_{\infty}+ \|\Phi_{\rm NN}\circ \cSX(P_{{\rm intp},\cX}(\cSX'(u)))-\Phi_{\rm NN}\circ \cSX(u)\|_{\infty}\right]\\
    \leq & \varepsilon+ \sup_{u \in \cM} \|\Phi_{\rm NN}\circ \cSX(P_{{\rm intp},\cX}(\cSX'(u)))-\Phi_{\rm NN}\circ \cSX(u)\|_{\infty},
\end{align*}
where the last inequality is based on Corollary \ref{coro:approximation}.

\subsection{Error bound in (\ref{eq.gene.inter.x})}
\label{sec:gene.interpolation.x}
We have
\begin{align*}
   & \EE_{\cJ}\EE_{u\sim \gamma} \|\Phi^n_{\rm NN}\circ \cSX(P_{{\rm intp},\cX}(\cS_{\cX}'(u))) -\cSY\circ\Psi(u)\|^2_{\cSY}
    \\
    \leq& \EE_{\cJ}\EE_{u\sim \gamma} [  \|\Phi^n_{\rm NN}\circ \cSX(u)-\cSY\circ\Psi(u)\|^2_{\cSY} + \|\Phi^n_{\rm NN}\circ \cSX(P_{{\rm intp},\cX}(\cS_{\cX}'(u))) -\Phi^n_{\rm NN}\circ \cSX(u)\|^2_{\cSY}]\\
   \leq & C (1+\sigma^2)n^{-\frac{1}{2+d}}\log^3 n+   \EE_{\cJ}\EE_{u\sim \gamma} \|\Phi^n_{\rm NN}\circ \cSX(P_{{\rm intp},\cX}(\cS_{\cX}'(u))) - \Phi^n_{\rm NN}\circ \cSX(u) \|^2_{\cSY},
\end{align*}
where the last inequality is based on Theorem \ref{thm:generalization}.

\section{Proofs of Lemmas}
\label{applemma}

\subsection{Proof of Lemma \ref{lemma:f1g1lip}}
\label{applemma:f1g1lipproof}

\begin{proof}[Proof of Lemma \ref{lemma:f1g1lip}]
We have
\begin{align*}
\widetilde{\gb} \circ \widetilde{\fb} (\cSX(u))=\cSX\circ \gb\circ\fb(u)=\cSX(u)
\end{align*}
which proves \eqref{lemma:f1g1lipeq1}. The Lipschitz property of $\widetilde{\fb}$  follows from, for any $u_1,u_2 \in \cM$,
\begin{align*}
\|\widetilde{\fb}(\cSX(u_1))-\widetilde{\fb}(\cSX(u_2))\|_2 = \|\fb(u_1)-\fb(u_2)\|_2
 \le  L_{\fb}\|u_1-u_2\|_{\cX} \le  2L_{\fb}\|\cSX(u_1)-\cSX(u_2)\|_{\cSX}.
\end{align*}
We next show that $\widetilde{\gb}$ is Lipschitz. For any $\zb_1,\zb_2\in[-1,1]^d$, we have
\begin{align*}
\|\widetilde{\gb}(\zb_1)-\widetilde{\gb}(\zb_2)\|_{\cSX} = \|\cSX\circ\gb(\zb_1)-\cSX\circ\gb(\zb_2)\|_{\cSX}
\le  2 \|\gb(\zb_1)-\gb(\zb_2)\|_{\cX}
\le  2L_{\gb}\|\zb_1-\zb_2\|_2.
\end{align*}

\end{proof}

\subsection{Proof of Lemma \ref{lemma:tildemmindim}}
\label{applemma:tildemmindimproof}

\begin{proof}[Proof of Lemma \ref{lemma:tildemmindim}]
By Lemma \ref{lemma:f1g1lip}, we have
$$\cSX(\cM) = \{\widetilde{\gb}(\zb):\ \zb\in [-1,1]^d\}$$ 
where $\widetilde{\gb}$ is Lipschitz. For any $\delta>0$, the covering number $\cN(\delta,(0,1)^d,\|\cdot\|_2)$ is upper bounded by $C\delta^{-d}$ with a constant $C>0$ \citep{shalev2014understanding}. In other words, there exists a finite set $F_{\delta} \subset [-1,1]^d$ such that
\begin{itemize}
\item $\# F_{\delta} \le C\delta^{-d}$,
\item $[-1,1]^d\subset\cup_{\zb \in F_\delta}B^d_2(\theta,\delta)$.
\end{itemize}
The Lipschitz property of $\gb$ and the condition $\|u\|_{L^\infty(\Omega_{\cX})} \le c_1\|u\|_{\cX}$ imply that
\begin{align}
&\|\widetilde{\gb}(\zb_1)-\widetilde{\gb}(\zb_2)\|_{\infty} 
= \|\cSX\circ\gb(\zb_1)-\cSX\circ\gb(\zb_2)\|_\infty
\nonumber
\\
\le & \|\gb(\zb_1)-\gb(\zb_2)\|_{L^\infty}
\le c_1 \|\gb(\zb_1)-\gb(\zb_2)\|_{\cX}
\le  c_1 L_{\gb}\|\zb_1-\zb_2\|_{2}.
\label{lemma:tildemmindimproofeq1}
\end{align}

The manifold $\cSX(\cM)$ can be covered as
\begin{align*}\cSX(\cM) 
&\subset \widetilde{\gb}\left( \cup_{\zb \in F_\delta}B^d_2(\zb,\delta)\cap [-1,1]^d\right)
 \subset \cup_{\zb \in F_\delta} \widetilde{\gb}\left( B^d_2(\zb,\delta)\cap [-1,1]^d\right)
 \subset \cup_{\zb \in F_\delta}  B^{D}_\infty\left(\widetilde{\gb}(\zb),c_1 L_{\gb}\delta\right)
\end{align*}
where the last inclusion follows from \eqref{lemma:tildemmindimproofeq1}.
By setting $c_1 L_{\gb}\delta =\varepsilon$, we have
$$\cN(\varepsilon,\cSX(\cM),\|\cdot\|_\infty) \le \#F_{\delta} \le C\left(\frac{\varepsilon}{c_1L_{\gb}}\right)^{-d}.$$
Therefore, the Minkowski dimension of $\cSX(\cM)$ is no more than $d$.
\end{proof}

\subsection{Proof of Lemma \ref{lem:Ex:0}}
\label{proof:Ex:0}
\begin{proof}[Proof of Lemma \ref{lem:Ex:0}]
    Lemma \ref{lem:Ex:0} is a variant of \cite[Theorem 5]{nakada2020adaptive}, and can be proved similarly. In \cite[Proof of Theorem 5]{nakada2020adaptive}, the authors constructed $\Phi_{\varepsilon}^{f_1}:=\Phi^{\max,5^D}\odot \Phi^{\rm sum}\odot \Phi^{\rm simul}_{\varepsilon/2}$ in order to have a constant number of layers. To relax such a requirement, we construct $\Phi_{\varepsilon}^{f_1}$ as $\Phi_{\varepsilon}^{f_1}=\Phi^{\max,\card{\cI}}\odot \Phi^{\rm simul}_{\varepsilon/2}$. Then the error bound can be derived by following the rest proof and the network architecture is specified in Lemma \ref{lem:Ex:0}.
    \end{proof}
\subsection{Proof of Lemma \ref{lem:GammaLip}}
\label{proof:GammaLip}
\begin{proof}[Proof of Lemma \ref{lem:GammaLip}]
    We have 
    \begin{align*}
        &\|\Gamma(\zb_1)-\Gamma(\zb_2)\|_{\cS_{\cY}}=\|\cS_{\cY}\circ\Psi\circ \gb(\zb_1)-\cS_{\cY}\circ\Psi\circ \gb(\zb_2)\|_{\cS_{\cY}}\\
        \leq &2\|\Psi\circ \gb(\zb_1)-\Psi\circ \gb(\zb_2)\|_{\cY}
        \leq 2L_{\Psi}\|\gb(\zb_1)- \gb(\zb_2)\|_{\cX}
        \leq  2L_{\Psi}L_{\gb}\|\zb_1-\zb_2\|_2.
    \end{align*}
\end{proof}

\subsection{Proof of Lemma \ref{lem:autoXGene}}
\label{proof:autoXGene}
\begin{proof}[Proof of Lemma \ref{lem:autoXGene}]
To simplify the notation, denote $G^n_{\rm NN}=D_{\cX}^n\circ E_{\cX}^n$. 
    We decompose the error as
    \begin{align}
        \EE_{\cJ_1}\EE_{u\sim \gamma} \left[\|G^n_{\rm NN}(\widetilde{u})-\widetilde{u}\|_{\infty}^2\right]= &\underbrace{2\EE_{\cJ_1} \left[\frac{1}{n}\sum_{i=1}^n\|G^n_{\rm NN}(\widetilde{u}_i)-\widetilde{u}_i\|_{\infty}^2\right]}_{\rm TX_1}+ \nonumber\\
        &\underbrace{\EE_{\cJ_1}\EE_{u\sim \gamma} \left[\|G^n_{\rm NN}(\widetilde{u})-\widetilde{u}\|_{\infty}^2\right]-2\EE_{\cJ_1} \left[\frac{1}{n}\sum_{i=1}^n\|G^n(\widetilde{u}_i)-\widetilde{u}_i\|_{\infty}^2\right]}_{\rm TX_2}.
        \label{eq:X:decom}
    \end{align}
    The term ${\rm TX_1}$ captures the bias and the term ${\rm TX_2}$ captures the variance.

    To bound ${\rm TX_1}$, we will use the approximation result in Theorem \ref{thm:approximation}. Let $\cF_{\rm NN}^{E_{\cX}}, \cF_{\rm NN}^{D_{\cX}}$ be specified as in Lemma \ref{lem:autoXGene}. According to Theorem \ref{thm:approximation} (i), there exists $\widetilde{\fb}_{\rm NN}\in\cF_{\rm NN}^{E_{\cX}}, \widetilde{\gb}_{\rm NN}\in \cF_{\rm NN}^{D_{\cX}}$ satisfying 
    \begin{align*}
        \|\widetilde{\gb}_{\rm NN}\circ\widetilde{\fb}_{\rm NN}(\widetilde{u})- \widetilde{\gb}\circ\widetilde{\fb}(\widetilde{u})\|_{\infty}\leq \varepsilon_2+2\sqrt{d}L_{\gb}\varepsilon_1
    \end{align*}
    for any $\widetilde{u}\in \widetilde{\cM}$. Denote $\widetilde{G}_{\rm NN}=\widetilde{\gb}_{\rm NN}\circ\widetilde{\fb}_{\rm NN}$. We bound ${\rm TX_1}$ as 
    \begin{align}
        {\rm TX_1}=&2\EE_{\cJ_1} \left[\frac{1}{n}\sum_{i=1}^n\|G^n_{\rm NN}(\widetilde{u}_i)-\widetilde{u}_i\|_{\infty}^2\right] \nonumber\\
        =&  2\EE_{\cJ_1} \inf_{G'_{\rm NN}\in \cF_{\rm NN}^G}\left[\frac{1}{n}\sum_{i=1}^n\|G'_{\rm NN}(\widetilde{u}_i)-\widetilde{u}_i\|_{\infty}^2\right] \nonumber\\
        \leq & 2\inf_{G'_{\rm NN}\in \cF_{\rm NN}^G}\EE_{\cJ_1} \left[ \frac{1}{n}\sum_{i=1}^n\|G'_{\rm NN}(\widetilde{u}_i)-\widetilde{u}_i\|_{\infty}^2\right] \nonumber\\
        \leq & 2\EE_{\cJ_1} \left[ \frac{1}{n}\sum_{i=1}^n\|\widetilde{G}_{\rm NN}(\widetilde{u}_i)-\widetilde{u}_i\|_{\infty}^2\right] \nonumber\\
        =& 2\EE_{u\sim \gamma} \left[ \|\widetilde{G}_{\rm NN}(\widetilde{u})-\widetilde{\gb}\circ\widetilde{\fb}(\widetilde{u})\|_{\infty}^2\right] \nonumber\\
        \leq & 16dL_{\gb}^2\varepsilon_1^2+ 4\varepsilon_2^2.
        \label{eq:TX1}
    \end{align}

To bound ${\rm TX}_2$, we will use the covering number of $\cF_{\rm NN}^G$.
We have the following lemma (see a proof in Appendix \ref{proof:TX2}).
\begin{lemma}\label{lem:TX2}
    Under the condition of Lemma \ref{lem:autoXGene}, for any $\delta>0$, we have
    \begin{align}
        {\rm TX_2} \leq \frac{48R_{\cX}^2}{n}\log\cN\left(\frac{\delta}{4R_{\cX}},\cF_{\rm NN}^G,\|\cdot\|_{\infty,\infty}\right)+ 6\delta.
    \end{align}
\end{lemma}
Substituting Lemma \ref{lem:TX2} and (\ref{eq:TX1}) into (\ref{eq:X:decom}) proves Lemma \ref{lem:autoXGene}.
\end{proof}

\subsection{Proof of Lemma \ref{lem:Ex:1}}
\label{proof:Ex:1}
\begin{proof}[Proof of Lemma \ref{lem:Ex:1}]
    We will use the following lemma, which is another variant of \cite[Theorem 5]{nakada2020adaptive}.
    \begin{lemma}
    \label{lem:Ex:2}
        Let $D,d$ be positive integers with $d<D$, $M>0$ and $\Xi\subset [-1,1]^D$ some set. Suppose $d\geq d_M \Xi$. For any $\varepsilon>0$, consider a network class $\cF_{\rm NN}(D,1,L,p,K,\kappa,M)$ with
        \begin{align*}
    L=O(\log \varepsilon^{-1}),\ p=O(\varepsilon^{-d}),\ K=O(\varepsilon^{-d}),\ \kappa=O(\varepsilon^{-3-4(1+\lceil\log_2 M \rceil}).
        \end{align*}
        Then if $\varepsilon\in(0,1/4)$ and for any Lipschitz function $f: [-1,1]^{D} \rightarrow [-M,M]$ with Lipschitz constant bounded by $L_f$, there exists a network $f_{\rm NN}$ with this architecture so that
        \begin{align*}
            \|f_{\rm NN}-f\|_{L^\infty(T_{\xi}(\Xi))} \leq CD(\varepsilon+\xi)
        \end{align*}
        for some constant $C$ depending on $M$.
        The constant hidden in $O(\cdot)$ depends on $d,M,L_f, \xi$ and is only polynomial in $D$. 
    \end{lemma}

    \begin{proof}[Proof of Lemma \ref{lem:Ex:2}]
    Lemma \ref{lem:Ex:2} is another variant of \cite[Theorem 5]{nakada2020adaptive}, and can be proved similarly. 
    In \cite[Proof of Theorem 5]{nakada2020adaptive}, the authors first cover $\Xi$ using hyper-cubes with diameter $r$, which was set to $\varepsilon/{3MD}$. Denote set of cubes by $\cC_r=\{\Lambda_{k,r}\}_{k=1}^{C_{\Xi}}$. 
    The authors in fact prove that if the network architecture is properly set, there exists a network $f_{\rm NN}$ with this architecture satisfying
    \begin{align}
        \sup_{\xb\in \bigcup_k \Lambda_{k,r}}|f_{\rm NN}(\widetilde{\xb})- f(\widetilde{\xb})|\leq \varepsilon.
    \end{align}
    
    Denote the center for $\Lambda_{k,r}$ by $c_k$. Instead covering $\Xi$ by $\cC_r$, we will use  $\cC_{r+2\xi}=\{\Lambda_{k,r+2\xi}\}_{k=1}^{C_{\Xi}}$, where $\Lambda_{k,r+2\xi}$ is the hyper-cube with center $c_k$ and diameter $r+2\xi$. Then we have $T_{\xi}(\Xi)\subset \bigcup_k \Lambda_{k,r+2\xi}$.

    Then similar to the proof of Lemma \ref{lem:Ex:0}, we construct $\Phi_{\varepsilon}^{f_1}$ as $\Phi_{\varepsilon}^{f_1}=\Phi^{\max,\card{\cI}}\odot \Phi^{\rm simul}_{\varepsilon/2}$. By following the rest of the proof, we deduce that
   \begin{align}
        \|f_{\rm NN}(\widetilde{\xb})- f(\widetilde{\xb})\|_{L^{\infty}(T_{\xi}(\Xi)}=\sup_{\xb\in T_{\xi}(\Xi)}|f_{\rm NN}(\widetilde{\xb})- f(\widetilde{\xb})|\leq \sup_{\xb\in \bigcup_k \Lambda_{k,r}}|f_{\rm NN}(\widetilde{\xb})- f(\widetilde{\xb})|\leq CD(\varepsilon+\xi)
    \end{align}
    for some constant $C$ depending on $M$. The network architecture is specified in Lemma \ref{lem:Ex:2}.

    \end{proof}

    Lemma \ref{lem:Ex:1} can be proved by following the first part of the proof of Theorem \ref{thm:approximation} and replacing Lemma \ref{lem:Ex:0} by Lemma \ref{lem:Ex:2}.
    \end{proof}

\subsection{Proof of Lemma \ref{lemma:T1:var}}
\label{sec:T1:var}

\begin{proof}[Proof of Lemma \ref{lemma:T1:var}]

Let $\{\phi_{{\rm NN},j}\}_{j=1}^{\cN(\delta,\cF_{\rm NN}^{\Phi},\|\cdot\|_{L^{\infty,\infty}})}$ be a $\delta$-cover of $\cF_{\rm NN}^{\Phi}$.  There exists $\phi_{{\rm NN},*}$ in this cover satisfying $\|\phi_{{\rm NN},*}-\Phi^n_{\rm NN}\|_{L^{\infty,\infty}}\leq \delta$. 
We have
\begin{align}
    &\EE_{\cJ} \left[\frac{1}{n}\sum_{i=n+1}^{2n}\left\langle \Phi^n_{\rm NN}\circ \cSX(u_i), \cSY(\epsilon_i)\right\rangle_{\cSY}\right] \nonumber\\
    =&\EE_{\cJ} \left[\frac{1}{n} \sum_{i=n+1}^{2n}\left\langle \Phi^n_{\rm NN}\circ \cSX(u_i)-\phi_{{\rm NN},*}\circ\cSX(u_i)+\phi_{{\rm NN},*}\circ\cSX(u_i)-\Phi\circ\cSX(u_i) , \cSY(\epsilon_i)\right\rangle_{\cSY}\right] \nonumber\\
    \leq & \EE_{\cJ} \left[\frac{1}{n}\sum_{i=n+1}^{2n}\left\langle \Phi^n_{\rm NN}\circ \cSX(u_i)-\phi_{{\rm NN},*}\circ\cSX(u_i), \cSY(\epsilon_i)\right\rangle_{\cSY}\right] \nonumber\\ 
    &  + \EE_{\cJ} \left[\frac{1}{n}\sum_{i=n+1}^{2n}\left\langle \phi_{{\rm NN},*}\circ\cSX(u_i)-\Phi\circ\cSX(u_i) , \cSY(\epsilon_i)\right\rangle_{\cSY}\right].
    \label{eq:T1:var:0}
\end{align}
For the first term in (\ref{eq:T1:var:0}), by Lemma \ref{lem:sampleinequality} and Jensen's inequality, we have
\begin{align}
    &\EE_{\cJ} \left[\frac{1}{n}\sum_{i=n+1}^{2n}\left\langle \Phi^n_{\rm NN}\circ \cSX(u_i)-\phi_{{\rm NN},*}\circ\cSX(u_i), \cSY(\epsilon_i)\right\rangle_{\cSY}\right] \nonumber\\
    \leq & \EE_{\cJ} \left[\frac{1}{n}\sum_{i=n+1}^{2n}\| \Phi^n_{\rm NN}\circ \cSX(u_i)-\phi_{{\rm NN},*}\circ\cSX(u_i)\|_{\cSY}\| \cSY(\epsilon_i)\|_{\cSY}\right] \nonumber\\
    \leq  &\frac{\sqrt{|\Omega_{\cY}|}\delta }{n}\sum_{i=n+1}^{2n}\EE_{\cJ} \left[\| \cSY(\epsilon_i)\|_{\cSY}\right] \nonumber\\
    \leq&  \frac{\sqrt{|\Omega_{\cY}|}\delta}{n}\sum_{i=n+1}^{2n}\sqrt{\EE_{\cJ} \left[\| \cSY(\epsilon_i)\|_{\cSY}^2\right]} \nonumber\\
    \leq &|\Omega_{\cY}|\delta\sigma.\label{eq:T1:innerSY}
\end{align}
Substituting (\ref{eq:T1:innerSY}) into (\ref{eq:T1:var:0}) gives rise to
\begin{align}
    &\EE_{\cJ} \left[\frac{1}{n}\sum_{i=n+1}^{2n}\left\langle \Phi^n_{\rm NN}\circ \cSX(u_i), \cSY(\epsilon_i)\right\rangle_{\cSY}\right]\nonumber\\
    \leq &\EE_{\cJ} \left[\frac{1}{n}\sum_{i=n+1}^{2n}\left\langle \phi_{{\rm NN},*}\circ\cSX(u_i)-\Phi\circ\cSX(u_i) , \cSY(\epsilon_i)\right\rangle_{\cSY}\right] + |\Omega_{\cY}|\delta\sigma \nonumber\\
    =& \EE_{\cJ} \left[ \frac{\|\phi_{{\rm NN},*}-\Phi\|_n}{\sqrt{n}}\frac{\sum_{i=n+1}^{2n}\left\langle \phi_{{\rm NN},*}\circ\cSX(u_i)-\Phi\circ\cSX(u_i) , \cSY(\epsilon_i)\right\rangle_{\cSY}}{\sqrt{n}\|\phi_{{\rm NN},*}-\Phi\|_n}\right] + |\Omega_{\cY}|\delta\sigma.
    \label{eq:T1:var:1}
\end{align}
Note that 
\begin{align}
    &\|\phi_{{\rm NN},*}-\Phi\|_n \nonumber\\
    = &\sqrt{\frac{1}{n}\sum_{i=n+1}^{2n} \|\phi_{{\rm NN},*}\circ\cSX(u_i)-\Phi^n_{\rm NN}\circ\cSX(u_i)+ \Phi^n_{\rm NN}\circ\cSX(u_i)-\Phi\circ\cSX(u_i)\|_{\cSY}^2 } \nonumber\\
    \leq &\sqrt{\frac{2}{n}\sum_{i=n+1}^{2n} \left(\|\phi_{{\rm NN},*}\circ\cSX(u_i)-\Phi^n_{\rm NN}\circ\cSX(u_i)\|_{\cSY}^2+ \| \Phi^n_{\rm NN}\circ\cSX(u_i)-\Phi\circ\cSX(u_i)\|_{\cSY}^2\right) } \nonumber\\
    \leq & \sqrt{\frac{2}{n}\sum_{i=n+1}^{2n} \left(|\Omega_{\cY}|\delta^2+ \| \Phi^n_{\rm NN}\circ\cSX(u_i)-\Phi\circ\cSX(u_i)\|_{\cSY}^2\right) } \nonumber\\
    \leq &\sqrt{2}\| \Phi^n_{\rm NN}-\Phi\|_{n}+ \sqrt{2|\Omega_{\cY}|}\delta,
    \label{eq:T1:nnorm}
\end{align}
where we used $(a+b)^2\leq 2a^2+2b^2$ in the first inequality, and $\sqrt{a^2+b^2}\leq a+b$ for $a,b\geq0$ in the last inequality.

Combining (\ref{eq:T1:nnorm}) and (\ref{eq:T1:var:1}), we have
\begin{align}
    &\EE_{\cJ} \left[\frac{1}{n}\sum_{i=n+1}^{2n}\left\langle \Phi^n_{\rm NN}\circ \cSX(u_i), \cSY(\epsilon_i)\right\rangle_{\cSY}\right]\nonumber \\
    \leq & \EE_{\cJ} \left[ \frac{\sqrt{2}\| \Phi^n_{\rm NN}-\Phi\|_{n}+ \sqrt{2|\Omega_{\cY}|}\delta}{\sqrt{n}}\frac{\sum_{i=n+1}^{2n}\left\langle \phi_{{\rm NN},*}\circ\cSX(u_i)-\Phi\circ\cSX(u_i) , \cSY(\epsilon_i)\right\rangle_{\cSY}}{\sqrt{n}\|\phi_{{\rm NN},*}\circ\cSX(u_i)-\Phi\circ\cSX(u_i)\|_n}\right] + |\Omega_{\cY}|\delta\sigma.
    \label{eq:T1:var:2}
\end{align}
Denote $z_j=\frac{\sum_{i=n+1}^{2n}\left\langle \phi_{{\rm NN},j}\circ\cSX(u_i)-\Phi\circ\cSX(u_i) , \cSY(\epsilon_i)\right\rangle_{\cSY}}{\sqrt{n}\|\phi_{{\rm NN},j}\circ\cSX(u_i)-\Phi\circ\cSX(u_i)\|_n}. $
Since $\phi_{{\rm NN},*}$ is one element of the $\delta$-cover of $\cF_{\Phi}$, we have
\begin{align*}
    \left|\frac{\sum_{i=n+1}^{2n}\left\langle \phi_{{\rm NN},*}\circ\cSX(u_i)-\Phi\circ\cSX(u_i) , \cSY(\epsilon_i)\right\rangle_{\cSY}}{\sqrt{n}\|\phi_{{\rm NN},*}\circ\cSX(u_i)-\Phi\circ\cSX(u_i)\|_n} \right|\leq\max_j |z_j|.
\end{align*}
Apply Cauchy-Schwarz inequality to (\ref{eq:T1:var:2}), we have
\begin{align}
    &\EE_{\cJ} \left[\frac{1}{n}\sum_{i=n+1}^{2n}\left\langle \Phi^n_{\rm NN}\circ \cSX(u_i), \cSY(\epsilon_i)\right\rangle_{\cSY}\right] \nonumber\\
    \leq & \EE_{\cJ} \left[ \frac{\sqrt{2}\| \Phi^n_{\rm NN}-\Phi\|_{n}+ \sqrt{2|\Omega_{\cY}|}\delta}{\sqrt{n}}\max_j |z_j|\right] + |\Omega_{\cY}|\delta\sigma \nonumber\\
    \leq& \sqrt{\frac{2}{n}} \sqrt{ \EE_{\cJ}\left[ \left(\| \Phi^n_{\rm NN}-\Phi\|_{n}+ \sqrt{|\Omega_{\cY}|}\delta\right)^2\right] \EE_{\cJ}\left[\max_j |z_j|^2\right]} + |\Omega_{\cY}|\delta\sigma \nonumber\\
    \leq& \sqrt{\frac{2}{n}} \left(\sqrt{ \EE_{\cJ}\left[ 2\| \Phi^n_{\rm NN}-\Phi\|_{n}^2\right]}+ \sqrt{2|\Omega_{\cY}|}\delta \right)\sqrt{\EE_{\cJ}\left[\max_j |z_j|^2\right]} + |\Omega_{\cY}|\delta\sigma \nonumber\\
    =& 2 \left(\sqrt{ \EE_{\cJ}\left[ \| \Phi^n_{\rm NN}-\Phi\|_{n}^2\right]}+ \sqrt{|\Omega_{\cY}|}\delta \right)\sqrt{\frac{\EE_{\cJ}\left[\max_j |z_j|^2\right]}{n}} + |\Omega_{\cY}|\delta\sigma.
    \label{eq:T1:var:3}
\end{align}
Since each element of $\cSY(\epsilon)$ is sub-Gaussian with variance parameter $\sigma^2$, for given $\{u_i\}_{i=1}^n$, each $z_j$ is sub-Gaussian with variance parameter $|\Omega_{\cY}|\sigma^2$. Thus $\EE_{\cJ}\left[\max_j |z_j|^2\right]$ involves a collection of squared sub-Gaussian variables. We bound it using moment generating function. For any $t>0$, we have
\begin{align}
    \EE_{\cJ}\left[\max_j |z_j|^2|\{u_i\}_{i=1}^n\right]=& \frac{1}{t} \log \exp \left( t\EE_{\cJ}\left[\max_j z_j^2|\{u_i\}_{i=1}^n\right] \right) \nonumber\\
    \leq & \frac{1}{t} \log \EE_{\cJ} \left[ \exp\left( t\max_j z_j^2| \{u_i\}_{i=1}^n \right) \right] \nonumber\\
    \leq & \frac{1}{t} \log \EE_{\cJ} \left[ \sum_j \exp(tz_j^2)| \{u_i\}_{i=1}^n\right] \nonumber\\
    \leq & \frac{1}{t} \log \cN(\delta,\cF_{\rm NN}^{\Phi},\|\cdot\|_{\infty,\infty}) + \frac{1}{t} \log \EE_{\cJ}\left[ \exp(tz_1^2)| \{u_i\}_{i=1}^n\right].
\end{align}
Since $z_1$ is sub-Gaussian with variance parameter $|\Omega_{\cY}|\sigma$ for given $\{u_i\}_{i=1}^n$, we have
\begin{align}
    \EE_{\cJ}\left[ \exp(tz_1^2)| \{u_i\}_{i=1}^n\right]=&1+ \sum_{k=1}^{\infty} \frac{t^p \EE_{\cJ}\left[ z_1^{2k}|\{u_i\}_{i=1}^n\right]}{k!} \nonumber\\
    =& 1+\sum_{k=1}^{\infty} \left[\frac{t^p }{k!} \int_0^{\infty} \PP(z_1\geq \eta^{1/2k})d\eta \right] \nonumber\\
    \leq & 1+2\sum_{k=1}^{\infty} \left[\frac{t^p }{k!} \int_0^{\infty} \exp\left( -\frac{\eta^{1/2p}}{2|\Omega_{\cY}|\sigma^2}\right)d\eta \right] \nonumber\\
    =& 1+\sum_{k=1}^{\infty} \frac{2k(2t|\Omega_{\cY}|\sigma^2)^k}{k!}\Gamma_G(k) \nonumber\\
    =& 1+2\sum_{k=1}^{\infty} (2t|\Omega_{\cY}|\sigma^2)^k,
\end{align}
where $\Gamma_G$ denotes the Gamma function. Setting $t=(4|\Omega_{\cY}|\sigma^2)^{-1}$, we have
\begin{align}
    \EE_{\cJ}\left[\max_j |z_j|^2|\{u_i\}_{i=1}^n\right] \leq& 4|\Omega_{\cY}|\sigma^2 \log \cN(\delta,\cF_{\rm NN}^{\Phi},\|\cdot\|_{\infty,\infty}) + 4|\Omega_{\cY}|\sigma^2\log 3 \nonumber\\
    \leq & 4|\Omega_{\cY}|\sigma^2 \log \cN(\delta,\cF_{\rm NN}^{\Phi},\|\cdot\|_{\infty,\infty}) + 6|\Omega_{\cY}|\sigma^2.
    \label{eq:T1:max}
\end{align}
Substituting (\ref{eq:T1:max}) into (\ref{eq:T1:var:2}) proves the lemma.

\end{proof}

\subsection{Proof of Lemma \ref{lem:T2}}
\label{proof:T2}
\begin{proof}[Proof of Lemma \ref{lem:T2}]
Lemma \ref{lem:T2} can be proved by following the proof of Lemma \ref{lem:TX2}. One only need to replace the definition of $\widehat{g}(u)$ to $\widehat{g}(u)=\|\Phi^n_{\rm NN}\circ \cSX(u) -\cSY\circ\Psi(u)\|^2_{\cSY}$. The proof is omitted here.

\end{proof}

\subsection{Proof of Lemma \ref{lem:TX2}}
\label{proof:TX2}
\begin{proof}[Proof of Lemma \ref{lem:TX2}]
    Denote $\widehat{g}(u)=\|G^n\circ \cS_{\cX}(u) -\cS_{\cX}(u)\|^2_{\infty}=\|G^n(\widetilde{u}) -\widetilde{u}\|^2_{\infty}$. We have $\|\widehat{g}\|_{L^{\infty}(\cM)}\leq 4R_{\cX}^2$. We deduce
    \begin{align}
        {\rm TX_2}=& \EE_{\cJ_1} \left[\EE_{u\sim \gamma} [\widehat{g}(u)]-\frac{2}{n} \sum_{i=1}^n \widehat{g}(u_i) \right] \nonumber\\
        =& 2\EE_{\cJ_1} \left[ \frac{1}{2}\EE_{u\sim\gamma} [\widehat{g}(u)]-\frac{1}{n} \sum_{i=1}^n \widehat{g}(u_i) \right] \nonumber\\
        =&2\EE_{\cJ_1} \left[ \EE_{u\sim\gamma} [\widehat{g}(u)]-\frac{1}{n} \sum_{i=1}^n \widehat{g}(u_i)-\frac{1}{2}\EE_{u\sim\gamma} [\widehat{g}(u)] \right].
        \label{eq:TX2:err:0}
    \end{align}
    Note that 
    \begin{align}
        \EE_{u\sim \gamma} [\widehat{g}^2(u)]= &\EE_{u\sim\gamma} \left[ \|G^n\circ \cSX(u) -\cSX(u)\|^4_{\infty} \right] \nonumber\\
        = & \EE_{u\sim\gamma} \left[ \|G^n\circ \cSX(u) -\cSX(u)\|^2_{\infty} \widehat{g}(u)\right] \nonumber\\
        \leq& \EE_{u\sim\gamma} \left[ 4R_{\cX}^2 \widehat{g}(u)\right].
        \label{eq:TX2:gsq}
    \end{align}

    Using relation (\ref{eq:TX2:gsq}), we have
    \begin{align}
        {\rm TX_2}\leq 2\EE_{\cJ_1} \left[ \EE_{u\sim\gamma} [\widehat{g}(u)]-\frac{1}{n} \sum_{i=1}^n \widehat{g}(u_i)-\frac{1}{8R^2_{\cX}}\EE_{u\sim\gamma} [\widehat{g}^2(u)] \right].
        \label{eq:TX2:err:1}
    \end{align}
    Let $\{u'_i\}_{i=1}^n$ be independent copies of $\{u_i\}_{i=1}^n$. Denote the function class
    \begin{align*}
        \cG=\left\{ g(u)=\|G'_{\rm NN}\circ \cSX(u) -\cSX(u)\|^2_{\infty}| G'_{\rm NN}\in \cF_{\rm NN}^G\right\}.
    \end{align*}
    We have $\|g\|_{L^{\infty}(\cM)}\leq 4R_{\cX}^2$ for any $g\in \cG$.
    We bound (\ref{eq:TX2:err:1}) as
    \begin{align}
        {\rm TX_2}\leq& 2\EE_{\cJ_1} \left[ \sup_{g\in \cG}\left(\EE_{u'\sim\gamma} [g(u')]-\frac{1}{n} \sum_{i=1}^n g(u_i)-\frac{1}{8R^2_{\cX}}\EE_{u\sim\gamma} [g^2(u)]\right) \right] \nonumber\\
        =& 2\EE_{\cJ_1} \left[ \sup_{g\in \cG}\left(\EE_{u'\sim\gamma} \left[\frac{1}{n}\sum_{i=1}^n \left( g(u')-g(u_i)\right)\right]-\frac{1}{16R^2_{\cX}}\EE_{u',u\sim\gamma} \left[g^2(u')+g^2(u)\right]\right) \right].
        \label{eq:TX2:err:2}
    \end{align}
    We then consider a $\delta$-cover of $\cG$: $\cG^*=\{g_i^*\}_{i=1}^{\cN(\delta,\cG,\|\cdot\|_{L^{\infty}})}$, where $\cN(\delta,\cG,\|\cdot\|_{L^{\infty}})$ is the covering number. For any $g\in \cG$, there is a $g^*\in \cG^*$ so that $\|g-g^*\|_{L^{\infty}}\leq \delta$.

    We will derive an upper bound for (\ref{eq:TX2:err:2}) by replacing $g$ by $g^*$. Then the problem is converted to analyzing the concentration result on a finite set. First note that
    \begin{align}
        g(u')-g(u)=&g(u')-g^*(u')+g^*(u')-g^*(u)+g^*(u)-g(u) \nonumber\\
        \leq &g^*(u')-g^*(u)+2\delta,
        \label{eq:TX2:gdiff}
    \end{align}
    and
    \begin{align}
        g^2(u')+g^2(u)=&\left(g^2(u')-(g^*)^2(u')\right) + \left((g^*)^2(u')+(g^*)^2(u)\right) -\left((g^*)^2(u)-g^2(u)\right) \nonumber\\
        =&(g^*)^2(u')+(g^*)^2(u)+(g(u')-g^*(u'))(g(u')+g^*(u')) - (g^*(u)-g(u))(g^*(u)+g(u)) \nonumber\\
        \geq & (g^*)^2(u')+(g^*)^2(u)-|g(u')-g^*(u')||g(u')+g^*(u')| - |g^*(u)-g(u)||g^*(u)+g(u)| \nonumber\\
        \geq& (g^*)^2(u')+(g^*)^2(u)-8R_{\cX}^2\delta-8R_{\cX}^2\delta \nonumber\\
        =& (g^*)^2(u')+(g^*)^2(u)-16R_{\cX}^2\delta.
        \label{eq:TX2:g2lowerbound}
    \end{align}
    Utilizing (\ref{eq:TX2:gdiff}) and (\ref{eq:TX2:g2lowerbound}) in (\ref{eq:TX2:err:2}), we get
    \begin{align}
        {\rm TX_2}\leq&  2\EE_{\cJ_1} \EE_{u'\sim\gamma}\left[ \sup_{g\in \cG}\left( \frac{1}{n}\sum_{i=1}^n \left( g^*(u'_i)-g^*(u_i)\right)-\frac{1}{16R^2_{\cX}}\EE_{u',u\sim\gamma} \left[(g^*)^2(u')+(g^*)^2(u)\right]\right) \right] +6\delta \nonumber\\
        =& 2\EE_{\cJ_1} \EE_{u'\sim\gamma} \left[ \max_j\left( \frac{1}{n}\sum_{i=1}^n \left( g_j^*(u'_i)-g_j^*(u_i)\right)-\frac{1}{16R^2_{\cX}}\EE_{u',u\sim\gamma} \left[(g_j^*)^2(u')+(g_j^*)^2(u)\right]\right) \right] +6\delta.
        \label{eq:TX2:err:3}
    \end{align}
    Denote $h_j(i)=g_j^*(u'_i)-g_j^*(u_i)$. We have $\EE_{u'_i,u_i\sim \gamma} h_j(i)=0$ and 
    \begin{align*}
        \Var[h_j(i)]=\EE_{u'_i,u_i\sim \gamma}\left[h^2_j(i)\right]= \EE_{u'_i,u_i\sim \gamma}\left[\left(g_j^*(u'_i)-g_j^*(u_i)\right)^2\right] \leq 2\EE_{u'_i,u_i\sim \gamma}\left[(g_j^*)^2(u'_i)+(g_j^*)^2(u_i)\right].
    \end{align*}
    Thus (\ref{eq:TX2:err:3}) can be written as
    \begin{align}
        &{\rm TX_2}\leq {\rm TX_2'} +6\delta \nonumber\\
        &\mbox{ with } {\rm TX_2}'=2\EE_{\cJ_1} \EE_{u'\sim\gamma} \left[ \max_j\left( \frac{1}{n}\sum_{i=1}^n h_j(i)-\frac{1}{32R^2_{\cX}}\frac{1}{n}\sum_{i=1}^n \Var\left[h_j(i)\right]\right) \right]. 
        \label{eq:TX2:err:4}
     \end{align}
    We will derive an upper bound for ${\rm TX_2}$ using moment generating function. Note that $\|h_j\|_{L^{\infty}}\leq 8R_{\cX}^2$. For $0<t/n<3/(8R_{\cX}^2)$, we have
    \begin{align}
        \EE_{u'_i,u_i\sim \gamma}\left[\exp\left(\frac{t}{n}h_j(i)\right)\right]= & \EE_{u'_i,u_i\sim \gamma} \left[ 1+\frac{t}{n}h_j(i)+ \sum_{k=2}^{\infty} \frac{(t/n)^kh_j^k(i)}{k!} \right] \nonumber\\
        \leq& \EE_{u'_i,u_i\sim \gamma} \left[ 1+\frac{t}{n}h_j(i)+ \sum_{k=2}^{\infty} \frac{(t/n)^kh_j^2(i)(8R_{\cX}^2)^{k-2}}{2\times 3^{k-2}} \right] \nonumber\\
        =& \EE_{u'_i,u_i\sim \gamma} \left[ 1+\frac{t}{n}h_j(i)+ \frac{(t/n)^2h^2_j(i)}{2}\sum_{k=2}^{\infty} \frac{(t/n)^{k-2}(8R_{\cX}^2)^{k-2}}{ 3^{k-2}}\right] \nonumber\\
        =& \EE_{u'_i,u_i\sim \gamma} \left[ 1+\frac{t}{n}h_j(i)+ \frac{(t/n)^2h^2_j(i)}{2}\frac{1}{1-8tR_{\cX}^2/(3n)}\right] \nonumber\\
        =&1+(t/n)^2\Var\left[h_j(i)\right]\frac{1}{2-16R_{\cX}^2/3} \nonumber\\
        \leq& \exp\left(\Var\left[h_j(i)\right]\frac{3(t/n)^2}{6-48tR_{\cX}^2/n} \right),
        \label{eq:TX2:exp:tool}
    \end{align}
    where the last inequality used the relation $1+a\leq \exp(a)$.

We use (\ref{eq:TX2:exp:tool}) to bound ${\rm TX_2}'$ as
\begin{align}
    \exp\left(t\frac{\rm T_2'}{2}\right)= &\exp \left( t\EE_{\cJ_1} \EE_{u'\sim\gamma} \left[ \max_j\left( \frac{1}{n}\sum_{i=1}^n h_j(i)-\frac{1}{32R^2_{\cX}}\frac{1}{n}\sum_{i=1}^n \Var\left[h_j(i)\right]\right) \right]\right) \nonumber\\
    \leq & \EE_{\cJ_1} \EE_{u'\sim\gamma} \left[ \exp\left( t\max_j\left( \frac{1}{n}\sum_{i=1}^n h_j(i)-\frac{1}{32R^2_{\cX}}\frac{1}{n}\sum_{i=1}^n \Var\left[h_j(i)\right]\right) \right)\right] \nonumber\\
    \leq &  \EE_{\cJ_1} \EE_{u'\sim\gamma} \left[ \sum_j \exp\left(  \sum_{i=1}^n\left( \frac{t}{n}h_j(i)-\frac{t}{n}\frac{1}{32R^2_{\cX}} \Var\left[h_j(i)\right]\right) \right)\right] \nonumber\\
    \leq&  \sum_j \exp\left( \sum_{i=1}^n\left( \Var\left[h_j(i)\right]\frac{3(t/n)^2}{6-48tR_{\cX}^2/n}-\frac{t}{n} \frac{1}{32R^2_{\cX}} \Var\left[h_j(i)\right]\right) \right) \nonumber\\
    =& \sum_j \exp\left(  \sum_{i=1}^n\frac{t}{n}\Var\left[h^2_j(i)\right]\left( \frac{3t/n}{6-48tR_{\cX}^2/n}-\frac{1}{32R^2_{\cX}} \right) \right).
\end{align}
Set $t=\frac{n}{24R_{\cX}^2}$ so that $\frac{3t/n}{6-48tR_{\cX}^2/n}-\frac{1}{32R^2_{\cX}}=0$. We have
\begin{align}
    t\frac{\rm T_2'}{2}\leq \log \sum_j \exp(0) \Rightarrow {\rm T_2'}\leq \frac{2}{t} \log\cN(\delta,\cG,\|\cdot\|_{L^{\infty}})=\frac{48R_{\cX}^2}{n}\log\cN(\delta,\cG,\|\cdot\|_{L^{\infty}}).
    \label{eq:tx2:t2p}
\end{align}
We then derive a relation between $\cN(\delta,\cG,\|\cdot\|_{L^{\infty}})$ and $\cN(\delta,\cF_{\rm NN}^G,\|\cdot\|_{\infty,\infty})$. Note that for any $g_1,g_2\in \cG$, there are $G_{\rm NN,1}',G_{\rm NN,2}'\in \cF_{\rm NN}^G$ with $g_1(u)=\|G_{\rm NN,1}'\circ \cSX(u) -\cSX(u)\|^2_{\infty}, g_2(u)=\|G_{\rm NN,2}'\circ \cSX(u) -\cSX(u)\|^2_{\infty}$.  

We have
\begin{align}
    &\|g_1-g_2\|_{L^{\infty}(\cM)} \nonumber\\
    =& \sup_{u\in \cM} \left|\|G_{\rm NN,1}'\circ \cSX(u) -\cSX(u)\|^2_{\infty}-\|G_{\rm NN,2}'\circ \cSX(u) -\cSX(u)\|^2_{\infty}\right| \nonumber\\
    =& \sup_{u\in \cM} \left| \max_k\left( [G_{\rm NN,1}'\circ \cSX(u)]_k - [\cSX(u)]_k\right)^2-\max_{k'}\left([G_{\rm NN,2}'\circ \cSX(u)]_{k'} -[\cSX(u)]_{k'}\right)^2\right| \nonumber\\
    \leq& \sup_{u\in \cM} \left| \max_k\left[\left( [G_{\rm NN,1}'\circ \cSX(u)]_k - [\cSX(u)]_k\right)^2-\left([G_{\rm NN,2}'\circ \cSX(u)]_{k} -[\cSX(u)]_{k}\right)^2\right]\right| \nonumber\\
    =& \sup_{u\in \cM} \left|\max_k \left[\left( [G_{\rm NN,1}'\circ \cSX(u)]_k-[G_{\rm NN,2}'\circ \cSX(u)]_k\right)\left( [G_{\rm NN,1}'\circ \cSX(u)]_k+ [G_{\rm NN,2}'\circ \cSX(u)]_k -2[\cSX(u)]_k\right)\right]\right| \nonumber\\
    \leq& \sup_{u\in \cM} \left\| G_{\rm NN,1}'\circ \cSX(u)-G_{\rm NN,2}'\circ \cSX(u)\right\|_{\infty} \left\| G_{\rm NN,1}'\circ \cSX(u)+ G_{\rm NN,2}'\circ \cSX(u) -2\cSX(u)\right\|_{\infty} \nonumber\\
    \leq & 4R_{\cX}\sup_{u\in \cM} \left\| G_{\rm NN,1}'\circ \cSX(u)-G_{\rm NN,2}'\circ \cSX(u)\right\|_{\infty} \nonumber\\
    =& 4R_{\cX}\left\| G_{\rm NN,1}'-G_{\rm NN,2}'\right\|_{\infty,\infty}.
    \label{eq:TX2:coverningrelation}
\end{align}
Substituting (\ref{eq:TX2:coverningrelation}) and (\ref{eq:tx2:t2p}) into (\ref{eq:TX2:err:4}) gives rise to
\begin{align*}
    {\rm TX_2}\leq \frac{48R_{\cX}^2}{n}\log\cN\left(\frac{\delta}{4R_{\cX}},\cF_{\rm NN}^G,\|\cdot\|_{\infty,\infty}\right)+ 6\delta.
\end{align*}
The lemma is proved.
\end{proof}

\section{Basic properties about $\|\cdot\|_{\cSX}$ and $\|\cdot \|_{\cSY}$}
\label{sec:empirical:property}
In this section, we provide some basic properties of $\|\cdot\|_{\cSX}$ and $\|\cdot \|_{\cSY}$. These properties will be used frequently in the proof of our main results.
\begin{lemma}\label{lem:sampleLip}
    Suppose Assumption \ref{assumption:samplinglip} holds. The discretization operator $\cS_{\cX},\cS_{\cY}$ are Lipschitz with Lipschitz constant 2:
    \begin{align*}
        \|\cS_{\cX}(u_1)-\cS_{\cX}(u_2)\|_{\cS_{\cX}}\leq 2\|u_1-u_2\|_{\cX},\ \text{ and } \
        \|\cS_{\cY}(v_1)-\cS_{\cY}(v_2)\|_{\cS_{\cY}}\leq 2\|v_1-v_2\|_{\cY}
    \end{align*}
    for any $u_1,u_2\in \cX, v_1,v_2\in \cY.$
\end{lemma}
\begin{proof}[Proof of Lemma \ref{lem:sampleLip}]
    For any $u_1,u_2\in \cX$, we have 
    $
        \|\cS_{\cX}(u_1)-\cS_{\cX}(u_2)\|_{\cS_{\cX}}=\|\cS_{\cX}(u_1-u_2)\|_{\cS_{\cX}} \leq 2\|u_1-u_2\|_{\cX}.
    $
    The case for $\cS_{\cY}$ can be proved similary.
\end{proof}

\begin{lemma}\label{lem:sampleinequality}
    The operation $\langle \cdot,\cdot\rangle_{\cSX}$ and $\langle \cdot,\cdot\rangle_{\cSY}$ satisfies
    \begin{align*}
        &|\langle \cSX(u_1),\cSX(u_2)\rangle_{\cSX}| \leq \|\cSX(u_1)\|_{\cSX}\|\cSX(u_2)\|_{\cSX},\\
        &|\langle \cSY(v_1),\cSY(v_2)\rangle_{\cSY}|\leq \|\cSY(v_1)\|_{\cSY}\|\cSY(v_2)\|_{\cSY}
    \end{align*}
    for any $u_1,u_2\in \cX$ and $v_1,v_2\in \cY$.
\end{lemma}
\begin{proof}[Proof of Lemma \ref{lem:sampleinequality}]
    We prove the inequality for $\langle \cdot,\cdot\rangle_{\cSX}$. The inequality for $\cSY$ can be proved similarly. 
     Denote $
        \widetilde{\ub}=[w_1^{1/2}u(\xb_1),...,w_{D_1}^{1/2}u(\xb_{D_1})]^{\top}.$
    By H\"{o}lder's inequality, we  have
    \begin{align*}
       & |\langle \cSX(u_1),\cSX(u_2)\rangle_{\cSX}|= \sum_{i=1}^{D_1} w_iu_1(\xb_i)u_2(\xb_i)
        = \left|\sum_{i=1}^{D_1} (w_i^{1/2}u_1(\xb_i))(w_i^{1/2}u_2(\xb_i))\right|\\
        =& |\langle \widetilde{\ub}_1,\widetilde{\ub}_2|\rangle
        \leq  \|\widetilde{\ub}_1\|_2\|\widetilde{\ub}_2\|_2
        = \sqrt{\sum_{i=1}^{D_1} w_iu_1^2(\xb_i)}\sqrt{\sum_{i=1}^{D_1} w_iu_2^2(\xb_i)}
        = \|\cSX(u_1)\|_{\cSX}\|\cSX(u_2)\|_{\cSX}.
    \end{align*}
\end{proof}

\begin{lemma}\label{lem:normInequality}
     $\|\cdot\|_{\cSX}$ and $\|\cdot \|_{\cSY}$ are norms in $\RR^{D_1}$ and $\RR^{D_2}$, respectively.  
\end{lemma}
\begin{proof}[Proof of Lemma \ref{lem:normInequality}]
For any $u\in\cX$, we have $\|\cSX(u)\|_{\cSX} \ge 0$ since the $w_i$'s are positive. Furthermore, $\|\cSX(u)\|_{\cSX} = 0$ if and only if $\cSX(u) = \mathbf{0}$. For any $\lambda\in \RR$, we have
\begin{align*}
\|\cSX(\lambda u)\|_{\cSX} = \sqrt{\sum_{i=1}^{D_1} w_i \lambda^2 u(\xb_i)^2} = |\lambda| \|\cSX( u)\|_{\cSX}.
\end{align*}
\commentout{
    \begin{align*}
        &\|\cS_{\cX}(u_1)+\cS_{\cX}(u_2)\|_{\cS_{\cX}}\leq \|\cS_{\cX}(u_1)\|_{\cS_{\cX}}+ \|\cS_{\cX}(u_2)\|_{\cX},\\
        &\|\cS_{\cY}(v_1)+\cS_{\cY}(v_2)\|_{\cS_{\cY}}\leq \|\cS_{\cY}(v_1)\|_{\cS_{\cY}}+\|\cS_{\cY}(v_2)\|_{\cY},
    \end{align*}
    for any $u_1,u_2\in \cX, v_1,v_2\in \cY.$
    }
We next prove the triangle inequality:    
    For any $u_1,u_2\in \cX$, we have
    \begin{align*}
       & \|\cS_{\cX}(u_1)+\cS_{\cX}(u_2)\|_{\cS_{\cX}} \\
       &=\sqrt{\langle \cS_{\cX}(u_1)+\cS_{\cX}(u_2),\cS_{\cX}(u_1)+\cS_{\cX}(u_2)\rangle_{\cX}}\\
        =&\sqrt{\sum_{i=1}^{D_1} w_i(u_1(\xb_i)+u_2(\xb_i))^2}
        = \sqrt{\sum_{i=1}^{D_1} w_iu_1^2(\xb_i)+2\sum_{i=1}^{D_1}w_i u_1(\xb_i)u_2(\xb_i)+\sum_{i=1}^{D_1}w_i u_2(\xb_i)^2)}\\
        =& \sqrt{\|\cS_{\cX}(u_1)\|_{\cS_{\cX}}^2+ 2\langle \cS_{\cX}(u_1), \cS_{\cX}(u_2)\rangle_{\cS_{\cX}} +\cS_{\cX}(u_2)\|_{\cS_{\cX}}^2 }\\
        \leq & \sqrt{\|\cS_{\cX}(u_1)\|_{\cS_{\cX}}^2+ 2\|\cS_{\cX}(u_1)\|_{\cS_{\cX}}\|\cS_{\cX}(u_2)\|_{\cS_{\cX}} +\cS_{\cX}(u_2)\|_{\cS_{\cX}}^2 }
        =\|\cS_{\cX}(u_1)\|_{\cS_{\cX}}+\|\cS_{\cX}(u_2)\|_{\cS_{\cX}},
    \end{align*}
    where we use Lemma \ref{lem:sampleinequality} in the first inequality. The result about $\|\cdot\|\cS_{\cY}$ can be proved similarly.
\end{proof}

\end{document}